\documentclass[11pt]{article}

\usepackage{fullpage}
\usepackage{amsmath, amsfonts, amssymb, amsthm, amsxtra}
\usepackage{mathrsfs, mathtools}
\usepackage{bbm}
\usepackage[shortlabels]{enumitem}
\usepackage{authblk}
\usepackage{subcaption}

\usepackage{float}
\usepackage[usenames, dvipsnames]{xcolor}
\usepackage[%dvips,
            CJKbookmarks=true, 
            bookmarksnumbered=true,
            bookmarksopen=true,
            colorlinks=true,
            citecolor=red,
            linkcolor=blue,
            anchorcolor=red,
            urlcolor=blue,
            linktoc=page
            ]{hyperref}
\usepackage{prettyref,xspace}
\newrefformat{eq}{(\ref{#1})}
\newrefformat{chap}{Chapter~\ref{#1}}
\newrefformat{sec}{Section~\ref{#1}}
\newrefformat{alg}{Algorithm~\ref{#1}}
\newrefformat{fig}{Fig.~\ref{#1}}
\newrefformat{tab}{Table~\ref{#1}}
\newrefformat{rmk}{Remark~\ref{#1}}
\newrefformat{clm}{Claim~\ref{#1}}
\newrefformat{def}{Definition~\ref{#1}}
\newrefformat{cor}{Corollary~\ref{#1}}
\newrefformat{lmm}{Lemma~\ref{#1}}
\newrefformat{prop}{Proposition~\ref{#1}}
\newrefformat{app}{Appendix~\ref{#1}}
\newrefformat{hyp}{Hypothesis~\ref{#1}}
\newrefformat{thm}{Theorem~\ref{#1}}
\newrefformat{ass}{Assumption~\ref{#1}}
\newrefformat{conj}{Conjecture~\ref{#1}}

\newcommand{\stepa}[1]{\overset{\rm (a)}{#1}}
\newcommand{\stepb}[1]{\overset{\rm (b)}{#1}}

\newcommand{\pin}{p_{\rm in}}
\newcommand{\pout}{p_{\rm out}}

\usepackage{algorithm}% http://ctan.org/pkg/algorithms
\usepackage{algorithmic}% http://ctan.org/pkg/algorithms

\newtheorem{theorem}{Theorem}[section]
\newtheorem*{theorem*}{Theorem}
\newtheorem{lemma}[theorem]{Lemma}
\newtheorem{proposition}[theorem]{Proposition}
\newtheorem{corollary}[theorem]{Corollary}

\theoremstyle{definition}

\newcommand{\iu}{\mathbf{i}}

\newcommand{\iid}{i.i.d.\xspace}
\newcommand{\ER}{Erd\H{o}s-R\'{e}nyi\xspace}
\newcommand{\eqdistr}{{\stackrel{\rm (d)}{=}}}
\newcommand{\iiddistr}{{\stackrel{\text{\iid}}{\sim}}}
\newcommand{\ones}{\mathbf 1}

\newcommand{\reals}{{\mathbb{R}}}
\newcommand{\complex}{{\mathbb{C}}}

\newcommand{\identity}{\mathbf I}
\newcommand{\allones}{\mathbf J}
\renewcommand{\flat}{\mathbf{F}}
\newcommand{\im}{\mathbf i}
\renewcommand{\Im}{\operatorname{Im}}
\newcommand{\bA}{\mathbf{A}}
\newcommand{\bB}{\mathbf{B}}

\newcommand{\poly}{\mathrm{poly}}
\renewcommand{\tilde}{\widetilde}
\renewcommand{\hat}{\widehat}

\newcommand{\R}{\mathbb{R}}
\newcommand{\p}{\mathbb{P}}
\newcommand{\E}{\mathbb{E}}

\newcommand{\C}{\mathbb{C}}
\newcommand{\fS}{\mathfrak{S}}
\newcommand{\cS}{\mathcal{S}}

\newcommand{\cL}{\mathcal{L}}

\newcommand{\cB}{\mathcal{B}}

\newcommand{\1}{\mathbbm{1}}
\newcommand{\coord}{\mathbf{e}}
\newcommand{\bone}{\mathbf{1}}

\newcommand{\eps}{\varepsilon}

\newcommand{\Ber}{\mathrm{Bern}}

\newcommand{\id}{\mathsf{id}}

\newcommand{\Tr}{\operatorname{Tr}}
\newcommand{\polylog}{\operatorname{polylog}}
\newcommand{\vectorize}{\mathsf{vec}}
\newcommand{\vecc}{\mathsf{vec}}
\newcommand{\defn}{\triangleq}

\newcommand{\argmax}{\operatorname*{argmax}}
\newcommand{\GOE}{\mathsf{GOE}(n)}

\newcommand{\sN}{N}
\newcommand{\stepsize}{\gamma}
\newcommand{\ridge}{\eta^2}
\renewcommand{\phi}{\varphi}

\newcommand{\tZ}{\widetilde{Z}}

\newcommand{\tL}{\widetilde{L}}
\newcommand{\tK}{\widetilde{K}}

\newcommand{\bx}{\mathbf{x}}

\newcommand{\bI}{\mathbf{I}}
\newcommand{\bJ}{\mathbf{J}}
\newcommand{\dimone}{n}
\newcommand{\dimtwo}{m}

\newcommand{\Expect}{\mathbb{E}}
\newcommand{\expect}[1]{\mathbb{E}\left[ #1 \right]}

\newcommand{\prob}[1]{ \mathbb{P}\left\{ #1 \right\} }

\newcommand{\Var}{\mathsf{Var}}

\newcommand{\Iprod}[2]{\langle #1, #2 \rangle}
\newcommand{\indc}[1]{{\mathbf{1}_{\left\{{#1}\right\}}}}

\newcommand{\pth}[1]{\left( #1 \right)}
\newcommand{\qth}[1]{\left[ #1 \right]}
\newcommand{\sth}[1]{\left\{ #1 \right\}}

\newcommand{\GRAMPA}{\texttt{GRAMPA}\xspace}
\newcommand{\DegreeProfile}{\texttt{DegreeProfile}\xspace}
\newcommand{\QPDS}{\texttt{QP-DS}\xspace}
\newcommand{\IsoRank}{\texttt{IsoRank}\xspace}
\newcommand{\EigenAlign}{\texttt{EigenAlign}\xspace}
\newcommand{\LowRankAlign}{\texttt{LowRankAlign}\xspace}
\newcommand{\TopEigenVec}{\texttt{TopEigenVec}\xspace}

\newcommand{\matS}{S}
\newcommand{\mattS}{\widetilde{S}}
\newcommand{\Htop}{H}

\title{Spectral Graph Matching and Regularized Quadratic Relaxations I: The Gaussian Model}

\author{Zhou Fan, Cheng Mao, Yihong Wu, and Jiaming Xu\thanks{
Z.\ Fan, C.\ Mao, and Y.\ Wu are with Department of Statistics and Data Science, Yale University, New Haven, USA, 
\texttt{\{zhou.fan,cheng.mao,yihong.wu\}@yale.edu}.
J.\ Xu is with The Fuqua School of Business, Duke University, Durham, USA, \texttt{jx77@duke.edu}.
Y.~Wu is supported in part by the NSF Grants CCF-1527105, CCF-1900507, an NSF CAREER award CCF-1651588, and an Alfred Sloan fellowship.
J.~Xu is supported by the NSF Grants IIS-1838124, CCF-1850743, and CCF-1856424.
}}

\date{\today}

\begin{document}

\maketitle

\begin{abstract}
%We study the \emph{graph matching} problem of identifying the unknown vertex correspondence between two unlabeled, edge-correlated graphs. 
Graph matching aims at finding the vertex correspondence between two
unlabeled graphs that maximizes the total edge weight correlation.
%number of common edges.
This amounts to
solving a computationally intractable quadratic assignment problem.
In this paper we propose a new spectral method, GRAph Matching by Pairwise eigen-Alignments (GRAMPA).
%We propose a new spectral method, GRAph Matching by Pairwise eigen-Alignments (GRAMPA), which can also be interpreted as solving a regularized quadratic programming relaxation of the quadratic assignment problem, followed by a simple rounding procedure. 
Departing from prior spectral approaches that only compare top eigenvectors, or
eigenvectors of the same order, GRAMPA first constructs a similarity matrix as
a weighted sum of outer products between \emph{all} pairs of eigenvectors  of the
two graphs, 
%with weights depending on the separation of the corresponding eigenvalues in a Cauchy form, 
with weights given by a Cauchy kernel applied to the separation of the
corresponding eigenvalues, then outputs a matching by a simple rounding procedure. 
The similarity matrix can also be interpreted as the solution to a regularized quadratic programming relaxation of the quadratic assignment problem. 

For the Gaussian Wigner model in which two complete graphs on $n$ vertices have
Gaussian edge weights with correlation coefficient $1-\sigma^2$, we show that
GRAMPA exactly recovers the correct vertex correspondence with high probability
when $\sigma = O(\frac{1}{\log n})$. This matches the state of the art of
polynomial-time algorithms, and significantly improves over existing spectral
methods which require $\sigma$ to be polynomially small in $n$. The superiority
of GRAMPA is also demonstrated on a variety of synthetic and real datasets, in terms of both statistical accuracy and computational efficiency. Universality results, including similar guarantees for dense and sparse \ER graphs, are deferred to the companion paper \cite{FMWX19b}.

\end{abstract}

\tableofcontents

\section{Introduction}

Given a pair of graphs, the problem of \emph{graph matching} or \emph{network alignment} refers to finding a bijection between the vertex sets so that the edge sets are maximally aligned \cite{conte2004thirty, livi2013graph, emmert2016fifty}. 
This is a ubiquitous problem arising in a variety of applications, including
network de-anonymization \cite{narayanan2008robust, narayanan2009anonymizing},
pattern recognition \cite{conte2004thirty, schellewald2005probabilistic}, and
computational biology \cite{singh2008global, kazemi2016proper}.
Finding the best matching between two graphs with adjacency matrices $A$ and
$B$ may be formalized as the following combinatorial optimization problem over the set of permutations $\cS_n$ on $\{1, \dots, n\}$:
\begin{align}
\max_{\pi \in \cS_n}  \sum_{i,j=1}^{n} A_{ij} B_{\pi(i),\pi(j)}.
%\Iprod{A}{\Pi B \Pi^\top},
\label{eq:QAP}
\end{align}
This is an instance of the notoriously difficult \emph{quadratic assignment problem} (QAP) \cite{Pardalos94thequadratic,burkard1998quadratic}, which is NP-hard to solve or to approximate within a growing factor \cite{makarychev2010maximum}. 

As the worst-case computational hardness of the QAP \eqref{eq:QAP} may not be
representative of typical graphs, various average-case models have been
studied. For example, when the two graphs are isomorphic, the resulting
\emph{graph isomorphism} problem can be solved for \ER random graphs in linear
time whenever information-theoretically possible
\cite{bollobas1982distinguishing,czajka2008improved},
but remains open to be solved in polynomial time for arbitrary graphs.
In ``noisy'' settings where the graphs are not exactly isomorphic, there is a
recent surge of interest in computer science, information theory, and
statistics for studying random graph matching \cite{yartseva2013performance, lyzinski2014seeded, kazemi2015growing, feizi2016spectral, cullina2016improved, shirani2017seeded, cullina2017exact, dai2018performance, barak2018nearly, cullina2018partial, DMWX18, mossel2019seeded}. 

\subsection{Random weighted graph matching}
\label{sec:model}

In this work, we study the following \emph{random weighted graph matching}
problem: Consider two weighted graphs with $n$ vertices, and
a latent permutation $\pi^*$ on $\{1,\ldots,n\}$ such that vertex $i$ of the
first graph corresponds to vertex $\pi^*(i)$ of the second.
Denoting by $A$ and $B$ their (symmetric) weighted adjacency matrices,
suppose that $\{(A_{ij},B_{\pi^*(i),\pi^*(j)}): 1 \leq i < j \leq n\}$ are
independent pairs of positively correlated random variables. We wish to
recover $\pi^*$ from $A$ and $B$.

Notable special cases of this model include the following:

\begin{itemize}
	\item \ER graph model: $(A_{ij},B_{\pi^*(i),\pi^*(j)})$ is a pair of
            correlated Bernoulli random variables. Then $A$ and $B$ are \ER graphs with correlated edges. This model has been extensively studied in 
             \cite{Pedarsani2011,lyzinski2014seeded,korula2014efficient,cullina2016improved,lubars2018correcting,barak2018nearly,dai2018performance,cullina2018partial,DMWX18}.
%            \cite{Pedarsani2011,lyzinski2014seeded,yartseva2013performance,korula2014efficient,kazemi2015growing,feizi2016spectral,cullina2016improved,cullina2017exact,lubars2018correcting,barak2018nearly,dai2018performance,cullina2018partial,DMWX18} 
%            \0nb{YW: need to trim}. 

\item Gaussian Wigner model: $(A_{ij},B_{\pi^*(i),\pi^*(j)})$ is a pair of
    correlated Gaussian variables, for example, $B_{\pi^*(i),\pi^*(j)} =
        A_{ij} + \sigma Z_{ij}$ where $\sigma \ge 0$ and $A_{ij},Z_{ij}$
        are independent standard Gaussians.
Then $A$ and $B$ are complete graphs with correlated Gaussian edge weights. 
This model was proposed in \cite[Section 2]{DMWX18} as a prototype of random
        graph matching due to its simplicity, and certain results in the Gaussian model
        may be expected to carry over to dense \ER graphs.
\end{itemize}

Spectral methods have a long history in testing graph isomorphism \cite{babai1982isomorphism} and the graph matching problem \cite{umeyama1988eigendecomposition}.
In this paper, we introduce a new spectral method for graph matching,
for which we prove the following exact recovery guarantee 
in the Gaussian setting.
\begin{theorem*}[Informal statement]
    Under the Gaussian Wigner model, if $\sigma \le c/\log n$ for a
    sufficiently small constant $c > 0$, then
    a spectral algorithm recovers $\pi^*$ exactly with high probability. 
\end{theorem*}

%\0nb{ZF: Maybe discuss sharpness of this result and connection to degree profile
%and related literature here.}\0nb{YW: see if the following is sufficient.}

%\noindent
We describe the method in Section \ref{sec:spectral} below. The
algorithm may also be interpreted as a regularized convex relaxation of the
QAP program (\ref{eq:QAP}), and we discuss this
connection in Section \ref{sec:convex}.
The performance guarantee matches the state of the art of polynomial-time
algorithms, namely, the degree profile method proposed in \cite{DMWX18}, and exponentially 
improves the performance of existing spectral matching algorithms,
which require $\sigma = \frac{1}{\poly(n)}$ as opposed to
$\sigma=O(\frac{1}{\log n})$ for our proposal.

%Compared with existing spectral approaches which require $\sigma$ to be polynomially small in $n$, our new spectral method is exponentially more robust which works when $\sigma = O(\frac{1}{\log n})$. The performance guarantee also matches the state of the art of polynomial-time algorithms, namely, the degree profile method proposed \cite{DMWX18}, and improves the performance of existing spectral matching algorithm exponentially, which requires $\sigma = \frac{1}{\poly(n)}$.

In a companion paper \cite{FMWX19b}, we apply resolvent-based
universality techniques to establish similar guarantees for our spectral
method on non-Gaussian models, including dense and sparse \ER graphs.
 %where $\sigma^2$ plays the role of the fraction of differed edges.
Our proofs here in the Gaussian model are more direct and transparent,
using instead the rotational invariance of $A$ and $B$ and yielding slightly stronger guarantees.

A variant of our method may also be applied to match two bipartite random
graphs, which we discuss in Section \ref{sec:bipartite}. This is an extension of
the analysis in the non-bipartite setting, which is our primary focus.

\subsection{A new spectral method}\label{sec:spectral}

Write the spectral decompositions of the weighted adjacency matrices $A$
and $B$ as
\begin{equation}\label{eq:spectraldecomp}
    A = \sum_{i=1}^n \lambda_i u_i u_i^\top
\quad \text{ and } \quad
    B = \sum_{j=1}^n \mu_j v_j v_j^\top
\end{equation}
where the eigenvalues are ordered\footnote{This is in fact not needed for computing the similarity matrix \prettyref{eq:spectralnew}.} such that
\[\lambda_1 \ge \cdots \ge \lambda_n \quad \text{ and } \quad
\mu_1 \ge \cdots \ge \mu_n.\]
%Our proposed method consists of two steps:
%
%\paragraph{\prettyref{alg:grampa}: Spectral Graph Matching}
%\begin{enumerate}
    %\item For a tuning parameter $\eta>0$, construct the similarity matrix
%\begin{align}
%\widehat{X}=
    %\sum_{i,j=1}^n w(\lambda_i,\mu_j) \cdot u_iu_i^\top \bJ v_j v_j^\top \in
    %\R^{n \times n}
%\label{eq:spectralnew}
%\end{align}
%where $\bJ \in \R^{n \times n}$ denotes the all-one matrix and
%$w$ is the Cauchy kernel of bandwidth $\eta$:
%\begin{equation}\label{eq:cauchykernel}
%w(x,y)=\frac{ 1 }{ (x-y)^2 + \ridge }.
%\end{equation}
%\item Construct the permutation estimate $\hat{\pi}$ by
    %``rounding'' $\widehat{X}$ to a permutation,
        %for example, by solving the \emph{linear assignment problem} (LAP)
%\begin{equation}\label{eq:linearassignment}
%\hat \pi=\argmax_{\pi \in \cS_n} \sum_{i=1}^n \widehat{X}_{i,\pi(i)}.
%\end{equation}
%\end{enumerate}

\begin{algorithm}
\caption{GRAMPA (GRAph Matching by Pairwise eigen-Alignments)}\label{alg:grampa}
\begin{algorithmic}[1]
\STATE {\bfseries Input:} Weighted adjacency matrices $A$ and $B$ on $n$ vertices, a tuning parameter $\eta>0$.
\STATE {\bfseries Output:} A permutation $\hat\pi \in \cS_n$.
\STATE  Construct the similarity matrix
\begin{align}
\widehat{X}=
    \sum_{i,j=1}^n w(\lambda_i,\mu_j) \cdot u_iu_i^\top \bJ v_j v_j^\top \in
    \R^{n \times n}
\label{eq:spectralnew}
\end{align}
where $\bJ \in \R^{n \times n}$ denotes the all-one matrix and
$w$ is the Cauchy kernel of bandwidth $\eta$:
\begin{equation}\label{eq:cauchykernel}
w(x,y)=\frac{ 1 }{ (x-y)^2 + \ridge }.
\end{equation}
\STATE Output the permutation estimate $\hat{\pi}$ by
    ``rounding'' $\widehat{X}$ to a permutation,
        for example, by solving the \emph{linear assignment problem} (LAP)
\begin{equation}\label{eq:linearassignment}
\hat \pi=\argmax_{\pi \in \cS_n} \sum_{i=1}^n \widehat{X}_{i,\pi(i)}.
\end{equation}
\end{algorithmic}
\end{algorithm}

Our new spectral method is given in \prettyref{alg:grampa}, which we refer to as 
\emph{graph matching by pairwise eigen-alignments} (GRAMPA). 
Therein, the linear assignment problem (\ref{eq:linearassignment}) may be cast as
a linear program (LP) over doubly stochastic matrices, i.e., the Birkhoff polytope (see (\ref{eq:birkhoff}) below),
 or solved efficiently using the Hungarian
algorithm \cite{kuhn1955hungarian}. We advocate this rounding approach in
practice, although our theoretical results apply equally to the simpler rounding procedure matching each $i$ to
\begin{equation}
\hat{\pi}(i)=\argmax_j \hat{X}_{ij}.
\label{eq:greedyround}
\end{equation}
We discuss the choice
of the tuning parameter $\eta$ further in Section \ref{sec:numeric}, and find in practice that the
performance is not too sensitive to this choice. 
%\0nb{ZF: Added note about $\eta$.}

Let us remark that \prettyref{alg:grampa} exhibits the following
two elementary but desirable properties.
\begin{itemize}
    \item Unlike previous proposals, our spectral method is insensitive to the choices of signs for individual eigenvectors $u_i$ and $v_j$ in (\ref{eq:spectraldecomp}). More generally, it does not depend on the specific choice of eigenvectors if certain eigenvalues have multiplicity greater than one. This is because the similarity matrix 
(\ref{eq:spectralnew}) depends on the eigenvectors of $A$ and $B$
only through the projections onto their distinct eigenspaces.

\item Let $\hat{\pi} (A, B)$ denote the output of \prettyref{alg:grampa} with inputs $A$ and $B$. 
For any fixed permutation $\pi$, denote by $B^\pi$ the matrix with entries 
$B^\pi_{ij}=B_{\pi(i),\pi(j)}$, and by $\pi \circ \hat{\pi}$ the composition
$(\pi \circ \hat{\pi})(i)=\pi(\hat{\pi}(i))$.
Then we have the \emph{equivariance} property
\begin{equation}
\pi \circ \hat{\pi}(A,B^\pi)=\hat{\pi}(A,B)
\label{eq:invariance}
\end{equation}
 and similarly for $A^\pi$.
That is, the outputs given $(A,B^\pi)$ and given $(A,B)$ represent the same
matching of the underlying graphs.
This may be verified from (\ref{eq:spectralnew}) as a consequence of
the identity $\bJ=\bJ \Pi=\Pi \bJ$ for any permutation matrix $\Pi$.
\end{itemize}

%\0nbr{JX. Shall we brieftly mention how to set the tuning parameter $\eta$ somewhere when we introduce the method?}
To further motivate the construction (\ref{eq:spectralnew}), we note that
\prettyref{alg:grampa} follows the same general paradigm as several existing spectral
methods for graph matching, which seek to recover $\pi^*$
by rounding a similarity matrix $\widehat{X}$ constructed to leverage
correlations between eigenvectors of $A$ and $B$. These methods include:

\begin{itemize}
    \item \emph{Low-rank methods} that use a small number of eigenvectors of $A$ and $B$. The simplest such approach uses only the leading eigenvectors, taking as the similarity matrix
\begin{equation}\label{eq:rankone}
\widehat{X}=u_1v_1^\top.
\end{equation}
Then $\hat \pi$ which solves (\ref{eq:linearassignment}) orders the entries of
$v_1$ in the order of $u_1$. Other rank-$1$ spectral methods and
        low-rank generalizations have also been proposed and studied in
        \cite{feizi2016spectral, kazemi2016structure}.
    \item \emph{Full-rank methods} that use all eigenvectors of $A$ and $B$. A
notable example is the popular method of Umeyama
\cite{umeyama1988eigendecomposition}, which sets
\begin{equation}\label{eq:umeyama}
\widehat{X}=\sum_{i=1}^n s_i u_i v_i^\top
\end{equation}
where $s_i$ are signs in $\{-1,1\}$; see also the related approach of
\cite{xu2001pca}. 
The motivation is that \prettyref{eq:umeyama} is the solution to the \emph{orthogonal relaxation} 
%(also known as the eigenvalue bound)
 of the QAP \prettyref{eq:QAP}, where the feasible set is relaxed to the set the orthogonal matrices \cite{finke1987quadratic}. 
As the correct choice of signs in \prettyref{eq:umeyama}
may be difficult to determine in practice, \cite{umeyama1988eigendecomposition}
suggests also an alternative construction
\begin{align}
\widehat{X}=\sum_{i=1}^n |u_i| |v_i|^\top
\label{eq:ume}
\end{align} 
where
$|u_i|$ denotes the entrywise absolute value of $u_i$.
\end{itemize}

%\0nb{Also the power iteration version}
%\0nb{CM: Is this just the Onaran, Villar paper I cited above?}
%\0nbr{JX. The projected power iteration does the following two steps in each iteration starting with an initial
%permutation $x^0=\vecc(\Pi^0)$:
%First, compute $y^{t+1}= (B \otimes A) x^t$; Second, round $y^t$ to a permutation using either
%linear assignment or greedy matching. The projected power iteration can be equivalently viewed as an alternating minimization
%$\Pi^{t+1} \in \arg \max_{\Pi} \iprod{A \Pi^t B}{\Pi}$ with $x^t=\vecc(\Pi^t)$. For this method to work, we need a good
%initialization $\Pi^0$. This initialization is often provided by the spectral methods or QP relaxations or seeded information.
%I am not sure whether it is OK to call the projected power iteration itself a spectral method; it is more like a greedy improving scheme in my opinion.}
%\0nb{YW: OK how about we do not discuss [OV17] here? We should keep this section not too long anyway. If we are going to add a ''other related work'' subsection, then we can dump it there.}
%\0nbr{CM: removed for the moment}

Compared with these constructions, the proposed new spectral method has two important features that we elaborate below:
\vspace{-5pt}
\paragraph{``All pairs matter.''}

%\begin{enumerate}
	%\item \emph{``All pairs matter.''}
Departing from existing approaches, our proposal $\widehat{X}$ in
(\ref{eq:spectralnew}) uses a combination of $u_iv_j^\top$ for \emph{all} $n^2$
pairs $i,j \in \{1,\ldots,n\}$, rather than only $i=j$. This renders our method significantly more resilient to noise. Indeed, while all of the above methods 
can succeed in recovering $\pi^*$ in the noiseless case,
%when $A$ and $B$ are isomorphic, 
methods based only on pairs $(u_i,v_i)$ with $i=j$
are brittle to noise if $u_i$ and $v_i$ quickly decorrelate as the amount of
noise increases---this may happen when $\lambda_i$ is not separated from other eigenvalues by a large spectral gap. 
When this decorrelation occurs, $u_i$
becomes partially correlated with $v_j$ for neighboring indices $j$, and the
construction (\ref{eq:spectralnew}) leverages these partial correlations in a
weighted manner to provide a more robust estimate of $\pi^*$.

%We expect this phenomenon to arise whenever $\lambda_i$ is not separated from other eigenvalues by a large spectral gap. 
The eigenvector alignment is
quantitatively understood in certain regimes for the Gaussian Wigner model $B=A+\sigma Z$ when $A,Z$ are
GOE or GUE: It is known that
$\E[\langle u_i,v_i \rangle^2]=o(1)$ for the leading eigenvector $i=1$ as soon
as $\sigma^2 \gg n^{-1/3}$ \cite[Theorem 3.8]{chatterjee2014superconcentration},
and for $i$ in the bulk of the Wigner semicircle spectrum as soon as
$\sigma^2 \gg n^{-1}$ \cite{bourgade2017eigenvector,benigni2017eigenvectors}. For $i$ in
the bulk, and the noise regime $n^{-1+\eps} \ll \sigma^2 \ll n^{-\eps}$,
\cite[Theorem 1.3]{benigni2017eigenvectors} further implies that
$\langle u_i,v_j \rangle$ is approximately Gaussian 
for each index $j$ sufficiently close to $i$, with zero mean and variance
\begin{equation}\label{eq:partialalignment}
\E[\langle u_i,v_j \rangle^2]
\approx \frac{\sigma^2/n}{(\lambda_i-\mu_j)^2+C\sigma^4}.
\end{equation}
Here, the value of $C \asymp 1$ depends on the Wigner semicircle density
near $\lambda_i \approx \mu_j$. Thus, for this range of noise, the eigenvector
$u_i$ of $A$ is most aligned with $O(n\sigma^2)$ eigenvectors $v_j$ of $B$
for which
$|\lambda_i-\mu_j| \lesssim \sigma^2$, and each such alignment is of typical
size $\E[\langle u_i,v_j \rangle^2] \asymp 1/(n\sigma^2) \ll 1$. The signal for
$\pi^*$ in our proposal (\ref{eq:spectralnew})
arises from a weighted average of these alignments.
As a result, while existing spectral approaches are only robust up to a noise
level $\sigma = \frac{1}{\poly(n)}$,\footnote{For the rank-one method
\prettyref{eq:rankone} based on the top eigenvector pairs, a necessary condition
for rounding to succeed is that the two top eigenvectors are perfectly aligned,
i.e., $\Expect[\Iprod{u_1}{v_1}^2]=1-o(1)$. Thus \prettyref{eq:rankone}
certainly fails for $\sigma \gg n^{-1/3}$. For the Umeyama method \prettyref{eq:ume}, experiment shows that it fails when $\sigma \gg n^{-1/4}$; cf.~\prettyref{fig:rate} in \prettyref{sec:numeric-sp}.} 
our new spectral method is polynomially more robust and can tolerate $\sigma= O(\frac{1}{\log n})$.

%\0nb{Here say the low-rank method \prettyref{eq:rankone} that only uses the top eigenvectors  fails at $\sigma=n^{....}$. Explain why. Furthermore, for Umeyama. 
%Instead, }

\paragraph{Cauchy spectral weights.}
The performance of the spectral method depends crucially on the choice of the weight function $w$ in \prettyref{eq:spectralnew}. In fact, there are other methods that are also of the form \prettyref{eq:spectralnew} but do not work equally well. For example, if we choose $w(\lambda,\mu)=\lambda \mu$, then \prettyref{eq:spectralnew} simply reduces to $\hat X = A \allones B = a b^\top$, where $a=A\ones$ and $b=B\ones$ are the vectors of ``degrees''. Rounding such a similarity matrix is equivalent to 
matching by sorting the degree of the vertices, which is known to fail when $\sigma = \Omega(n^{-1})$ due to the small spacing of the order statistics (cf.~\cite[Remark 1]{DMWX18}).

The Cauchy spectral weight \prettyref{eq:cauchykernel} is a particular instance of the more general form $w(\lambda,\mu) = K(\frac{|\lambda-\mu|}{\eta})$, where $K$ is a monotonically decreasing kernel function and $\eta$ is a bandwidth parameter. Such a choice upweights the eigenvector pairs whose eigenvalues are close and significantly penalizes those whose eigenvalues are separated more than $\eta$.
%, which is consist with the behavior of the corresponding eigenvector alignment. 
The specific choice of the Cauchy kernel matches the form of $\E[\langle u_i,v_j \rangle^2]$ in (\ref{eq:partialalignment}), and 
is in a sense optimal as explained by a heuristic signal-to-noise calculation in Appendix \ref{appendix:SNR}.
%we motivate this further by a heuristic signal-to-noise calculation in Appendix \ref{appendix:SNR}.
In addition, the Cauchy kernel has its genesis as a regularization term in the associated convex relaxation, which we explain next.

%\item \emph{Cauchy spectral weights.} 

%\end{enumerate}

\subsection{Connections to regularized quadratic programming}
	\label{sec:convex}

Our new spectral method is also rooted in optimization, as the similarity matrix $\widehat{X}$ in (\ref{eq:spectralnew}) corresponds to the solution to a convex relaxation of the
QAP (\ref{eq:QAP}), regularized by an added ridge penalty.
	
Denote the set of permutation matrices in $\R^{n \times n}$
by $\fS_n$. Then \eqref{eq:QAP}
may be written in matrix notation as one of the three equivalent
optimization problems
\begin{align}
\max_{\Pi \in \fS_n} \, \langle A , \Pi B \Pi^\top \rangle
\quad \Longleftrightarrow \quad
\min_{\Pi \in \fS_n} \| A - \Pi B \Pi^\top \|_F^2 
\quad \Longleftrightarrow \quad
\min_{\Pi \in \fS_n} \| A \Pi - \Pi B \|_F^2 . 
\label{eq:opt}
\end{align}
Note that the third objective $\| A \Pi - \Pi B \|_F^2$ above
is a convex function in $\Pi$. Relaxing the set of permutations to its convex
hull (the Birkhoff polytope of doubly stochastic matrices)
\begin{equation}\label{eq:birkhoff}
\cB_n \defn \{ X \in \R^{n \times n}: X \bone = \bone, \, X^\top \bone = \bone ,
\, X_{ij} \ge 0 \text{ for all } i,j\},
\end{equation}
we arrive at the quadratic programming (QP) relaxation
\begin{align}
\min_{X \in \cB_n} \| A X - X B \|_F^2, \label{eq:ds}
\end{align}
which was proposed in \cite{zaslavskiy2008path,aflalo2015convex}, following an earlier LP relaxation using the $\ell_1$-objective proposed in \cite{almohamad1993linear}.
%\footnote{A LP relaxation using the $\ell_1$ objective was proposed earlier in \cite{almohamad1993linear}.}
Although this QP relaxation has achieved empirical success
\cite{aflalo2015convex, vogelstein2015fast,lyzinski2016graph,dym2017ds++}, understanding its performance theoretically is a challenging task yet to be accomplished.
%theoretical understanding of its performance has been limited.

Our spectral method can be viewed as the solution of a \emph{regularized} further relaxation of the doubly stochastic QP \prettyref{eq:ds}.
Indeed, we show in Corollary~\ref{cor:sol} that the matrix $\widehat{X}$
in (\ref{eq:spectralnew}) is the minimizer of
\begin{align}
\min_{X \in \R^{n \times n}}
\frac{1}{2} \|AX-XB\|_F^2 + \frac{\ridge }{2} \|X\|_F^2-
\bone^\top X \bone.
\label{eq:x-est}
\end{align}
Equivalently, $\hat X$ is a positive scalar multiple of the solution $\widetilde{X}$ to
\begin{align}
\min_{X \in \R^{n \times n}} & ~ 
\| A X - X B \|_F^2 + \ridge \|X\|_F^2 \nonumber \\
\text{s.t.} & ~ \bone^\top X \bone=n
\label{eq:constrained}
\end{align}
which further relaxes (\ref{eq:ds}) and adds a
ridge penalty term $\eta^2\|X\|_F^2$. 
%Note that multiplying $\widehat{X}$ by any
%positive scalar does not change the permutation estimate in the rounding step
%(\ref{eq:linearassignment}), so it is equivalent to consider
%$\widehat{X}$ or $\widetilde{X}$. 
Note that $\widehat{X}$ and $\tilde{X}$ are equivalent as far as the rounding step (\ref{eq:linearassignment}) or \prettyref{eq:greedyround} is concerned.
In contrast to (\ref{eq:ds}), for which
there is currently limited theoretical understanding, we are able to provide an
exact recovery 
%\0nbr{JX. A minor comment. Here readers may mistakenly think that we prove the solutions exactly recoves the true permutation.}
analysis for the rounded solutions to (\ref{eq:x-est}) and (\ref{eq:constrained}).
%\0nb{ZF: I added the word ``rounded'' here to clarify.}

%\0nb{ZF: Previously there was a paragraph here motivating the ridge penalty from
%the partial expectation over $Z$. I understood the calculation but did not
%understand why this motivates adding such a penalty term in the algorithm.
%If there is some intuitive explanation, we can put this back in.}

Note that the constraint in \prettyref{eq:constrained} is a significant
relaxation of the double stochasticity \prettyref{eq:ds}. To make this further relaxed program work, the regularization term plays a key role. If $\eta$ were zero, the similarity matrix $\widehat{X}$ in \prettyref{eq:spectralnew} would involve the eigengap $|\lambda_i - \mu_j|$ in the denominator which can be polynomially small in $n$. Hence the regularization is crucial for reducing the variance of the estimate and making $\widehat{X}$ stable, a rationale  reminiscent of the ridge regularization in high-dimensional linear regression.
In a companion paper \cite{FMWX19b}, we analyze
a tighter relaxation than \prettyref{eq:constrained} which replaces $\ones^\top
X\ones=n$ by the row-sum constraint $X\ones =\ones$, and there the ridge penalty
is still indispensable for achieving exact recovery up to noise level
$\sigma = O(1/\polylog(n))$.
%In the companion paper \cite{FMWX19b}, we are able to show a tighter relaxation than \prettyref{eq:constrained}, where the total-sum constraint is replaced by the row-sum constraints $X\ones =\ones$, satisfies a similar performance guarantee; nevertheless, the ridge penality is still indispensible.
%This rationale is reminiscent of ridge regularization in high-dimensional linear regression, and furthermore, the $\eta$ term here indeed corresponds to a ridge penalty when we view $\widehat{X}$ as the solution to a regularized quadratic program \eqref{eq:x-est} below. 

Viewing $\widehat{X}$ as the minimizer of \eqref{eq:x-est} provides not only
an optimization perspective, but also an associated gradient descent
algorithm to compute $\widehat{X}$. More precisely, starting from the
initial point $X^{(0)} = 0$ and fixing a step size $\gamma>0$, a
straightforward computation verifies that gradient descent
for optimizing (\ref{eq:x-est}) is given by the dynamics
\begin{equation}
X^{(t+1)} = X^{(t)} - \stepsize \big( A^2 X^{(t)} + X^{(t)} B^2 - 2 A X^{(t)} B
+  \ridge X^{(t)} - \bJ \big).
\label{eq:gd}
\end{equation}
Corollary~\ref{cor:sol} below shows that running gradient descent for
$t=O \big( (\log n)^3 \big)$ iterations suffices to produce a similarity matrix, which, upon rounding, exactly recovers $\pi^*$
with high probability, using $X^{(t)}$ in place of $\widehat{X}$
in (\ref{eq:linearassignment}). Each iteration involves
several matrix multiplication operations with $A$ and $B$, which may be more
efficient and parallelizable than performing spectral decompositions when the
graphs are large and/or sparse.

%\0nb{ZF: I believe the theoretical complexity of eigendecomposition and the above
%multiplication operations are the same in both dense and sparse settings, so
%maybe the brief comment above would suffice.}

\subsection{Diagonal dominance of the similarity matrix}

Equipped with this optimization point of view, we now explain the typical structure of solutions to the above quadratic programs including the spectral similarity matrix \prettyref{eq:spectralnew}. It is well known that even the solution to the most stringent relaxation \prettyref{eq:ds} is \emph{not} the latent permutation matrix, which has been shown in \cite{lyzinski2016graph} by proving that the KKT conditions cannot be fulfilled with high probability.
In fact, a heuristic calculation explains why the solution to
\prettyref{eq:ds} is far from any permutation matrix:
Let us consider the ``population version'' of \prettyref{eq:ds}, where the objective function is replaced by its expectation over the random instances $A$ and $B$.
Consider $\pi^* = \id$ and the Gaussian Wigner model $B = A + \sigma Z$, where $A$ and $Z$ are independent GOE matrices with $\sN(0, \frac{1}{n})$ off-diagonal entries and $\sN(0, \frac{2}{n})$ diagonal entries. 
Then the expectation of the objective function is
\begin{align*}
\E  \| A X - X B \|_F^2 
&= \E \| A X \|_F^2  + \E \| X B \|_F^2 - 2 \E \langle A X , X A \rangle \\
&= (2+\sigma^2) \frac{n+1}{n} \|X\|_F^2 - \frac{2}{n} \Tr ( X )^2 - \frac{2}{n} \langle X, X^\top \rangle .
\end{align*}
Hence the population version of the quadratic program \eqref{eq:ds} is 
\begin{align}
\min_{X \in \cB_n} 
(2+\sigma^2) (n+1) \|X\|_F^2 - 2 \Tr ( X )^2 - 2 \langle X, X^\top \rangle , \label{eq:ds-pop}
\end{align}
whose solution\footnote{In fact, \prettyref{eq:sol-struc} is the solution to \prettyref{eq:ds-pop} even if the constraint is relaxed to $\ones^\top X \ones =n$.} is
\begin{align}
%\overline{X} = \frac{ 2 }{ 2 + (n+1) \sigma^2 } \bI + \frac{ (n+1) \sigma^2 }{ 2 n + n (n+1) \sigma^2 } \bJ .
\overline{X} \triangleq \epsilon \bI + (1-\epsilon) \flat,
\qquad \epsilon = \frac{ 2 }{ 2 + (n+1) \sigma^2 } \approx \frac{2}{n \sigma^2}.
\label{eq:sol-struc}
\end{align}
%where $\epsilon \triangleq \frac{ 2 }{ 2 + (n+1) \sigma^2 } \approx \frac{2}{n \sigma^2}$  and 
%$\flat  = \frac{1}{n} \allones$ is the center of the Birkhoff polytope. 
This is a convex combination of the true permutation matrix and the center of the Birkhoff polytope $\flat  = \frac{1}{n} \allones$.
%Therefore the population solution $\overline{X}$ is in fact very close to the center of the Birkhoff polytope and far away from any of its vertices (permutation matrices).
Therefore, the population solution $\overline{X}$ is in fact a very ``flat''
matrix, with each entry on the order of $\frac{1}{n}$, and is close to the center of the Birkhoff polytope and far from any of its vertices.

This calculation nevertheless provides us with important structural information about the solution to such a QP relaxation:  $\overline{X}$ is \emph{diagonally dominant} for small $\sigma$, with diagonals about $2/\sigma^2$ times the off-diagonals.  
%$\overline{X}$ is flat (all entries of order $1/n$) but \emph{diagonal dominant} (with diagonals about $2/\sigma^2$ times bigger than off-diagonals).  
%In other words, while the solution is quite flat, it exhibits \emph{diagonal dominance} for small $\sigma$ (when $\pi^*$ is the identity).
Although the actual solution of the relaxed program \eqref{eq:ds} or
\eqref{eq:x-est} is not equal to the population solution $\overline{X}$ in
expectation, it is reasonable to expect that it inherits the diagonal dominance
property in the sense that $\widehat{X}_{i,\pi^*(i)}>\widehat{X}_{ij}$ for all
$j \neq \pi^*(i)$, which enables rounding procedures such as (\ref{eq:linearassignment}) to succeed.
%In particular, the smaller $\sigma \ge 0$ is, the stronger the diagonal signal is compared to the off-diagonal part. In fact, this heuristic is made rigorous for the program~\eqref{eq:x-est}, where we can provably establish diagonal dominance of the solution when $\sigma$ is sufficiently small.

%It is important to note that the minimizers of any of the above relaxations
%do not need to be the exact permutation matrix which solves the original QAP
%(\ref{eq:opt}), in order for the subsequent rounding in
%(\ref{eq:linearassignment}) to recover $\pi^*$.
%Indeed, $\widehat{X}$ minimizing (\ref{eq:x-est}) is typically
%not a permutation matrix, and we will establish instead
%the ``diagonal dominance'' property
%$\widehat{X}_{i,\pi^*(i)}>\widehat{X}_{ij}$ for all $j \neq \pi^*(i)$.

With this intuition in mind, let us revisit the regularized quadratic program \eqref{eq:x-est} whose solution is the spectral similarity matrix \prettyref{eq:spectralnew}. By a similar calculation, the solution to the population version of \eqref{eq:x-est} is given by
$\alpha \identity + \beta \allones$, with $\alpha = \frac{2 n^2}{\left(n \left(\eta ^2+\sigma ^2\right)+\sigma ^2\right) \left(n \left(\eta ^2+\sigma ^2+2\right)+\sigma ^2\right)} \approx \frac{2}{\left(\eta ^2+\sigma ^2\right) \left(\eta ^2+\sigma ^2+2\right)} $ and $\beta  = \frac{n}{n (\eta ^2+\sigma ^2+2)+\sigma ^2} \approx \frac{1}{\eta ^2+\sigma ^2 + 2}$, which is diagonally dominant for small $\sigma$ and $\eta$. 
%\0nb{YW: someone please double check this}
In turn, the basis of our theoretical guarantee is to establish the diagonal dominance of the actual solution $\hat X$; see \prettyref{fig:diag-dom} for an empirical illustration.

Although the ridge penalty $\eta^2 \|X\|_F^2$ guarantees the stability of the solution as discussed in \prettyref{sec:convex}, it may seem counterintuitive since it moves the solution closer to the center of the Birkhoff polytope and further away from the vertices (permutation matrices). 
In fact, several works in the literature \cite{fogel2013convex, dym2017ds++} advocate adding a negative ridge penalty, in order to make the solution closer to a permutation at the price of potentially making the optimization non-convex. 
This consideration, however, is not necessary, as the ensuing rounding step can automatically map the solution to the correct permutation, even if they are far away in the Euclidean distance.
%As suggested by our algorithm and theory, the ridge term not only provides stability, but also determines the order of diagonal dominance as we shall see. 
%However, at least in our case, the ridge term not only provides stability, but also determines the order of diagonal dominance as we shall see. 
%In short, a flat yet diagonally dominant solution is actually very informative as the permutation can be recovered by downstream rounding. 

\begin{figure}[!ht]
\begin{subfigure}{0.48\textwidth}
%\fbox{
\includegraphics[clip, trim=4.0cm 8.5cm 4.0cm 9.0cm, width=\textwidth]{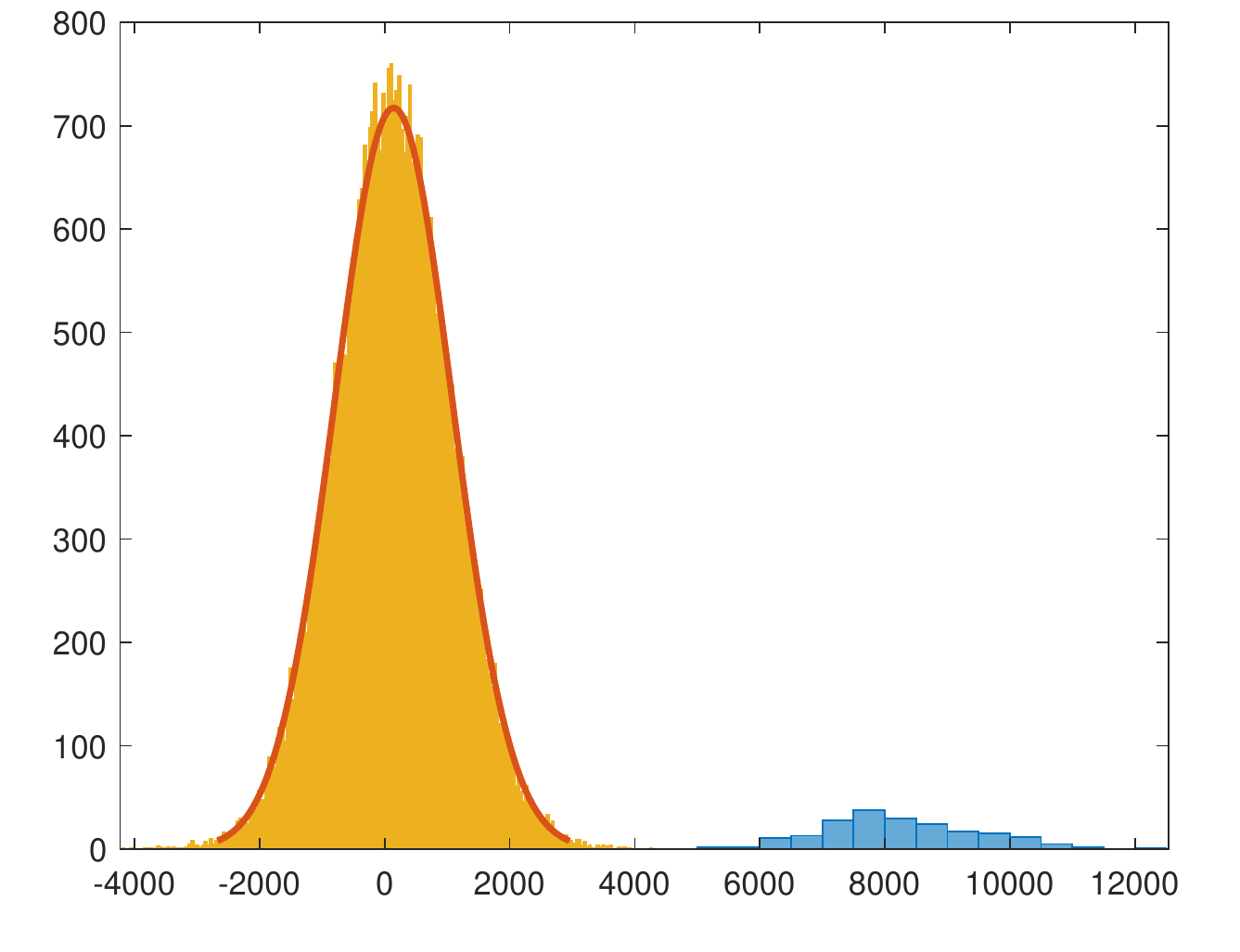}
%}
\caption{Histogram of diagonal (blue) and off-diagonal (yellow with a normal fit) entries of $\widehat{X}$.}
\label{fig:histogram}
\end{subfigure}
\hspace{0.5cm}
\begin{subfigure}{0.48\textwidth}
%\fbox{
\includegraphics[clip, trim=4.0cm 8.5cm 4.0cm 9.0cm, width=\textwidth]{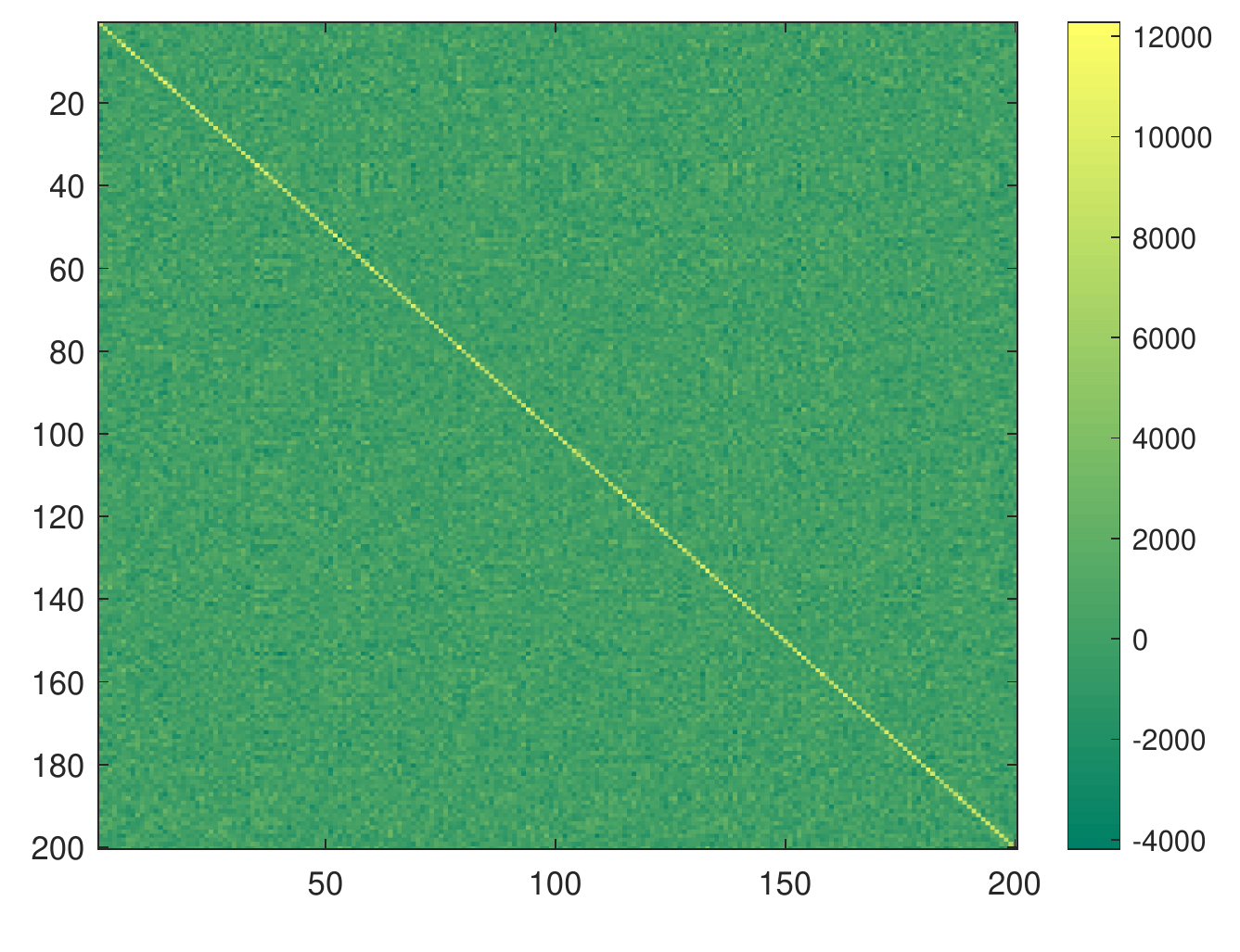}
%}
\caption{Heat map of $\widehat{X}$. \\ \ }
\label{fig:heatmap}
\end{subfigure}
\caption{Diagonal dominance of the similarity matrix $\widehat{X}$ defined by \eqref{eq:spectralnew} or \eqref{eq:x-est} for the Gaussian Wigner model $B = A + \sigma Z$ with $n = 200$, $\sigma = 0.05$ and $\eta = 0.01$.}
\label{fig:diag-dom}
\end{figure}

%The diagonal dominance of the similarity matrix $\widehat{X}$ is demonstrated in Figure~\ref{fig:diag-dom} for an instance of the Gaussian Wigner model.
%%The matrix $\widehat{X}$ defined by \eqref{eq:spectralnew} is computed for an instance of the Gaussian Wigner model $B = A + \sigma Z$ with $n = 200$, $\sigma = 0.05$ and $\eta = 0.01$. 
%As shown in On the left we plot a histogram of the diagonal entries of $\widehat{X}$ in blue and a histogram of the off-diagonal entries in yellow fit by a normal distribution. 
%On the right we plot the heat map of the entries of $\widehat{X}$. 
%From both plots we clearly observe the diagonal dominance of the similarity matrix $\widehat{X}$. 
%Systematic numerical experiments can be found in Section~\ref{sec:numeric}.

It is worth noting that, in contrast to the prevalent analysis of convex
relaxations in statistics and machine learning (where the goal is to show that
the solution to the relaxed program is close to the ground truth in a certain
distance) or optimization (where the goal is to bound the gap of the objective
value to the optimum), here our goal is not to show that
the optimal solution \emph{per se} constitutes a good estimator, but to show
that it exhibits a diagonal dominance structure, which guarantees the success of the subsequent rounding procedure. For this reason, it is unclear from first principles that the guarantees obtained for one program, such as \prettyref{eq:constrained}, automatically carry over to a tighter program, such as \prettyref{eq:ds}. 
%While we believe that the solution to \eqref{eq:ds} may also exhibit diagonal dominance under our model, its quantitative behavior may be different from the solution to \eqref{eq:x-est} that we analyze. 
In the companion paper \cite{FMWX19b}, we analyze a tighter relaxation than
\prettyref{eq:constrained}, where the constraint is tightened to $X\ones
=\ones$, and show that this has a similar performance guarantee; however, this requires a different analysis.

%
%since we are showing a certain structure
%of the solution to a convex program rather than bounding the optimum value, the guarantee we obtain
%for one program does \emph{not} necessarily lead to the same conclusion for a
%tighter program. While we believe that the solution to \eqref{eq:ds} may also
%exhibit diagonal dominance under our model, its quantitative behavior may be
%different from the solution to \eqref{eq:x-est} that we analyze. 

\subsection{Notation}
%\0nb{CM: I think ``notation" is preferred over ``notations" in such a title as it's a mass noun.} 
Let $[n] \defn \{1, \dots, n\}$. 
We use $C$ and $c$ to denote universal constants that may change from line to line. 
For two sequences $\{a_n\}_{n=1}^\infty$ and $\{b_n\}_{n=1}^\infty$ of real numbers, we write
$a_n \lesssim b_n$ if there is a universal constant $C$ such that $a_n
\leq C b_n$ for all $n \geq 1$. The relation $a_n \gtrsim b_n$ is
defined analogously. We write $a_n \asymp b_n$ if both the relations $a_n \lesssim b_n$ and $a_n \gtrsim b_n$ hold. 
Let $x \vee y=\max(x,y)$. Let $\id$ denote the
identity permutation, i.e.~$\id(i)=i$ for every
$i$.
In a Euclidean space $\R^n$ or $\C^n$, let $\coord_i$ be the $i$-th standard
basis vector, $\bone = \bone_n$ the all-ones vector, $\bJ = \bJ_n$ the $n \times
n$ all-ones matrix, and $\bI = \bI_n$ the $n \times n$ identity matrix. We omit
the subscripts when there is no ambiguity. Let $\|\cdot\|=\|\cdot\|_2$ denote
the Euclidean vector norm on $\R^n$ or $\C^n$. Let
$\|M\|=\max_{v \neq 0} \|Mv\|_2/\|v\|_2$ denote the Euclidean operator norm,
$\|M\|_F=(\Tr M^*M)^{1/2}$ the Frobenius norm, and
$\|M\|_\infty=\max_{ij} |M_{ij}|$ the elementwise
$\ell_\infty$ norm of a matrix $M$. An
eigenvector is always a unit column vector by convention.
Denote by $X \eqdistr Y$ if random variables $X$ and $Y$ are equal in law.

\section{Main results}
\label{sec:gw}

In this section, we formulate the models, state more precisely
our main results, provide an outline of the proof, and discuss the extension to
bipartite graphs.

\subsection{Gaussian Wigner model}

We say that $A \in \R^{n \times n}$ is from the Gaussian Orthogonal
Ensemble, or simply $A \sim \GOE$, if $A$ is symmetric,
$\{A_{ij}:i \leq j\}$ are independent, and $A_{ij} \sim \sN(0, \frac{1}{n})$
for $i \neq j$ and $\sN(0, \frac{2}{n})$ for $i=j$.
We say that the pair $A,B \in \R^{n \times n}$
follows the \emph{Gaussian Wigner model} for graph matching if
\begin{align}
B^{\pi^*} = A + \sigma Z , 
\label{eq:model}
\end{align}
where $\pi^*$ is an unknown permutation (ground truth),
$B^{\pi^*}$ denotes the permuted matrix  $B^{\pi^*}_{ij}=B_{\pi^*(i),\pi^*(j)}$,
the matrices $A,Z \sim \GOE$ are independent, and $\sigma \geq 0$ is the noise
level. Our goal is to recover the latent permutation $\pi^*$ from $(A,B)$.

We may also consider the rescaled definition
$B^{\pi^*}=\sqrt{1-\sigma^2} A + \sigma Z $, so that $A$
and $B$ have the same marginal law with correlation coefficient $1-\sigma^2$. Our proofs are easily adapted to this
setup, and we assume \eqref{eq:model} for simplicity and a
cleaner presentation. 

We now formalize the exact recovery guarantee for \prettyref{alg:grampa}.

\begin{theorem} \label{thm:wigner}
Consider the model \eqref{eq:model}. There exist constants $c,c'>0$ such that if
\[1/n^{0.1} \leq \eta \leq c/\log n \quad \text{ and } \quad
    \sigma \leq c'\eta,\]
then with probability at least $1-n^{-4}$ for all large $n$, the matrix
$\widehat{X}$ in (\ref{eq:spectralnew}) satisfies
\begin{equation}\label{eq:diagonaldominance}
    \min_{i \in [n]} \widehat{X}_{i, \pi^*(i)} >
    \max_{i, j \in [n]: \, j \ne \pi^*(i)} \widehat{X}_{ij}
\end{equation}
and hence, \prettyref{alg:grampa} recovers $\hat{\pi}=\pi^*$.
\end{theorem}

Choosing $\eta \asymp 1/\log n$ in \prettyref{alg:grampa},
we thus obtain exact recovery for $\sigma \lesssim 1/\log n$.
The same exact recovery guarantee clearly also holds if
rounding were performed by the simple scheme (\ref{eq:greedyround}),
instead of solving the linear assignment (\ref{eq:linearassignment}). The
probability $n^{-4}$ is arbitrary and may be strengthened to $n^{-D}$
for any constant $D>0$, where $c,c'$ above depend on $D$.

Consider next the gradient descent iterates $X^{(t)}$ defined by \eqref{eq:gd}.
We verify that these iterates converge to $\widehat{X}$, and that the same
guarantee holds for $X^{(t)}$ for sufficiently large $t$.

\begin{corollary} \label{cor:sol}
Let $X^{(0)}=0$, define recursively $X^{(t)}$ by the gradient descent dynamics
    (\ref{eq:gd}), and let $\widehat{X}$ be the similarity matrix
    (\ref{eq:spectralnew}).
\begin{enumerate}[(a)]
    \item The matrix $\widehat{X}$ is the minimizer of the unconstrained program
        (\ref{eq:x-est}), and $\alpha \widehat{X}$ is the minimizer of the
        constrained program (\ref{eq:constrained})
        for some (random) scalar multiplier $\alpha>0$.
    \item In terms of the spectral decompositions of $A$ and $B$ in
        (\ref{eq:spectraldecomp}), each iterate $X^{(t)}$ is given by
        \[X^{(t)} = \sum_{i, j = 1}^n \frac{ 1 - [ 1 - \stepsize \ridge - \stepsize (
    \lambda_i - \mu_j )^2 ]^t }{ \ridge + ( \lambda_i - \mu_j )^2 } u_i u_i^\top
        \bJ v_j v_j^\top.\]
In particular, if the step size satisfies 
$\stepsize<1/(\eta^2+(\lambda_i-\mu_j)^2)$ for all $i,j$, then
$X^{(t)} \to \widehat{X}$ as $t \to \infty$.
    \item In the setting of Theorem \ref{thm:wigner}, for some constants
        $C,c>0$, if $\stepsize<c$ and
        \[t>\frac{C\log n}{\stepsize \eta^2},\]
        then the guarantees of Theorem \ref{thm:wigner} also hold with
        probability at least $1-n^{-4}$ with $X^{(t)}$ in place of $\widehat{X}$.
\end{enumerate}
\end{corollary}

In particular, setting $\stepsize$ to be a small constant and
$\eta \asymp 1/\log n$, we obtain the same
diagonal dominance property for $X^{(t)}$ as long as $t \gtrsim (\log n)^3$, and
consequently either the rounding scheme \prettyref{eq:linearassignment} or
\prettyref{eq:greedyround} applied to
$X^{(t)}$ recovers the true matching $\pi^*$. 

\subsection{Proof outline for Theorem~\ref{thm:wigner}}\label{sec:proofideas}

We give an outline of the proof for Theorem \ref{thm:wigner}.
By the permutation invariance property \prettyref{eq:invariance} of the algorithm, we may assume without loss of
generality that $\pi^*=\id$, the identity permutation.
Then we must show in (\ref{eq:diagonaldominance}) that every diagonal entry
of $\widehat{X}$ is larger than every off-diagonal entry.

%Assuming $\pi^*=\id$, 
Denote the similarity matrix in \prettyref{eq:spectralnew} by
\begin{equation}
X=\widehat{X}(A,B)= \sum_{i,j=1}^n
\frac{1}{(\lambda_i-\mu_j)^2+\eta^2}u_iu_i^\top \bJ v_jv_j^\top
\label{eq:X}
\end{equation}
 and
introduce $X_*$ as the similarity matrix constructed in the noiseless setting with $A$ in place of
$B$. That is,
\begin{equation}
X_* =\widehat{X}(A,A) =\sum_{i,j=1}^n
\frac{1}{(\lambda_i-\lambda_j)^2+\eta^2}u_iu_i^\top \bJ u_ju_j^\top.
\label{eq:Xstar}
\end{equation}
We divide the proof into establishing the diagonal dominance of $X_*$,
and then bounding the entrywise difference $X-X_*$.

\begin{lemma}\label{lem:noiseless}
    For some constants $C,c>0$, if $1/n^{0.1} < \eta < c/\log n$, then
    with probability at least $1-5n^{-5}$ for large $n$,
    \[\min_{i \in [n]} (X_*)_{ii}>1/(3\eta^2)\]
    and
    \begin{equation}\label{eq:offdiag}
        \max_{i,j \in [n]:i \neq j} (X_*)_{ij}<
    C\left(\frac{\sqrt{\log n}}{\eta^{3/2}}+\frac{\log n}{\eta}\right).
    \end{equation}
\end{lemma}
\begin{lemma}\label{lem:noise}
    If 
		%$\pi^*=\id$ and 
		$\eta>1/n^{0.1}$, then for a constant $C>0$,
    with probability at least $1-2n^{-5}$ for large $n$,
    \begin{equation}\label{eq:noise}
        \max_{i,j \in [n]} |X_{ij}-(X_*)_{ij}|
    <
%    C\sigma \left(\frac{1}{\eta^3}+\frac{\log n}{\eta^2}\right)
%    \left(1+\frac{\sigma}{\eta}\right).
    C \sigma \left(\frac{1}{\eta^3}
    +\frac{\log n}{\eta^2}\left(1+\frac{\sigma}{\eta}\right)\right).
    \end{equation}
\end{lemma}

\begin{proof}[Proof of Theorem \ref{thm:wigner}]
Assuming these lemmas, for some $c,c'>0$ sufficiently small,
setting $\eta<c/\log n$ and $\sigma<c'\eta$ ensures that the
right sides of both (\ref{eq:offdiag}) and (\ref{eq:noise}) are at most
$1/(12 \eta^2)$. Then when $\pi^*=\id$, these lemmas combine to imply
\[\min_{i \in [n]} X_{ii}>
\frac{1}{4\eta^2}>\frac{1}{6\eta^2}>\max_{i,j \in [n]:i \neq j} X_{ij}\]
with probability at least $1-n^{-4}$. On this event, by definition,
both the LAP rounding procedure (\ref{eq:linearassignment}) and the simple greedy rounding \prettyref{eq:greedyround} output  $\hat{\pi}=\id$.
The result for general $\pi^*$ follows from the equivariance of the algorithm, 
applying this result to the inputs $A$ and $B^\pi$ with $\pi=(\pi^*)^{-1}$.
%\0nb{ZF: Again, I think this is sufficient, but we can elaborate if it is
%unclear.}
\end{proof}

A large portion of the technical work will lie in establishing
Lemma \ref{lem:noiseless} for the noiseless setting.
We give here a short sketch of the intuition for this lemma,
ignoring momentarily any factors that are logarithmic in $n$ and are hidden by the notations $\approx$ and $\lessapprox$ below.
%\0nb{I removed the $\asymp_L$ notation and just added this comment---is this OK
%by you guys?} 
Let us write
\begin{align}
X_* = \sum_{i=1}^n \frac{ 1 }{ \ridge } (u_i^\top \bJ u_i) u_i  u_i^\top + \sum_{i \ne j} \frac{ 1 }{ \ridge + ( \lambda_i - \lambda_j )^2 } (u_i^\top \bJ u_j) u_i u_j^\top . 
\label{eq:spec-noiseless}
\end{align}
We explain why the first term is diagonally dominant, while the second term is
a perturbation of smaller order. Central to our proof is the fact that $A \sim
\GOE$ is rotationally invariant in law, so that $U=(u_1,\ldots,u_n)$ is 
uniformly distributed on the orthogonal group and
independent of $\lambda_1,\ldots,\lambda_n$. The
coordinates of $U$ are approximately independent with
distribution $\sN(0,\frac{1}{n})$.

For the first term in \prettyref{eq:spec-noiseless}, with high probability $u_i^\top \bJ u_i = \Iprod{u_i}{\ones}^2 \approx 1$ for
every $i$. Then for each $k$, the $k^\text{th}$ diagonal entry of the first term satisfies
\begin{equation}
\sum_{i=1}^n \frac{ 1 }{ \ridge } (u_i^\top \bJ u_i) (u_i)_k^2
\approx \sum_{i=1}^n \frac{ 1 }{ \ridge } (u_i)_k^2 \approx \frac{1}{\ridge}.
\label{eq:spec-noiseless1diag}
\end{equation}
Applying the heuristic $(u_i)_k \sim \sN(0,\frac{1}{n})$,
each $(k,\ell)^\text{th}$ off-diagonal entry satisfies
\begin{equation}
\sum_{i=1}^n \frac{ 1 }{ \ridge } (u_i^\top \bJ u_i) (u_i)_k (u_i)_\ell
\lessapprox \frac{1}{\ridge \sqrt{n} } . 
\label{eq:spec-noiseless1off}
\end{equation}

For the second term in \prettyref{eq:spec-noiseless}, each $(k,\ell)^\text{th}$ entry is
\[
\sum_{i \ne j} \frac{ 1 }{ \ridge + ( \lambda_i - \lambda_j )^2 } (u_i^\top
\bJ u_j) (u_i)_k (u_j)_\ell =  g^\top Q h , 
\]
where $Q$ is defined by $Q_{ii} = 0$ and $Q_{ij} = \frac{ 1 }{
    \ridge + ( \lambda_i - \lambda_j )^2 } (u_i^\top \bJ u_j) $ for $i \ne j$,
and the vectors $g$ and $h$ are defined by $g_i = (u_i)_k$ and $h_j =
(u_j)_\ell$. Applying the heuristic that $g,h$ are approximately iid $\sN(0, \frac 1n \bI)$ and approximately independent of $Q$, we have
a Hanson-Wright type bound
$$
g^\top Q h \lessapprox \frac{1}{ n } \| Q \|_F . 
$$
As $n \to \infty$, the empirical spectral distribution
$n^{-1}\sum_{i=1}^n \delta_{\lambda_i}$ of $A$
converges to the Wigner semicircle law with density $\rho$.
Then, applying also $u_i^\top \bJ u_j = \Iprod{u_i}{\ones} \Iprod{u_j}{\ones}   \lessapprox 1$, we obtain
$$
\frac{1}{ n^2 } \| Q \|_F^2 \lessapprox \frac{1}{ n^2 } \sum_{i \ne j} \Big( \frac{ 1 }{ \ridge + ( \lambda_i - \lambda_j )^2 } \Big)^2
\approx \iint \Big( \frac{ 1 }{ \ridge + ( x - y )^2 } \Big)^2 \rho(x) \rho(y)
dx dy \lessapprox \frac{1}{ \eta^3 },
$$
where the last step is an elementary computation that holds for any bounded
density $\rho$ with bounded support. As a result, each entry of the second term of
(\ref{eq:spec-noiseless}) satisfies
\begin{equation}
\sum_{i \ne j} \frac{ 1 }{ \ridge + ( \lambda_i - \lambda_j )^2 } (u_i^\top \bJ
u_j) (u_i)_k (u_j)_\ell \lessapprox \frac{1}{ \eta^{3/2} }.
\label{eq:spec-noiseless2}
\end{equation}
Combining \prettyref{eq:spec-noiseless1diag}--\prettyref{eq:spec-noiseless2} shows that the noiseless solution $X_*$ in \eqref{eq:spec-noiseless} is indeed diagonally
dominant, with diagonals approximately $\eta^{-2}$ and off-diagonals at most of the order $\eta^{-3/2}$, omitting logarithmic factors. We carry out
this proof more rigorously in Section~\ref{sec:noiseless}
to establish Lemma \ref{lem:noiseless}.

\subsection{Gaussian bipartite model}\label{sec:bipartite}

Consider the following asymmetric variant of this problem:
Let $F,G \in \R^{n \times m}$ be adjacency matrices of two weighted
bipartite graphs on $n$ left vertices and $m$ right vertices, where $m \geq n$ is assumed without loss of generality.
Suppose, for latent permutations
$\pi_1^*$ on $[n]$ and $\pi_2^*$ on $[m]$, that
$\{(F_{ij},G_{\pi_1^*(i),\pi_2^*(j)}):1 \leq i \leq n, 1 \leq j \leq m\}$
are \iid pairs of correlated random variables. We wish to recover
$(\pi_1^*,\pi_2^*)$ from $F$ and $G$.

We propose to apply \prettyref{alg:grampa} on the left singular
values and singular vectors of $F$ and $G$ to first recover the (smaller) row permutation $\pi_1^*$, and
then solve a second LAP to recover the (bigger) column permutation $\pi_2^*$. This is summarized as follows:

%\paragraph{\prettyref{alg:bigrampa}: Matching Bipartite Graphs}
%
%\begin{enumerate}
%\item Construct the similarity matrix $\widehat{X}$ as in
%(\ref{eq:spectralnew}), where now $\lambda_1 \geq \ldots \geq \lambda_n$ and
%$\mu_1 \geq \ldots \geq \mu_n$ are the singular values of $F$ and $G$, and
%$u_i$ and $v_j$ are the corresponding left singular vectors.
%
%\item Let $\hat{\pi}_1$ be the estimate in (\ref{eq:linearassignment}),
%and denote by $G^{\hat{\pi}_1, \id} \in \R^{n \times m}$ the matrix
%$G^{\hat{\pi}_1, \id}_{ij}=G_{\hat{\pi}_1(i),j}$.
%
%\item Estimate $\hat{\pi}_2$ by solving the linear assignment problem
%\begin{align}
%\hat{\pi}_2=\argmax_{\pi \in \cS_m} \sum_{j=1}^m (F^\top
%G^{\hat{\pi}_1})_{j,\pi(j)}.
%\label{eq:lin-step}
%\end{align}
%\end{enumerate}

\begin{algorithm}
\caption{Bi-GRAMPA (Bipartite GRAph Matching by Pairwise eigen-Alignments)}\label{alg:bigrampa}
\begin{algorithmic}[1]
\STATE {\bfseries Input:} $F,G \in \reals^{n\times m}$, a tuning parameter $\eta>0$.
\STATE {\bfseries Output:} Row permutation $\hat\pi_1 \in \cS_n$ and column permutation $\hat\pi_2 \in \cS_m$.
\STATE  Construct the similarity matrix $\widehat{X}$ as in
(\ref{eq:spectralnew}), where now $\lambda_1 \geq \ldots \geq \lambda_n$ and
$\mu_1 \geq \ldots \geq \mu_n$ are the singular values of $F$ and $G$, and
$u_i$ and $v_j$ are the corresponding left singular vectors.

\STATE Let $\hat{\pi}_1$ be the estimate in (\ref{eq:linearassignment}),
and denote by $G^{\hat{\pi}_1, \id} \in \R^{n \times m}$ the matrix
$G^{\hat{\pi}_1, \id}_{ij}=G_{\hat{\pi}_1(i),j}$.

\STATE Find $\hat{\pi}_2$ by solving the linear assignment problem
\begin{align}
\hat{\pi}_2=\argmax_{\pi \in \cS_m} \sum_{j=1}^m (F^\top
G^{\hat{\pi}_1,\id})_{j,\pi(j)}.
\label{eq:lin-step}
\end{align}

\end{algorithmic}
\end{algorithm}

We also establish an exact recovery guarantee for this method in a Gaussian
setting: We say that the pair $F,G \in \R^{\dimone \times \dimtwo}$ follows the
\emph{Gaussian bipartite model} for graph matching if
\begin{align}
    G^{\pi_1^*,\pi_2^*}
    = F + \sigma W , \label{eq:model-wishart}
\end{align}
where $G^{\pi_1^*,\pi_2^*}$ denotes the permuted matrix
$G^{\pi_1^*,\pi_2^*}_{ij}=G_{\pi_1^*(i),\pi_2^*(j)}$,
the matrices $F$ and $W$ are independent with \iid~$\sN(0, \frac{1}{\dimtwo})$
entries, and $\sigma \ge 0$ is the noise level.
We assume the asymptotic regime
\begin{equation}
\dimtwo = \dimtwo (\dimone )
\quad \text{ and } \quad
\dimone/ \dimtwo \to \kappa \in (0,1] \quad \text{ as } \quad
n \to \infty.
\label{eq:bipartite-ass}
\end{equation}

\begin{theorem} \label{thm:bipartite}
Consider the model~\eqref{eq:model-wishart}, where
$\dimone/ \dimtwo \to \kappa \in (0, 1]$. There exist $\kappa$-dependent
constants $c,c'>0$ such that if
\[1/n^{0.1} \leq \eta \leq c/\log n \quad \text{ and } \quad
    \sigma \log(1/\sigma) \leq c'\eta/\log n,\]
then with probability at least $1-n^{-4}$ for all large $n$,
\prettyref{alg:bigrampa} recovers $(\hat{\pi}_1, \hat{\pi}_2)=(\pi_1^*,\pi_2^*)$.
\end{theorem}

\noindent Setting $\eta \asymp 1/\log n$, we obtain exact recovery under the
condition $\sigma \lesssim (\log n)^{-2}(\log\log n)^{-1}$.

The proof is an extension of that of Theorem \ref{thm:wigner}:
Note that the first step of \prettyref{alg:bigrampa} is equivalent to applying \prettyref{alg:grampa} on
the symmetric polar parts $A=\sqrt{FF^\top}$ and $B=\sqrt{GG^\top}$, where
$\sqrt{\cdot}$ denotes the symmetric matrix square root. In the Gaussian
setting, $A$ is still rotationally invariant, and Lemma \ref{lem:noiseless} will
extend directly to $X_*$ constructed from this $A$. We will show a simpler and
slightly weaker version of Lemma \ref{lem:noise} to establish exact recovery
of the left permutation $\pi_1^*$, under the stronger requirement for $\sigma$
above. We will then analyze separately the linear assignment for
recovering $\pi_2^*$. 
%Theorem~\ref{thm:linear} on the recovery guarantee for the LAP may be of independent interest. 
Details of the argument are provided in Section
\ref{sec:bipartiteproof}.

We conclude this section by discussing the assumption
\prettyref{eq:bipartite-ass}. The condition
$n\to\infty$ is information-theoretically
necessary to recover the right permutation $\pi_2^*$, unless $\sigma$ is as
small as $1/\poly(n)$. This can be seen by considering the oracle setting when $\pi_1^*$ is given, in which case the necessary and sufficient condition for the maximal likelihood (linear assignment) to succeed is given by $n \log \left( 1 + \frac{1}{\sigma^2} \right)- 4 \log m \to \infty$ \cite{dai2019database}.
The condition of finite aspect ratio $n=\Theta(m)$ is assumed for the analysis
of the Bi-GRAMPA algorithm; otherwise, if $n=o(m)$, then the empirical
distribution of singular values of $F$ converges to a point mass at $1$, and it is unclear whether the  spectral similarity matrix in \prettyref{eq:X} or \prettyref{eq:Xstar} continues to be diagonally dominant.
We note that such a condition is not information-theoretically necessary. In fact, as long as 
$n$ and $m$ are polynomially related, running the degree profile matching algorithm \cite[Section 2]{DMWX18} on the row-wise and column-wise empirical distributions succeeds for $\sigma = O(\frac{1}{\log n})$.

\section{Proofs}\label{sec:proofs}

We prove our main results in this section. Section
\ref{sec:noiseless} proves Lemma \ref{lem:noiseless}, which shows the diagonal
dominance of $X_*$
in the noiseless setting of $A=B$. Section \ref{sec:noise} proves Lemma
\ref{lem:noise}, which bounds the difference $X-X_*$.
Together with the argument in Section \ref{sec:proofideas},
these yield Theorem \ref{thm:wigner} on the exact recovery in the
Gaussian Wigner model.

Section
\ref{sec:bipartiteproof} extends this analysis to establish Theorem
\ref{thm:bipartite} for the bipartite model. Finally, Section \ref{sec:gd}
(which may be read independently)
proves Corollary \ref{cor:sol} relating $\widehat{X}$ to the gradient descent
dynamics (\ref{eq:gd}) and the optimization problems (\ref{eq:x-est}) and
(\ref{eq:constrained}).

\subsection{Analysis in the noiseless setting}\label{sec:noiseless}
We first prove Lemma \ref{lem:noiseless}, showing diagonal dominance in the
noiseless setting. Throughout, we write the spectral decomposition 
\begin{align}
A = U \Lambda U^\top
\quad \text{ where } \quad U = [u_1 \cdots u_n] \text{ and } \Lambda = \mathsf{diag} ( \lambda_1, \dots, \lambda_n ) . \label{eq:a-dec}
\end{align}

\subsubsection{Properties of $A$ and rotation by $U$}

In the proof, we will in fact only use the properties of the matrix
$A \sim \GOE$ recorded in the following proposition. The same proof will then apply to the
bipartite case wherein the suitably defined $A$ satisfies the same properties.
%\0nb{ZF: I removed the previous structure which is to state a general theorem and
%then apply it directly to the Wigner and bipartite cases.
%The reason is because our
%analysis of noise is now different in these two cases---it is sharper for the
%Wigner setting---so that the results are also a little different.}

\begin{proposition}\label{prop:wignerproperties}
Suppose $A \sim \GOE$. Then for constants $C,c>0$,

\begin{enumerate}[(a)]
\item 
%Let $A=U\Lambda U^\top$. Then 
    $U$ is a uniform random orthogonal matrix independent of $\Lambda$.

\item The empirical spectral distribution $\rho_n = \frac 1n \sum_{i=1}^n
\delta_{\lambda_i}$ of $A$ converges to a limiting law $\rho$, which has 
a density function bounded by $C$ and support contained in
$[-C, C]$. Moreover, for all large $n$,
$$
\p \Big\{\sup_x |F_n(x) - F(x)|>Cn^{-1}(\log n)^5 \Big\}<n^{-c\log\log n} ,
$$
%\0nbr{JX. Here the constant $C$ in the probability may be confused with 
%the bound on the spectrum $C$.}
%\0nb{CM: not a problem even if the constants are the same?}
where $F_n$ and $F$ are the cumulative distribution functions of $\rho_n$ and
$\rho$ respectively. 

\item For all large $n$, $\p \big\{\|A\|>C\}<e^{-cn}$.

\end{enumerate}
\end{proposition}
\begin{proof}
    Parts (a) and (c) are well-known, see for example
    \cite[Corollary 2.5.4 and Lemma 2.6.7]{AGZbook}.
    For (b), $\rho$ is the Wigner
    semicircle law on $[-2,2]$, and the rate of convergence follows
    from \cite[Theorem 1.1]{GotTik13}. 
\end{proof}

Recall the definition
\[X_* = \sum_{i, j = 1}^n \frac{ 1 }{ \ridge + ( \lambda_i - \lambda_j )^2 } u_i
u_i^\top  \bJ u_j u_j^\top .\]
Our goal is to exhibit the diagonal dominance of this matrix. Without loss of
generality, we analyze $(X_*)_{11}=\coord_1^\top X_* \coord_1$
and $(X_*)_{12}=\coord_1^\top X_* \coord_2$.

Applying \prettyref{prop:wignerproperties} (a) above, let us rotate by $U$ to
write the quantities of interest in a more convenient form. Namely, we set 
\begin{align}
\phi = U^\top \coord_1, \quad \psi = U^\top \coord_2, \quad \text{and} \quad \xi
    = U^\top \bone. \label{eq:ppx}
\end{align} 
These vectors satisfy
\begin{align}
\| \phi \|_2 = \| \psi \|_2 = 1, \quad  \| \xi \|_2 = \sqrt{n}, \quad  \langle
\phi, \xi \rangle = 1, \quad  \langle \psi, \xi \rangle = 1 , \quad
\text{and} \quad \langle \phi, \psi \rangle = 0, \label{eq:def-3}
\end{align}
and are otherwise ``uniformly random''. By this, we mean that $(\phi,\psi,\xi)$
is equal in law to $(O\phi,O\psi,O\xi)$ for any
orthogonal matrix $O \in \R^{n \times n}$, which follows from Proposition
\ref{prop:wignerproperties}(a).

Define a symmetric matrix $L \in \R^{n \times n}$ by
\begin{align}
L_{ij} = \frac{ 1 }{ \ridge + ( \lambda_i - \lambda_j )^2 },\label{eq:l-def}
\end{align}
and define $\tL \in \R^{n \times n}$ such that
$\tL_{ii}=0$ and $\tL_{ij} = L_{ij}$ for $i \ne j$. Then
\begin{align}
(X_*)_{12}=\sum_{i,j=1}^n L_{ij} \phi_i\psi_j\xi_i\xi_j,\label{eq:summed}
\end{align}
and
\begin{align}
(X_*)_{11}
= \frac{ 1 }{ \ridge } 
\sum_{i = 1}^n \phi_i^2\xi_i^2
+ \sum_{i,j=1}^n \tL_{ij} \phi_i\phi_j\xi_i\xi_j. \label{eq:diag-summed}
\end{align}
Importantly, by Proposition \ref{prop:wignerproperties}(a), the triple
$(\phi,\psi,\xi)$ is independent of $L$ and $\tL$. We will establish the
following technical lemmas.

\begin{lemma} \label{lem:diag-dominate}
With probability at least $1- 3 n^{-7}$ for large $n$,
$$
\sum_{i=1}^n \phi_i^2 \xi_i^2>\frac {1}{2} .
$$
\end{lemma}

\begin{lemma} \label{lem:sig-conc}
For some constants $C,c>0$, if $1/n^{0.1}<\eta<c$, then
with probability at least $1 - 2 n^{-7}$ for large $n$,
\begin{align}
\Big| \sum_{i, j = 1}^n L_{ij} \phi_i \psi_j \xi_i \xi_j \Big| \lor
\Big| \sum_{i, j = 1}^n \tL_{ij} \phi_i \phi_j \xi_i \xi_j \Big|<
    C\left(\frac{\sqrt{\log n} }{\eta^{3/2}} + \frac{ \log n}{ \eta}\right).
\end{align}
\end{lemma}

Lemma \ref{lem:noiseless} follows immediately from these results.
Indeed, for $\eta<c/\log n$ and sufficiently small $c>0$,
these results and the forms (\ref{eq:summed}--\ref{eq:diag-summed})
combine to yield $(X_*)_{11}>1/(3\eta^2)$ and $(X_*)_{12}<C(
\sqrt{\log n}/\eta^{3/2}+(\log n)/\eta)$ with probability at least $1-5n^{-7}$.
By symmetry,
the same result holds for all $(X_*)_{ii}$ and $(X_*)_{ij}$, and Lemma
\ref{lem:noiseless} follows from taking a union bound over $i$ and $j$.

It remains to show Lemmas \ref{lem:diag-dominate} and \ref{lem:sig-conc}.
The general strategy is to approximate the law of $(\phi,\psi,\xi)$ by 
suitably defined Gaussian random vectors, and then to apply Gaussian tail bounds
and concentration inequalities which are collected in Appendix
\ref{appendix:gaussian}. As an intermediate step, we will show the
following estimates for the matrix $L$, using the convergence
of the empirical spectral distribution in Proposition
\ref{prop:wignerproperties}(b).

\begin{lemma} \label{lem:l-fro}
    For constants $C,c>0$, with probability at least $1-n^{-10}$ for large $n$,
    \[\min_{i, j \in [n]} L_{ij} \ge c, \quad  \max_{i,j \in [n]} L_{ij}
    \le \frac{1}{\ridge} , \quad \frac{1}{n}\|L\|_F \leq \frac{C}{\eta^{3/2}},\]
\[ 
\frac 1n \max_{i \in [n]} \sum_{j=1}^n L_{ij}^2 \le \frac{C}{\eta^3} 
\quad \text{ and } \quad
    \frac 1n \max_{i \in [n]} \sum_{j=1}^n L_{ij} \le \frac{C}{\eta}.\]
\end{lemma}
%\0nb{ZF: I added an $\ell_2$ bound on each row, which implies the one on the
%Frobenius norm, to shorten the proof.}

\subsubsection{Proof of Lemma~\ref{lem:diag-dominate}}

Let  $z$ be a standard Gaussian vector in $\R^n$ independent of $\phi$. 
First, we note that marginally $\phi$ is equal to $z/\|z\|_2$ in law.
 By standard bounds on $\max_j |z_j|$
and $\|z\|_2$ (see Lemmas~\ref{lem:g-concentrate} and \ref{lem:g-norm}), we have that with probability at least $1 - 2 n^{-7}$,
\begin{align}
\max_{i \in [n]} \big| \phi_i \big| \le 5 \left( \frac{\log n}{n} \right)^{1/2} .
\label{eq:sup-bd}
\end{align}

Next, the random vectors $\phi$ and $\xi$ satisfy that $\|\phi\|_2 = 1$, $\|\xi\|_2 = \sqrt{n}$ and $\langle \phi, \xi \rangle = 1$, and are otherwise uniformly random. Hence if we let $z$ be a standard Gaussian vector in $\R^n$ and define
$$
\tilde{\xi} \defn \sqrt{n-1} \, \frac{ z - (\phi^\top z) \phi }{ \big\| z -
(\phi^\top z) \phi \big\|_2}  + \phi,
$$
then $(\phi, \xi)\eqdistr(\phi, \tilde{\xi})$. Note that we can
write $\tilde{\xi} = \alpha z + \beta \phi$, where $\alpha$ and $\beta=1-\alpha (\phi^\top z)$ are
random variables satisfying $0.9 \le \alpha \le 1.1$ and $|\beta| \le 4
\sqrt{\log n}$ with probability at least $1 - 4 n^{-8}$ by concentration of $\|z\|_2$ and $\phi^\top z$ (Lemmas~\ref{lem:g-concentrate} and~\ref{lem:g-norm}).
Therefore, we obtain
\begin{align}
\sum_{i=1}^n \phi_i^2 \tilde{\xi}_i^2
%& =  \sum_{i=1}^n \phi_i^2 \left( \alpha z_i + \beta \phi_i \right)^2  \\
= \sum_{i=1}^n \phi_i^2 \left( \alpha^2 z_i^2 + \beta^2 \phi_i^2 + 2 \alpha \beta z_i \phi_i \right)  
\ge 0.8 \sum_{i=1}^n \phi_i^2 z_i^2  - 9 \sqrt{\log n} \left| \sum_{i=1}^n \phi_i^3 z_i \right| . \label{eq:three-terms}
\end{align}
%\0nbr{JX. I am not sure how to get the last inequality. First, since $\beta^2 \phi_i^2 \ge0$, we can always throw it off in the lower bound. Second, we can bound $2\alpha \beta$ by $8.8\sqrt{\log n}$.}
For the first term of~\eqref{eq:three-terms}, applying Lemma~\ref{lem:g-norm} and then~\eqref{eq:sup-bd} yields
$$
\sum_{i=1}^n \phi_i^2 z_i^2  \ge 1 - 22 (\log n) \left( \sum_{i=1}^n \phi_i^4 \right)^{1/2} \ge 1 - 22 (\log n) \left( \sum_{i=1}^n 5^4 \left(\frac{\log n}{n} \right)^2 \right)^{1/2} \ge 0.9
$$
%\0nbr{JX. It seems that $\left( \sum_{i=1}^n \phi_i^4 \right)^2$ should be 
%$\left( \sum_{i=1}^n \phi_i^4 \right)^{1/2}$.}
with probability at least $1 - 3 n^{-7}$. 
%An upper bound of $0.1$ on the second term of~\eqref{eq:three-terms} follows similarly from~\eqref{eq:sup-bd}. 
For the second term of~\eqref{eq:three-terms}, we once again apply Lemma~\ref{lem:g-concentrate} and then~\eqref{eq:sup-bd} to obtain
$$
9 \sqrt{\log n} \Big| \sum_{i=1}^n \phi_i^3 z_i \Big| \le 20 \log n \left( \sum_{i=1}^n \phi_i^6 \right)^{\frac 12} \le 0.1 
$$
with probability at least $1 - 3 n^{-7}$. Combining the three terms finishes the proof.  

\subsubsection{Proof of Lemma \ref{lem:l-fro}}

Let $\rho_n$ be the empirical spectral distribution of $A$.
For a large enough constant $C_1>0$ where $[-C_1,C_1]$ contains the support of
$\rho$, let $\mathcal{E}$ be the event where it also contains the support of
$\rho_n$, and
\begin{align}
\sup_x | F_n(x) - F(x) |<n^{-0.5}.
\label{eq:conv-spd}
\end{align}
By Proposition \ref{prop:wignerproperties}, $\mathcal{E}$ holds with
probability at least $1-n^{-10}$.

The bound $L_{ij} \leq 1/\eta^2$ holds by the definition (\ref{eq:l-def}).
The bound $n^{-1}\|L\|_F \le C \eta^{-3/2}$ follows from summing
$n^{-1}\sum_j L_{ij}^2 \le C \eta^{-3}$ also over $i$ and taking a
square root. The bound $c \leq L_{ij}$ also holds on $\mathcal{E}$
by the definition of $L$. It remains to prove the last two bounds on the rows of $L$. 

For this, fix $a=1$ or $a=2$. 
For each $\lambda \in [-C_1, C_1]$, define a function 
$$
g_\lambda(r) \defn \left(\frac{ 1 }{ \ridge + (r - \lambda)^2 } \right)^a.
$$
Then for each $i$,
\begin{equation}\label{eq:Lija}
    \frac 1n \sum_{j=1}^n L_{ij}^a
=\frac 1n \sum_{j=1}^n g_{\lambda_i} ( \lambda_j )
=\int_{-C_1}^{C_1} g_{\lambda_i} ( r ) d \rho_n(r) .
\end{equation}
For some constants $C,C'>0$ and every $\lambda \in [-C_1,C_1]$,
replacing $\rho_n$ by the limiting density $\rho$, we have 
\begin{align}
\int_{-C_1}^{C_1} g_\lambda ( r ) d \rho(r) 
    &\le C \int_{-C_1}^{C_1} \left(\frac{ 1 }{ \ridge + (r - \lambda)^2
    }\right)^a d r \nonumber \\
    &\le C \left( \int_{|r - \lambda| \le \eta} \frac{1}{\eta^{2a}} d r +
    \int_{\eta \le |r - \lambda| \le 2C_1}  \frac{ 1  }{ (r - \lambda)^{2a} } d
    r \right) \le C' \eta^{1-2a} . \label{eq:limit-bd-1-row}
\end{align}

To bound the difference between $\rho_n$ and $\rho$,
note that $g_\lambda(r) \ge y$ for $y \ge 0$ if and only
if $|r - \lambda| \le b$ for some $b = b(y) \ge 0$. 
%Fixing $\rho_n$,
Consider random variables $R_n \sim \rho_n$ and $R \sim \rho$. Since
$g_\lambda \le \eta^{-2a}$, we have
\begin{align*}
    &\Big| \int_{- C_1}^{C_1} g_\lambda(r) d \rho_n(r) - \int_{- C_1 }^{C_1}
    g_\lambda(r) d \rho (r) \Big|\\
&= \Big| \int_0^{\eta^{-2a}} \left( \p \big\{ g_\lambda(R_n) \ge y \big\}  - \p
    \big\{ g_\lambda(R) \ge y \big\}  \right) dy \Big| \\
& \le  \int_0^{\eta^{-2a}} \Big| \p \big\{ |R_n-\lambda| \le b(y) \big\}  - \p \big\{
    |R-\lambda| \le b(y) \big\}  \Big| dy \\
& \le \int_0^{\eta^{-2a}} 2n^{-0.5} dy = 2\eta^{-2a} n^{-0.5},
\end{align*}
where the last inequality holds on the event $\mathcal{E}$ by \eqref{eq:conv-spd}.
Combining the last display with \eqref{eq:limit-bd-1-row}, we get that (\ref{eq:Lija}) is at most
$C\eta^{1-2a}$ for $\eta>n^{-0.1}$. This gives the desired bounds for $a=1$ and
$a=2$.
%\0nb{ZF: I don't think we need to take a union bound over $\lambda$ here
%and apply a covering argument, as the KS-bound already gives uniform control
%over $\lambda$; please check.}

\subsubsection{Proof of Lemma~\ref{lem:sig-conc}}
We now use Lemma \ref{lem:l-fro} to prove Lemma \ref{lem:sig-conc}.
Recall $L$ defined in (\ref{eq:l-def}), and $\tL$ which sets its diagonal to 0.
We need to bound the quantities 
\begin{align}
\text{(I)} : 
\sum_{i, j = 1}^n L_{ij} \phi_i \psi_j \xi_i \xi_j
\quad \text{ and } \quad
\text{(II)} : 
\sum_{i, j = 1}^n \tL_{ij} \phi_i \phi_j \xi_i \xi_j  . \label{eq:two-terms}
\end{align}
The proof for (II) is almost the same as that for (I), so we focus on (I) and briefly discuss the differences for (II). 
Let us define a matrix $K \in \R^{n \times n}$ by setting 
\begin{align}
K_{ij} = L_{ij} \phi_i \psi_j .
\label{eq:def-k}
\end{align}

\paragraph{Estimates for $K$.}
We translate the estimates for $L$ in Lemma \ref{lem:l-fro} to ones for $K$.
Note that since $\phi$ and $\psi$ are independent of $L$ and uniform over the sphere with entries on the order of $\frac{1}{\sqrt{n}}$, it is reasonable to expect that
$\|K\|_F \lesssim \frac{1}{n} \|L\|_F$ and $\|K\| \lesssim \frac{1}{n} \|L\|$ with high probability; 
however, neither statement is true in general, as shown by the counterexamples $L=\coord_1\coord_1^\top$ and $L=\identity$.
%for any $L$ independent of $(\phi,\psi)$, 
%As a counterexample, consider $L=\identity$, then $\|K\|=\|\phi\|_\infty\|\psi\|_\infty = \Theta(\frac{\log n}{n})$ with high probability), 
Fortunately, both statements hold for $L$ defined by \prettyref{eq:l-def} thanks to the structural properties established in \prettyref{lem:l-fro}.

\begin{lemma} \label{lem:k-fro}
In the setting of \prettyref{lem:sig-conc}, for the matrix $K \in \R^{n \times n}$ defined by~\eqref{eq:def-k}, we have $\| K \|_F  \lesssim \frac{1}{\eta^{3/2}}$ with probability at least $1 - 2 n^{-8}$.
\end{lemma}

\begin{proof}
It suffices to prove that  conditional on $L$, with probability at least $1 - n^{-8}$, we have
$$
\| K \|_F
\lesssim 
\frac{1}{n} \| L \|_F + \frac{\log n}{n^{1/4}} \|L\|_\infty .
$$
This together with Lemma~\ref{lem:l-fro} yields that
$$
\| K \|_F \lesssim \frac{1}{\eta^{3/2}} + \frac{ \log n}{ n^{1/4} \ridge } \lesssim \frac{1}{\eta^{3/2}} ,
$$
where the last inequality holds since we choose $\eta \gtrsim n^{-0.1}$. 

Note that we have
$$
\| K \|_F^2 = \sum_{i, j = 1}^n \phi_i^2 \psi_j^2 L_{ij}^2 \le \frac 12 \sum_{i, j = 1}^n \phi_i^4 L_{ij}^2 + \frac 12 \sum_{i, j = 1}^n \psi_j^4 L_{ij}^2 .
$$
It suffices to bound the first term, as the second term has the same distribution. Let $z$ be a standard Gaussian vector in $\R^n$. Then $z/\|z\|_2\eqdistr\phi$. By the concentration of $\|z\|_2$ around $\sqrt{n}$ (Lemma~\ref{lem:g-norm}), it remains to prove that with probability at least $1 - n^{-10}$,
$$
\sum_{i, j = 1}^n z_i^4 L_{ij}^2 = \sum_{i=1}^n z_i^4 \alpha_i
\lesssim 
%\|L\|_F^2 + \left( \sum_{i, j, k} L_{ij}^2 L_{i k}^2 \right)^{1/2} (\log n)^2 
%\le 
\|L\|_F^2 + \| L \|_\infty^2 n^{3/2} (\log n)^2 
$$
where $\alpha_i \triangleq \sum_{j = 1}^n L_{ij}^2$.
 %satisfying $\sum \alpha_i = \|L\|_F^2 \lesssim \frac{n^2}{\eta^{3}}$ and $\max \alpha_i \lesssim \frac{n}{\eta^3}$.

To this end, we compute
$$
\E\left[ \sum_{i = 1}^n z_i^4 \alpha_i \right]
= 3 \|L\|_F^2
$$
and moreover
\begin{align*}
\Var \left( \sum_{i= 1}^n z_i^4 \alpha_i \right) 
&=   \sum_{i= 1}^n \Var(z_i^4) \alpha_i^2 = 105 \sum_{i= 1}^n \alpha_i^2 \lesssim n^3 \|L\|_\infty^4.
\end{align*}
%\0nbr{JX. The above $9 \|L\|_F^2$ should be $9 \|L\|_F^4$.}
%\begin{align*}
%\Var \left( \sum_{i, j = 1}^n z_i^4 L_{ij}^2 \right) 
%&= \E\Big[ \sum_{k, \ell, i, j} z_k^4 z_i^4 L_{k \ell}^2 L_{ij}^2 \Big] - 
%\bigg(\E\Big[ \sum_{i, j = 1}^n z_i^4 L_{ij}^2 \Big]\bigg)^2 \\
%&= 9 \sum_{k \ne i} \sum_{\ell, j} L_{k \ell}^2 L_{ij}^2 + 105 \sum_{i, j, k} L_{ij}^2 L_{i k}^2 - 9 \|L\|_F^4 
%\le 96 \sum_{i, j, k} L_{ij}^2 L_{i k}^2 .
%\end{align*}
%%\0nbr{JX. The above $9 \|L\|_F^2$ should be $9 \|L\|_F^4$.}
Therefore, applying Theorem~\ref{thm:hyper} with $d=4$ we obtain
$$
\left| \sum_{i= 1}^n z_i^4 \alpha_i^2 - 3 \|L\|_F^2 \right| \lesssim 
%\left( \sum_{i, j, k} L_{ij}^2 L_{i k}^2 \right)^{1/2} (\log n)^2 \le 
\| L \|_\infty^2 n^{3/2} (\log n)^2 
$$
with probability at least $1 - n^{-10}$, which completes the proof.
\end{proof}

\begin{lemma} \label{lem:l-lem}
It holds with probability at least $1 - n^{-8}$ that for all $j, k \in [n]$,
$$
\sum_{i=1}^n \phi_i^2 L_{ij} L_{ik} \lesssim \frac 1n \sum_{i=1}^n L_{ij} L_{ik} 
\quad \text{ and } \quad 
\sum_{i=1}^n \psi_i^2 L_{ij} \lesssim \frac{1}{\eta} .
$$
\end{lemma}

\begin{proof}
Since $z/\|z\|_2$ has the same distribution as $\phi$ or $\psi$. By the concentration of $\|z\|_2$ around $\sqrt{n}$ (Lemma~\ref{lem:g-norm}) and a union bound, it remains to prove that with probability at least $1 - n^{-10}$, 
\begin{align}
\sum_{i=1}^n z_i^2 L_{ij} L_{ik} \lesssim\sum_{i=1}^n L_{ij} L_{ik} 
\quad \text{ and } \quad 
\sum_{i=1}^n z_i^2 L_{ij} \lesssim \frac{n}{\eta}.
%\sum_{i=1}^n L_{ij} .
\label{eq:two-ineq}
\end{align}
For the first inequality, Lemma~\ref{lem:g-norm} gives that with probability at least $1 - n^{-11}$, 
$$
\sum_{i=1}^n z_i^2 L_{ij} L_{ik} \lesssim \sum_{i=1}^n L_{ij} L_{ik} + \left( \sum_{i=1}^n L_{ij}^2 L_{ik}^2 \log n \right)^{1/2} + \left( \max_{i \in [n]} L_{ij} L_{ik} \right) \log n . 
$$
Note that $1 \lesssim L_{ij} \le 1/\ridge$ by Lemma~\ref{lem:l-fro}, so 
$$
\sum_{i=1}^n L_{ij} L_{ik} \gtrsim n 
\quad \text{ and } \quad
\left( \sum_{i=1}^n L_{ij}^2 L_{ik}^2 \log n \right)^{1/2} + \left( \max_{i \in [n]} L_{ij} L_{ik} \right) \log n \lesssim \frac{\sqrt{ n \log n } }{\eta^4} + \frac{\log n}{\eta^4} .
$$
Therefore, if $\eta \gtrsim n^{-0.1}$, then $\sum_{i=1}^n L_{ij} L_{ik}$ is the dominating term. Hence the first bound in~\eqref{eq:two-ineq} is established. 

The same argument also works to yield 
$$
\sum_{i=1}^n z_i^2 L_{ij} \lesssim  \sum_{i=1}^n L_{ij} .
$$
Combining this with Lemma~\ref{lem:l-fro}, we obtain the second bound in~\eqref{eq:two-ineq}.
%
%Moreover, Lemma~\ref{lem:g-norm} implies that with probability $1 - n^{-11}$, 
%$$
%\sum_{i=1}^n z_i^2 L_{ij} \lesssim \sum_{i=1}^n L_{ij} + \left( \sum_{i=1}^n L_{ij}^2 \log n \right)^{1/2} + \left( \max_{i \in [n]} L_{ij} \right) \log n . 
%$$
\end{proof}

\begin{lemma} \label{lem:k-op}
For the matrix $K \in \R^{n \times n}$ defined by~\eqref{eq:def-k}, we have
$
\| K \| \lesssim 1/ \eta 
$
with probability at least $1 - 2 n^{-8}$.
\end{lemma}

\begin{proof}
Consider the event where the estimates of Lemmas \ref{lem:l-fro}
and \ref{lem:l-lem} hold. Fix a unit vector $x \in \R^n$. We have 
\begin{align*}
\| K x \|_2^2 = \sum_{i=1}^n  \left( \sum_{j=1}^n \phi_i \psi_j L_{ij} x_j \right)^2 
= \sum_{j,k=1}^n \left( \sum_{i=1}^n \phi_i^2 L_{ij} L_{ik} \right) \psi_j \psi_k x_j x_k .
\end{align*}
The first bound in Lemma~\ref{lem:l-lem} then yields that 
\begin{align}
\| K x \|_2^2 
&\lesssim \frac 1n \sum_{j,k=1}^n \left( \sum_{i=1}^n L_{ij} L_{ik} \right) | \psi_j \psi_k x_j x_k |  \notag \\
&= \frac 1n \sum_{i=1}^n \left( \sum_{j=1}^n  | \psi_j |  L_{ij}  | x_j |  \right)^2 
= \frac 1n \big\| M  | x |  \big\|_2^2 \le \frac 1n \| M \|^2 , \label{eq:kbd}
\end{align}
where $|x|$ denotes the vector whose $i$-th entry is $|x_i|$, and the matrix $M$ is defined by 
$$
M_{ij} = |\psi_j| L_{ij} .
$$

Moreover, we have that
\begin{align*}
\| M^\top x \|_2^2 = \sum_{i=1}^n \psi_i^2 \left( \sum_{j=1}^n  L_{ij} x_j \right)^2 
\le \sum_{i=1}^n \psi_i^2 \left( \sum_{j=1}^n  L_{ij} \right) \left( \sum_{j=1}^n  L_{ij} x_j^2 \right)
\lesssim \frac{n}{\eta} \sum_{j=1}^n \left( \sum_{i=1}^n \psi_i^2 L_{ij} \right) x_j^2 , 
\end{align*}
where the first inequality follows from the Cauchy-Schwarz inequality, and the second inequality follows from the row sum bound in Lemma~\ref{lem:l-fro}. In addition, by the second inequality in Lemma~\ref{lem:l-lem}, 
$$
\| M^\top x \|_2^2 \lesssim \frac{n}{\ridge} \sum_{j=1}^n x_j^2 = \frac{n}{\ridge} .
$$
It follows that $\|M\|^2 = \|M^\top \|^2 \lesssim n / \ridge$ which, combined with~\eqref{eq:kbd}, yields 
$\| K x \|_2^2 \lesssim 1/\ridge$. 
Therefore, we conclude that
$
\| K \| \lesssim 1/ \eta .
$
\end{proof}

\paragraph{Bounding (I).}

We now bound $\sum_{i, j = 1}^n L_{ij} \phi_i \psi_j \xi_i \xi_j$.
Recall that the vectors $\phi, \psi$ and $\xi$ satisfy the relations \eqref{eq:def-3} and are otherwise uniform random. 
Let $z$ be a standard Gaussian vector in $\R^n$ independent of
$(\phi,\psi)$  and define
\begin{equation}
\tilde{\xi} \defn \sqrt{n-2} \, \frac{ z - (\phi^\top z) \phi - (\psi^\top z)
\psi }{ \big\| z - (\phi^\top z) \phi - (\psi^\top z) \psi \big\|_2}  + \phi +
\psi.
\label{eq:txi}
\end{equation}
Then the tuple $(\phi, \psi, \tilde{\xi})$ is equal to $(\phi, \psi, \xi)$ in
law.
Thus it suffices to study
\begin{align*}
\sum_{i, j = 1}^n L_{ij} \phi_i \psi_j \tilde \xi_i \tilde \xi_j .
\end{align*}

Note that we can write $\tilde{\xi} = \alpha z + \beta_1 \phi + \beta_2 \psi$ for random variables $\alpha, \beta_1, \beta_2 \in \R$, 
where 
$\beta_1 = 1-\alpha (\phi^\top z)$ and $\beta_2=1-\alpha (\psi^\top z)$. By concentration inequalities for $\|z\|_2$ and $\phi^\top z$  (Lemmas~\ref{lem:g-concentrate} and~\ref{lem:g-norm}), we have $0.9 \le \alpha \le 1.1$ and $|\beta_1| \lor |\beta_2| \le 5 \sqrt{\log n}$ with probability at least $1 - 4 n^{-8}$. Therefore, we obtain
\begin{align}
& \left| \sum_{i, j = 1}^n L_{ij} \phi_i \psi_j \tilde \xi_i \tilde \xi_j \right| 
\lesssim \left| \sum_{i, j = 1}^n L_{ij} \phi_i \psi_j z_i z_j  \right| 
+ \sqrt{ \log n } \left| \sum_{i, j = 1}^n L_{ij} \phi_i \phi_j \psi_j z_i \right| 
+  \sqrt{ \log n }  \left| \sum_{i, j = 1}^n L_{ij} \phi_i \psi_j^2 z_i \right| \notag \\
& \qquad  \qquad +  \sqrt{ \log n } \left| \sum_{i, j = 1}^n L_{ij} \phi_i^2 \psi_j z_j \right| 
+  \sqrt{ \log n } \left| \sum_{i, j = 1}^n L_{ij} \phi_i \psi_i \psi_j z_j \right|  
+ ( \log n ) \left| \sum_{i, j = 1}^n L_{ij} \phi_i^2 \phi_j \psi_j \right| \notag \\
& \qquad  \qquad +  ( \log n ) \left| \sum_{i, j = 1}^n L_{ij} \phi_i^2 \psi_j^2 \right| 
+  ( \log n ) \left| \sum_{i, j = 1}^n L_{ij} \phi_i \phi_j \psi_i \psi_j \right|  
+  ( \log n ) \left| \sum_{i, j = 1}^n L_{ij} \phi_i \psi_i \psi_j^2 \right| . \label{eq:nine-terms}
\end{align}
By the symmetry of $\phi$ and $\psi$, it suffices to study the following quantities
\begin{align}
& (\text{i}): \sum_{i, j = 1}^n L_{ij} \phi_i \psi_j z_i z_j   , \quad 
 (\text{ii}): \sum_{i, j = 1}^n L_{ij} \phi_i \phi_j \psi_j z_i  , \quad  
(\text{iii}): \sum_{i, j = 1}^n L_{ij} \phi_i \psi_j^2 z_i  ,  \notag \\
& (\text{iv}): \sum_{i, j = 1}^n L_{ij} \phi_i^2 \phi_j \psi_j  , \quad   
(\text{v}): \sum_{i, j = 1}^n L_{ij} \phi_i^2 \psi_j^2  , \qquad 
(\text{vi}): \sum_{i, j = 1}^n L_{ij} \phi_i \phi_j \psi_i \psi_j   .  \label{eq:six-terms}
\end{align}
We now bound each of these sums.

\emph{Bounding (i).} 
For the matrix $K$ defined by~\eqref{eq:def-k}, Lemma~\ref{lem:hw} yields that 
$$
\left| \sum_{i, j = 1}^n L_{ij} \phi_i \psi_j z_i z_j \right| = | z^\top K z | \lesssim |\Tr (K)| + \| K \|_F \sqrt{\log n} + \| K \| \log n ,
$$
with probability at least $1 - n^{-10}$. The trace vanishes because
$$
\Tr (K) = \sum_{i=1}^n L_{ii} \phi_i \psi_i = \frac{1}{\eta^2} \langle \phi, \psi \rangle =  \frac{1}{\eta^2} \langle U^\top \coord_1, U^\top \coord_2 \rangle 
= 0 .
$$
Moreover, Lemmas~\ref{lem:k-fro} and~\ref{lem:k-op} imply that with probability at least $1- 4 n^{-8}$,   we have 
$ \| K \|_F \lesssim 1/\eta^{3/2} $ and $\| K \| \lesssim 1/\eta$. Therefore, we conclude that 
$$
\left| \sum_{i, j = 1}^n L_{ij} \phi_i \psi_j z_i z_j \right| 
\lesssim \frac{\sqrt{\log n} }{\eta^{3/2}} + \frac{ \log n}{ \eta} .
$$

\emph{Bounding (ii) and (iii).} 
For (ii) in~\eqref{eq:six-terms}, Lemma~\ref{lem:g-concentrate} gives that with probability at least $1 - n^{-10}$, 
$$
\left| \sum_{i, j = 1}^n L_{ij} \phi_i \phi_j \psi_j z_i \right| \lesssim \left[ \sum_{i=1}^n \phi_i^2 \left( \sum_{j = 1}^n L_{ij} \phi_j \psi_j \right)^2 \right]^{1/2} \sqrt{  \log n } .
$$
Applying Lemmas~\ref{lem:l-lem} and~\ref{lem:l-fro}, we obtain that with probability at least $1 - 3 n^{-8}$, 
%\nb{YW: I got $1-2n^{-8}-n^{-10}$. Please check if this propogates elsewhere.}
$$
\left| \sum_{j = 1}^n L_{ij} \phi_j \psi_j \right| \le \frac 12 \sum_{j = 1}^n L_{ij} \phi_j^2 + \frac 12 \sum_{j = 1}^n L_{ij} \psi_j^2 
\lesssim \frac 1n \sum_{j = 1}^n L_{ij} \lesssim \frac{1}{\eta} .
$$
Combining the above two bounds yields 
$$
\left| \sum_{i, j = 1}^n L_{ij} \phi_i \phi_j \psi_j z_i \right| \lesssim \frac{1}{\eta} \left( \sum_{i=1}^n \phi_i^2 \right)^{1/2} \sqrt{  \log n } = \frac{ \sqrt{\log n} }{ \eta } .  
$$
The same argument also gives the same upper bound on (iii) in~\eqref{eq:six-terms}.

\emph{Bounding (iv), (v) and (vi).} 
The proofs for quantities (iv), (v) and (vi) in~\eqref{eq:six-terms} are similar, so we only present a bound on (vi). Since $\phi_i \phi_j \psi_i \psi_j \le \frac 12 ( \phi_i^2 + \psi_i^2 ) | \phi_j  \psi_j |$, it holds that
$$
\left| \sum_{i, j = 1}^n L_{ij} \phi_i \phi_j \psi_i \psi_j \right| 
\le \frac 12 \sum_{j = 1}^n \left( \sum_{i=1}^n L_{ij} \phi_i^2 \right) | \phi_j  \psi_j |
+ \frac 12 \sum_{j = 1}^n \left( \sum_{i=1}^n L_{ij} \psi_i^2 \right) | \phi_j  \psi_j | .
$$
By the second bound in Lemma~\ref{lem:l-lem} and the symmetry of $\phi$ and $\psi$, we then obtain that
with probability at least $1 - 2 n^{-8}$,
%\0nb{YW: Again, I got $1-2n^{-8}$ because of both $\phi$ and $\psi$. Please check if this propogates elsewhere.}
$$
\left| \sum_{i, j = 1}^n L_{ij} \phi_i \phi_j \psi_i \psi_j \right|  
\lesssim  \frac{1}{\eta} \sum_{j = 1}^n | \phi_j  \psi_j | \leq \frac{1}{\eta} 
%\le  \frac{1}{2 \eta} \sum_{j = 1}^n (\phi_j^2 +  \psi_j^2)= \frac{1}{\eta} 
$$
where the last step is by Cauchy-Schwarz.
Similar arguments yield the same bound on (iv) and (v) in~\eqref{eq:six-terms}.

\medskip

Substituting the bounds on (i)--(vi) into~\eqref{eq:nine-terms}, we obtain that with probability at least $1 - n^{-7}$, 
$$
\left| \sum_{i, j = 1}^n L_{ij} \phi_i \psi_j \tilde \xi_i \tilde \xi_j \right| 
\lesssim  \frac{\sqrt{\log n} }{\eta^{3/2}} + \frac{ \log n}{ \eta},
$$
which is the desired bound for quantity (I),

\paragraph{Bounding (II).}

The argument for establishing the same bound on $\sum_{i, j = 1}^n \tL_{ij} \phi_i \phi_j \xi_i \xi_j$ is similar, so we briefly sketch the proof. Analogous to~\eqref{eq:nine-terms}, we may use a (simpler) Gaussian approximation argument to obtain
\begin{align}
\left| \sum_{i, j = 1}^n \tL_{ij} \phi_i \phi_j \xi_i \xi_j \right| 
&\lesssim \left| \sum_{i, j = 1}^n \tL_{ij} \phi_i \phi_j z_i z_j  \right| 
+ \sqrt{ \log n } \left| \sum_{i, j = 1}^n \tL_{ij} \phi_i^2 \phi_j z_j \right| 
+  \sqrt{ \log n }  \left| \sum_{i, j = 1}^n \tL_{ij} \phi_i \phi_j^2 z_i \right|
    \nonumber \\
& \quad  \    +  ( \log n ) \left| \sum_{i, j = 1}^n \tL_{ij} \phi_i^2 \phi_j^2 \right|  , \label{eq:4terms}
\end{align}
where $z$ is a standard Gaussian vector independent of $\phi$ and $\tL$. Note that the matrix $\tL$ is the same as $L$ except that its diagonal entries are set to zero. Hence the last three terms on the right-hand side can be bounded in the same way as before. 

For the first term on the right-hand side of~\eqref{eq:4terms}, we again apply Lemma~\ref{lem:hw} to obtain 
$$
\left| \sum_{i, j = 1}^n \tL_{ij} \phi_i \phi_j z_i z_j  \right| \lesssim | z^\top \tK z | \lesssim |\Tr (\tK)| + \| \tK \|_F \sqrt{\log n} + \| \tK \| \log n ,
$$
where $\tK$ is defined by $\tK_{ij} = \tL_{ij} \phi_i \phi_j$. 
The trace term vanishes because the diagonal of $\tL$ is zero by definition.
The proofs of Lemmas~\ref{lem:k-fro},~\ref{lem:l-lem} and~\ref{lem:k-op} continue to hold with $\tK$ and $\tL$ in place of $K$ and $L$ respectively, and hence the norms $\| \tK \|_F$ and $\| \tK \|$ admit the same bounds as $\| K \|_F$ and $\| K \|$.

Combining the bounds on (I) and (II) completes the proof of Lemma
\ref{lem:sig-conc}, and hence also of Lemma \ref{lem:noiseless}.

\subsection{Bounding the effect of the noise}\label{sec:noise}

We now prove Lemma \ref{lem:noise}, bounding the difference $X-X_*$ between
the noisy and noiseless settings. Again, $\pi^*=\id$ is assumed without loss of generality.

\subsubsection{Vectorization and rotation by $U$}

Without loss of generality, we consider the entries
$X_{11}-(X_*)_{11}$ and $X_{21}-(X_*)_{21}$.
Writing the spectral decomposition $A=U\Lambda U^\top$,
we apply a vectorization followed by a rotation by $U$ to first
put these differences in a more convenient form.

Recall the notation
\[\phi=U^\top \coord_1, \quad \psi = U^\top \coord_2, \quad \text{and} \quad \xi
    = U^\top \bone,\]
where this triple $(\phi,\psi,\xi)$ satisfies (\ref{eq:def-3}). Set
\[\tZ=U^\top ZU, \qquad
\cL_*=\bI_n \otimes \Lambda-\Lambda \otimes \bI_n,
\qquad \cL=\bI_n \otimes \Lambda-(\Lambda+\sigma \tZ) \otimes \bI_n,\]
and introduce
\begin{equation}\label{eq:H}
    H=(\cL-\im\eta\,\bI_{n^2})^{-1}-(\cL_*-\im\eta\,\bI_{n^2})^{-1} \in \C^{n^2
\times n^2}
\end{equation}
where $\im=\sqrt{-1}$. We review relevant Kronecker product identities in
Appendix \ref{appendix:kronecker}.

We formalize the vectorization and rotation operations as the following lemma.
\begin{lemma}\label{lem:vectorization}
    When $\pi^*=\id$,
\begin{align}
X_{11}-(X_*)_{11} &= \frac{1}{\eta}\Im (\phi \otimes \phi)^\top H (\xi \otimes
    \xi),\label{eq:11diff}\\
X_{21}-(X_*)_{21} &= \frac{1}{\eta}\Im (\phi \otimes \psi)^\top H (\xi \otimes
    \xi),\label{eq:12diff}
\end{align}
where $\Im$ denotes the imaginary part.
Furthermore, the triple $(\phi,\psi,\xi)$ is independent of $H$.
\end{lemma}
\begin{proof}
From the definitions, $X_*$ and $X$ can be written in vectorized form as
\begin{align*}
\bx_* &\defn
\vectorize(X_*) = \sum_{ij} \frac{ 1 }{(\lambda_i -
\lambda_j)^2+\ridge } ( u_j \otimes u_i ) ( u_j \otimes u_i )^\top \bone_{n^2}
\in \R^{n^2}\\
\bx &\defn
\vectorize(X) = \sum_{ij} \frac{ 1 }{(\lambda_i -
\mu_j)^2+\ridge } ( v_j \otimes u_i ) ( v_j \otimes u_i )^\top \bone_{n^2}
\in \R^{n^2}.
\end{align*}
Since $u_i,v_j,\lambda_i,\mu_j,\eta$ are all real-valued, we may further write
\begin{align*}
\bx_* &=  \frac{1}{\eta}\Im \sum_{ij} \frac{ 1 }{\lambda_i-\lambda_j-\im \eta}
( u_j \otimes u_i ) ( u_j \otimes u_i )^\top \bone_{n^2}\\
\bx &=  \frac{1}{\eta}\Im \sum_{ij} \frac{ 1 }{\lambda_i-\mu_j-\im \eta}
( v_j \otimes u_i ) ( v_j \otimes u_i )^\top \bone_{n^2}.
\end{align*}
Recall the spectral decomposition~\eqref{eq:spectraldecomp}.
Note that 
$\bI_n \otimes A-A
\otimes \bI_n$ is a real symmetric matrix with 
orthonormal eigenvectors $\{ u_j \otimes u_i \}_{i,j \in [n]}$ and
corresponding eigenvalues $\lambda_i-\lambda_j$. Similarly,
$\bI_n \otimes A-B \otimes \bI_n$ is real symmetric with eigenvectors $\{ v_j \otimes u_i \}_{i,j \in [n]}$ and 
eigenvalues $\lambda_i-\mu_j$. 
Thus, using $U^\top AU=\Lambda$, we have
\begin{align*}
\bx_*&= \frac{1}{\eta} \Im(\bI_n \otimes A-A \otimes \bI_n
-\im \eta\,\bI_{n^2})^{-1}\bone_{n^2}\\
&=\frac{1}{\eta} \Im (U\otimes U) (\bI_n \otimes \Lambda -\Lambda  \otimes \bI_n -\iu \eta\,\bI_{n^2})^{-1} (U^\top \otimes U^\top) (\bone_{n} \otimes \bone_n)\\
&=\frac{1}{\eta} \Im (U\otimes U) (\cL_* -\iu \eta\,\bI_{n^2})^{-1} (\xi \otimes \xi).
\end{align*}
Similarly, using $U^\top BU=U^\top(A+\sigma Z)U=\Lambda+\sigma \tZ$, we have
\begin{align*}
\bx&=\frac{1}{\eta}\Im(\bI_n \otimes A-B \otimes \bI_n
-\im \eta\,\bI_{n^2})^{-1}\bone_{n^2} = \frac{1}{\eta} \Im (U\otimes U) (\cL -\iu \eta\,\bI_{n^2})^{-1} (\xi \otimes \xi).
\end{align*}
Therefore, for both $(k,\ell)=(1,1)$ and $(2,1)$,
\begin{align*}
    X_{k\ell}-(X_*)_{k\ell}&=\coord_k^\top (X-X_*) \coord_\ell =(\coord_\ell \otimes \coord_k)^\top (\bx-\bx_*) 
		= \frac{1}{\eta} \Im  ((U^\top \coord_\ell) \otimes (U^\top \coord_k) ) H (\xi \otimes \xi)
\end{align*}
which gives the desired (\ref{eq:11diff}) and (\ref{eq:12diff}).
%Then we have simply
%\begin{align*}
%\bx_*&= \frac{1}{\eta} \Im(\bI_n \otimes A-A \otimes \bI_n
%-\im \eta\,\bI_{n^2})^{-1}\bone_{n^2}\\
%\bx&=\frac{1}{\eta}\Im(\bI_n \otimes A-B \otimes \bI_n
%-\im \eta\,\bI_{n^2})^{-1}\bone_{n^2}.
%\end{align*}
%In this vectorized notation, for both $(k,\ell)=(1,1)$ and $(2,1)$,
%\begin{align*}
    %X_{k\ell}-(X_*)_{k\ell}&=\coord_k^\top (X-X_*) \coord_\ell\\
    %&=(\coord_\ell \otimes \coord_k)^\top (\bx-\bx_*)\\
    %&=\Im \frac{1}{\eta}(\coord_\ell \otimes \coord_k)^\top
%\Big[(\bI_n \otimes A-A \otimes \bI_n-\im \eta\,\bI_{n^2})^{-1}
    %-(\bI_n \otimes A-B \otimes \bI_n-\im
    %\eta\,\bI_{n^2})^{-1}\Big]\bone_{n^2}.
%\end{align*}
%Writing $(\coord_1,\coord_2,\bone)=(U\phi,U\psi,U\xi)$ and then
%applying $U^\top U=\bI_n$, $U^\top AU=\Lambda$, and
%$U^\top BU=U^\top(A+\sigma Z)U=\Lambda+\sigma \tZ$, we get
%(\ref{eq:11diff}) and (\ref{eq:12diff}).

For the independence claim, note that
$H$ is a function of $(\Lambda,\tZ)$, while $(\phi,\psi,\xi)$ is a function
of $U$. Crucially, since $Z \sim \GOE$,
Proposition \ref{prop:wignerproperties}(a) implies that $\tZ=U^\top ZU
\sim \GOE$ also for every fixed orthogonal matrix $U$. This distribution
does not depend on $U$, so $\tZ$ is independent
of $U$, and hence $(\phi,\psi,\xi)$ is independent of $H$.
\end{proof}

We divide the remainder of the proof into the following two results.

\begin{lemma}\label{lem:noise-conc}
For some constant $C>0$ and any deterministic matrix 
    $H \in \C^{n^2 \otimes n^2}$, with probability at least $1-n^{-7}$ over
    $(\phi,\psi,\xi)$,
    \[|(\varphi \otimes \varphi)^\top H(\xi \otimes \xi)|
\vee |(\varphi \otimes \psi)^\top H(\xi \otimes \xi)|
    \leq C \left(\|H\|+\|H\|_F\frac{\log n}{n}\right).\]
\end{lemma}

\begin{lemma}\label{lem:Hbounds}
    In the setting of Lemma \ref{lem:noise}, for some constant $C>0$ and for
    $H$ defined by (\ref{eq:H}), with probability at least $1-n^{-7}$,
    \[\|H\| \leq \frac{C\sigma}{\eta^2}
    \quad \text{ and } \quad
    \|H\|_F \leq \frac{C\sigma n}{\eta}\left(1+\frac{\sigma}{\eta}\right).\]
\end{lemma}
%\0nb{ZF: Lemma \ref{lem:Hbounds} above is the only step that does not
%immediately generalize to the bipartite setting.}

Lemma \ref{lem:noise} follows immediately from these results.
Indeed, applying Lemma \ref{lem:noise-conc} conditional on $H$, followed by
the estimates for $H$ in Lemma \ref{lem:Hbounds}, we get for both
$(k,\ell)=(1,1)$ and $(2,1)$ that
\[|X_{k\ell}-(X_*)_{k\ell}|
<\frac{C}{\eta}\left(\frac{\sigma}{\eta^2}
+\frac{\sigma \log n}{\eta}\left(1+\frac{\sigma}{\eta}\right)\right)\]
with probability at least $1-2n^{-7}$. By symmetry,
the same result also holds for all pairs $(k,\ell)$, and Lemma \ref{lem:noise}
follows from a union bound over $(k,\ell)$.

\subsubsection{Proof of Lemma~\ref{lem:noise-conc}}

%\0nb{YW: here $K,\tK$ are different than $K,\tK$ defined in Sec 3.1. Since they are central definition rather than local quantities, it might cause confusion. Consider change the notation here.}
Introduce $\matS,\mattS \in \C^{n \times n}$ such that
\[\vectorize(\matS)^\top=(\phi \otimes \phi)^\top H \quad \text{ and } \quad
\vectorize(\mattS)^\top=(\phi \otimes \psi)^\top H.\]
Then the quantities to be bounded are 
\begin{equation}\label{eq:quadform}
    (\phi \otimes \phi)^\top H (\xi \otimes \xi)=\xi^\top \matS \xi
\quad \text{ and } \quad (\phi \otimes \psi)^\top H(\xi \otimes \xi)
=\xi^\top \mattS \xi.
\end{equation}
We bound these in three steps: First, we bound
$\|\matS\|_F$ and $\|\mattS\|_F$. Second, we bound $|\Tr \matS|$ and $|\Tr \mattS|$.
Finally, we make a Gaussian approximation
for $\xi$ and apply the Hanson-Wright inequality to bound the quantities in \eqref{eq:quadform}.

\paragraph{Estimates for $\|\matS\|_F$ and $\|\mattS\|_F$.}

We show that with probability at least $1-5n^{-8}$,
\begin{equation}\label{eq:Kfro}
\|\matS\|_F \vee \|\mattS\|_F
\lesssim \frac{1}{n}\|H\|_F+\left(\frac{\log n}{n}\right)^{1/4}\|H\|.
\end{equation}
Note that $H^\top = H$, and 
%\0nb{YW: I might have been mistaken. But it seems $\cL_*$ is diagonal and $\cL$ is symmetric, so $H$ is symmetric (not Hermitian). Thus all $H^\top$ below should be replaced by $H$ for simplicity?}
\[\|\matS\|_F=\|\Htop (\phi \otimes \phi)\|_2 \quad
\text{ and } \quad \|\mattS\|_F=\|\Htop(\phi \otimes \psi)\|_2.\]
We give the argument for Gaussian vectors,
and then apply a Gaussian approximation for $(\phi,\psi)$.

\begin{lemma}\label{lemma:Hxx}
Let $x,y \in \R^n$ be independent with $\sN(0,1)$ entries. Then for a constant
$C>0$ and any deterministic $H \in \C^{n^2 \times n^2}$,
with probability at least $1-2n^{-10}$,
    \[\|H (x \otimes x)\|_2 \vee \|H(x \otimes y)\|_2
    \leq C \big[ \|H\|_F+(n^3\log n)^{1/4} \|H\| \big].\]
\end{lemma}
\begin{proof}
Let $M=H^*H$. Then 
$\|H (x \otimes x)\|^2=(x \otimes x)^\top M (x \otimes x)$.
We bound the mean and then apply Gaussian concentration of measure. Recall
the Wick formula  for Gaussian expectations: For any numbers $a_{ijk\ell} \in \C$,
\[\sum_{i,j,k,\ell} \E[x_ix_jx_kx_\ell]a_{ijk\ell}
=\sum_{i,j,k,\ell} \Big(\1\{i=j,k=l\}+\1\{i=k,j=l\}+\1\{i=l,j=k\}\Big)
a_{ijk\ell}.\]
Denoting the entry $M_{ij,k\ell}=(\coord_i \otimes \coord_j)^\top M(\coord_k
\otimes \coord_\ell)$ and applying this to $a_{ijk\ell}=M_{ij,k\ell}$, we get
\[\E[(x \otimes x)^\top M(x \otimes x)]
=\sum_{ij} (M_{ii,jj}+M_{ij,ij}+M_{ij,ji}).\]
Introduce the involution $Q \in \R^{n^2 \times n^2}$ defined by $Q(\coord_i 
\otimes \coord_j)=\coord_j \otimes \coord_i$, and note also that
$\sum_i (\coord_i \otimes \coord_i)=\vectorize(\bI_n)$. Then the above yields
\begin{align}
\Big|\E[(x \otimes x)^\top M(x \otimes x)]\Big|
&=\Big|\vectorize(\bI_n)^\top M \vectorize(\bI_n)+\Tr M+\Tr MQ
    \Big|\nonumber\\
&\leq \|M\|\cdot \|\vectorize(\bI_n)\|_2^2+\|H\|_F^2+\|H\|_F\|HQ\|_F
=n\|H\|^2+2\|H\|_F^2.\label{eq:meanxxxx}
\end{align}

To establish the concentration of $F(x) \triangleq (x \otimes x)^\top M(x \otimes x)$ around its
mean, we aim to apply \prettyref{lmm:lipext} by bounding the Lipschitz constant of $F$ on the ball
\[B=\{x \in \R^n:\|x\|^2 \leq 2n\}.\]
Note that for each $i\in[n]$,
\begin{align}
&\frac{\partial}{\partial x_i} [(x \otimes x)^\top M(x \otimes x)]\nonumber\\
&=(\coord_i \otimes x)^\top M(x \otimes x)
+(x \otimes \coord_i)^\top M(x \otimes x)
+(x \otimes x)^\top M(\coord_i \otimes x)
+(x \otimes x)^\top M(x \otimes \coord_i).\label{eq:partialxi}
\end{align}
For $x \in B$, using $\sum_i \coord_i\coord_i^\top=\bI_n$ and $M=M^*$, we have
\begin{align*}
    \sum_{i=1}^n |(\coord_i \otimes x)^\top M(x \otimes x)|^2
    &=\sum_{i=1}^n (x \otimes x)^\top M(\coord_i \otimes x)(\coord_i \otimes x)^\top M(x \otimes x)\\
&=(x \otimes x)^\top M(\bI_n \otimes xx^\top)M(x \otimes x)\\
&\leq \|x \otimes x\|_2^2 \cdot \|M\|^2 \cdot \|I \otimes xx^\top\| = \|x\|_2^6 \|M\|^2\\
&\leq (2n)^3\|M\|^2=8n^3\|H\|^4.
\end{align*}
The same bound holds for the other three terms in (\ref{eq:partialxi}). Thus
for all $x \in B$ and some constant $C_0>0$, $\|\nabla F(x) \leq C_0n^3\|H\|^4$.
Thus $x \mapsto (x \otimes x)^\top M(x \otimes
x)$ is $L_0$-Lipschitz on $B$, where $L_0 \triangleq \sqrt{C_0n^3}\|H\|^2$.
Finally, note that $F(0)=0$ and $\prob{x \notin B} \leq e^{-cn}$ by the $\chi^2$ tail bound of
Lemma \ref{lem:g-norm}. 
Applying \prettyref{lmm:lipext} with $t\asymp \sqrt{\log n}$,
 we conclude that 
\begin{align*}
|(x \otimes x)^\top M(x \otimes x)|&\leq 
|\E[(x \otimes x)^\top M (x \otimes x)]|+
Cn^2e^{-cn/2}\|H\|^2+C'\sqrt{n^3\log n}\|H\|^2 \\
&\overset{\prettyref{eq:meanxxxx}}{\leq} C''(\|H\|_F^2+\sqrt{n^3 \log n} \|H\|^2)
\end{align*}
holds with probability at least $1-n^{-10}$, where $C,C',C''$ are absolute constants. 
Taking the square root gives the desired result for $\|H(x \otimes x)\|_2$.

The bound for $H(x \otimes y)$ is similar: We have
\[\E[\|H(x \otimes y)\|_2^2]
=\E[(x \otimes y)^\top M(x \otimes y)]=\sum_{ij} M_{ij,ij}
=\Tr M=\|H\|_F^2.\]
On $B^2=\{(x,y) \in \R^{2n}:\|x\|_2^2 \leq 2n,\|y\|_2^2 \leq 2n\}$, we obtain
$\|\nabla_{x,y} [(x \otimes y)^\top M(x \otimes y)]\|^2
\leq L_0^2 \equiv Cn^3\|H\|^2$ as above.
Applying \prettyref{lmm:lipext} as above yields the desired bound on $\|H(x \otimes y)\|_2$.
%We may define an $L_0$-Lipschitz extension $F(x,y)$
%such that $F(x,y)=(x \otimes y)^\top M(x \otimes y)$ for $(x,y) \in
%B^2$, apply concentration of $F(x,y)$ around $\E[F(x,y)]$, and use the same
%argument to bound $\E[F(x,y)]-\E[(x \otimes y)^\top M (x \otimes y)]$. We omit
%the details for brevity.
\end{proof}

We now apply this and a Gaussian approximation to show (\ref{eq:Kfro}).
For $\|\matS\|_F$, let $x$ be a standard Gaussian vector in $\R^n$, so that
$x/\|x\|_2$ is equal to $\phi$ in law. Lemma~\ref{lem:g-norm} shows with
probability $1-n^{-10}$ that $\|x\|_2^2 \geq n/2$, so
\[\|\matS\|_F=\|\Htop(\phi \otimes \phi)\|_2 \leq \frac{2}{n}\|\Htop
(x \otimes x)\|_2,\]
and the result follows from Lemma~\ref{lemma:Hxx}.
For $\|\mattS\|_F$,
let $x,y$ be independent standard Gaussian vectors in $\R^n$. Since $\phi^\top
\psi = 0$ and $(\phi,\psi)$ is rotationally invariant in law, this pair is equal
in law to
$$
\Big( \frac{x}{\|x\|_2}, \frac{ y - (y^\top x / \|x\|_2^2) x }{ \|y - (y^\top x
/ \|x\|_2^2) x \|_2 } \Big).
$$
Standard concentration inequalities of Lemmas~\ref{lem:g-concentrate}
and~\ref{lem:g-norm} then yield 
\begin{equation}\label{eq:tKfro}
    \|\mattS\|_F=\|\Htop(\phi \otimes \psi)\|_2
\le \frac{2}{n}\|\Htop(x \otimes y)\|_2
+\frac{5 \sqrt{\log n} }{n^{3/2}}\|\Htop(x \otimes x)\|_2
\end{equation}
with probability at least $1 - 4 n^{-8}$, so the result also follows from Lemma~\ref{lemma:Hxx}.

\paragraph{Estimates for $\Tr \matS$ and $\Tr \mattS$.}

Next, we show that with probability at least $1-5n^{-8}$,
\begin{equation}\label{eq:trK}
    |\Tr \matS| \vee |\Tr \mattS| \le 2\|H\|.
\end{equation}
Note that
\[\Tr \matS = \Tr \matS\bI_n
= ( \phi  \otimes \phi )^\top H \vectorize(\bI_n),\]
and similarly
\[\Tr \mattS= ( \phi  \otimes \psi )^\top H \vectorize(\bI_n).\]

We apply a Gaussian approximation.
To bound $\Tr \matS$, let $x$ be a standard Gaussian vector, so $x/\|x\|_2$ is
equal to $\phi$ in law. Define $G \in \C^{n \times n}$ by
$\vectorize(G)=H\vectorize(\bI_n)$. Then
it follows from Lemmas~\ref{lem:g-norm} and~\ref{lem:hw} that
\begin{align}
|\Tr \matS|=|\phi^\top G \phi|=\frac { | x^\top G x | }{\|x\|_2^2}
\le \frac {1}{0.9n} \big( |\Tr G| + C\|G\|_F \log n \big)  \label{eq:x-x}
\end{align}
with probability at least $1 - n^{-10}$. We apply
$$
|\Tr G| = |\Tr G \bI_n|
= | \vectorize(\bI_n)^\top H \vectorize(\bI_n) | \le \|H\| \|\bI_n\|_F^2 = n \|H\| ,
$$
and
\begin{align*}
\|G\|_F = \|H \vectorize(\bI_n)\|_2 \le \|H\| \| \bI_n \|_F = \sqrt{n} \, \|H\|  .
\end{align*} 
Combining these yields $|\Tr \matS| \leq 2\|H\|$ for large $n$.
For $\Tr \mattS$, introducing independent standard Gaussian vectors $x,y$,
the same arguments as leading to (\ref{eq:tKfro}) give
\begin{align*}
|(\phi \otimes \psi)^\top H\vectorize(\bI_n)|
\leq & ~ \frac{2}{n} \big| (x \otimes y)^\top H \vectorize(\bI_n)  \big| + \frac{5
\sqrt{\log n} }{n^{3/2}} \big| (x \otimes x)^\top H \vectorize(\bI_n)  \big| \\
= & ~ \frac{2}{n} | x^\top G y|  + \frac{5
\sqrt{\log n} }{n^{3/2}} | x^\top G x| 
\end{align*}
with probability at least $1 - 4 n^{-8}$. Then $|\Tr \mattS| \leq 2\|H\|$ follows
from the same arguments as above by invoking \prettyref{lem:hw}.

\paragraph{Quadratic form bounds.}

We now use (\ref{eq:Kfro}) and (\ref{eq:trK}) to bound the quadratic forms
(\ref{eq:quadform}) in $\xi$.
Note that $\xi$ is dependent on $(\varphi,\psi)$ and hence on $\matS$ and $\mattS$, thus tools such as the Hanson-Wright inequality is not directly applicable;
nevertheless, thanks to the uniformity of $U$ on the orthogonal group, $(\xi,\varphi,\psi)=U^T (\ones,\coord_1,\coord_2)$ are only weakly dependent and well approximated by Gaussians. Below we make this intuition precise.

Consider $\xi^\top \mattS \xi$.
Let $z$ be a standard Gaussian vector in $\R^n$ independent of
$(\phi,\psi)$ (and hence of $\mattS$), and recall the Gaussian representation \prettyref{eq:txi} so that
%define
%$$
%\tilde{\xi} \defn \sqrt{n-2} \, \frac{ z - (\phi^\top z) \phi - (\psi^\top z)
%\psi }{ \big\| z - (\phi^\top z) \phi - (\psi^\top z) \psi \big\|_2}  + \phi +
%\psi.
%$$
%Then 
$(\phi, \psi, \tilde{\xi}) \eqdistr (\phi, \psi, \xi)$. 
Write \prettyref{eq:txi} as
 $\tilde{\xi} = \alpha z + \tilde{\phi}$,
where 
$\tilde{\phi} = (1-\alpha (\phi^\top z) )\phi + (1-\alpha (\psi^\top z) )\psi$  is a linear combination of $\phi$ and $\psi$.
By concentration inequalities for $\|z\|_2$ and $\phi^\top z$ in
Lemmas~\ref{lem:g-concentrate} and~\ref{lem:g-norm}, we have
$0.9 \le \alpha \le 1.1$
 and $\| \tilde{\phi} \|_2 \le 8 \sqrt{\log n}$ with
probability $1 - 4 n^{-8}$. On this event,
    \[|\xi^\top \mattS \xi|
    \le 1.21 |z^\top \mattS z|+|\tilde{\phi}^\top \mattS \tilde{\phi}|
    +1.1|z^\top \mattS \tilde{\phi}|
    +1.1|\tilde{\phi}^\top \mattS z|.\]
We bound these four terms separately conditional on $(\phi,\psi)$.

For the first term, applying
the Hanson-Wright inequality of Lemma~\ref{lem:hw}, we have
\[|z^\top \mattS z| \le | \Tr \mattS | + C\|\mattS\|_F \log n\]
with probability at least $1 - n^{-10}$.
For the second term, applying $\|\tilde{\phi}\|_2 \leq 8\sqrt{\log n}$,
\[|\tilde{\phi}^\top \mattS\tilde{\phi}|
    \leq \|\mattS\| \| \tilde{\phi}\|_2^2 \le 64 \|\mattS\|_F \log n.\]
For the third term, 
%\0nb{ZF: I modified this slightly as $z$ is not independent
%of $\tilde{\phi}$, please check.}
\[|\tilde{\phi}^\top \mattS z| \leq \|\tilde{\phi}\|_2 \|\mattS z\|_2.\]
Applying again Lemma \ref{lem:hw}, with probability at least $1-n^{-10}$,
\[\|\mattS z\|_2^2 \leq \Tr \mattS^*\mattS+C\|\mattS^*\mattS\|_F \log n
\leq (C+1)\|\mattS\|_F^2 \log n,\]
so
\[|\tilde{\phi}^\top \mattS z| \lesssim \|\mattS\|_F \log n.\]
The fourth term is bounded similarly to the third, and combining these gives
$|\xi^\top \mattS \xi| \lesssim |\Tr \mattS|+\|\mattS\|_F\log n$
with probability at least  $1-6n^{-8}$.
Applying (\ref{eq:Kfro}) and (\ref{eq:trK}), we get
\[|(\phi \otimes \psi)^\top H(\xi \otimes \xi)|
=|\xi^\top \mattS \xi| \lesssim \|H\|+\frac{\log n}{n}\|H\|_F\]
as desired. 

The Gaussian approximation argument for $(\phi \otimes \phi)^\top H
(\xi \otimes \xi)=\xi^\top \matS \xi$ is simpler  and omitted for brevity. This concludes the proof of
Lemma \ref{lem:noise-conc}.

\subsubsection{Proof of Lemma \ref{lem:Hbounds}}

Recalling the definition of $H$ in (\ref{eq:H}) and
applying 
\begin{equation}
A^{-1}-B^{-1}=A^{-1}(B-A)B^{-1},
\label{eq:ABinv}
\end{equation}
 we get
\begin{equation}\label{eq:minusH}
    -H=(\cL_*-\im\eta\,\bI_{n^2})^{-1}(\sigma \tZ \otimes \bI_n)(\cL-\im\eta\,\bI_{n^2})^{-1}.
\end{equation}
As $\cL_*-\im\eta\,\bI_{n^2}$ is diagonal with each entry at least $\eta$ in magnitude, we have the deterministic bound
$\|(\cL_*-\im\eta\,\bI_{n^2})^{-1}\| \leq 1/\eta$, and similarly
$\|(\cL-\im\eta\,\bI_{n^2})^{-1}\| \leq 1/\eta$. Applying Proposition
\ref{prop:wignerproperties}(c), $\|\tZ\| \leq C$ with probability at least $1-n^{-10}$
for a constant $C>0$. On this event, $\|H\| \leq C\sigma/\eta^2$.

To bound $\|H\|_F$, let us apply \prettyref{eq:ABinv} again to $(\cL-\im\eta\,\bI_{n^2})^{-1}$ in (\ref{eq:minusH}),
to write
    \[-H=(\cL_*-\im\eta \bI)^{-1}(\sigma \tZ \otimes \bI)(\cL_*-\im\eta
    \bI)^{-1}\Big(\bI-(\sigma \tZ \otimes \bI)(\cL-\im\eta \bI)^{-1}\Big).\]
    Applying 
		\begin{equation}
		\|AB\|_F \leq \|A\|_F\|B\|,
		\label{eq:ABF}
		\end{equation}
		 we get with probability at least $1-n^{-10}$
    that
\begin{equation}\label{eq:Hfro}
    \|H\|_F \leq \|(\cL_*-\im\eta \bI)^{-1}(\sigma \tZ \otimes \bI)
(\cL_*-\im\eta \bI)^{-1}\|_F (1+C\sigma/\eta).
\end{equation}

Note that here, $\cL_*$ is diagonal, and $\tZ$ is independent of $\cL_*$. 
%\0nb{
%Let us
%introduce 
%\[ w=\mathsf{diag} \big( (\cL_* - \im \eta \bI)^{-1} \big) \in \C^{n^2} \] 
%%\0nbr{JX. Should it be $\textbf{diag}$ instead of $\vectorize$?}
%and
%index $w$ by the pair $(i,k) \in [n]^2$. 
%YW: I don't think here $\sf diag$ matches the meaning in \prettyref{eq:a-dec}. 
%So I rewrote the following.
%}
Let $w \in \C^{n^2}$ consist of the diagonal entries of $(\cL_* - \im \eta \bI)^{-1}$, 
indexed by the pair $(i,k) \in [n]^2$, i.e., $w_{ik} = \frac{1}{\lambda_i-\lambda_k-\iu \eta}$.
Let us also
desymmetrize $\tZ$ and write
\[\tZ=\frac{1}{\sqrt{2n}}(W+W^\top),\]
where $W \in \R^{n \times n}$ has $n^2$ independent $\sN(0,1)$ entries. Then
\begin{align*}
\|(\cL_*-\im\eta \bI)^{-1}(\sigma \tZ \otimes \bI)(\cL_*-\im\eta \bI)^{-1}\|_F^2
&=\frac{\sigma^2}{2n}\left\|(\cL_*-\im\eta \bI)^{-1}
    \big((W+W^\top) \otimes \bI\big)(\cL_*-\im\eta \bI)^{-1}\right\|_F^2\\
&=\frac{\sigma^2}{2n}
\sum_{i,j,k=1}^n (W_{ij}+W_{ji})^2|w_{ik}|^2|w_{jk}|^2.
\end{align*}
Recall the symmetric matrix $L \in \R^{n \times n}$ defined by (\ref{eq:l-def}).
We have
\[\sum_{k=1}^n |w_{ik}|^2|w_{jk}|^2
=\sum_{k=1}^n \frac{1}{(\lambda_i-\lambda_k)^2+\eta^2}
\cdot \frac{1}{(\lambda_j-\lambda_k)^2+\eta^2}
=\sum_{k=1}^n L_{ik}L_{jk}=(L^2)_{ij}.\]
Introducing $v=\vectorize(L^2)\in\reals_+^{n^2}$ indexed by $(i,j)$, and applying Lemma \ref{lem:g-norm} conditional on $v$, 
we get
\begin{align}
    \|(\cL_*-\im\eta \bI)^{-1}(\sigma \tZ \otimes \bI)(\cL_*-\im\eta
    \bI)^{-1}\|_F^2 \leq 
\frac{2\sigma^2}{n} \sum_{i,j=1}^n W_{ij}^2v_{ij}
    \leq \frac{2\sigma^2}{n}\left(\|v\|_1+C\|v\|_2 \log
    n\right)\label{eq:LZLfro}
\end{align}
with probability at least $1-n^{-10}$.

Finally, we apply Lemma \ref{lem:l-fro} to bound $\|v\|_1$ and $\|v\|_2$.
Note that $\|v\|_1=\bone L^2 \bone=\|L\bone\|_2^2$.
Applying $\max_i (L\bone)_i \leq Cn/\eta$ from Lemma \ref{lem:l-fro}, we get
with probability at least $1-n^{-8}$ that
\[\|v\|_1 \lesssim n^3/\eta^2.\]
By \prettyref{eq:ABF}, we also have $\|v\|_2^2=\|L^2\|_F^2\leq \|L\|^2 \cdot \|L\|_F^2$.
%We also have 
%\[\|v\|_2^2=\|L^2\|_F^2=\Tr L^4
%\leq \|L\|^2 \cdot \Tr L^2=\|L\|^2 \cdot \|L\|_F^2.\]
Applying $\|L\| \leq \max_i (L\bone)_i \leq Cn/\eta$ and
$\|L\|_F^2 \leq Cn^2/\eta^3$ from Lemma \ref{lem:l-fro},
we get
\[\|v\|_2^2 \lesssim n^4/\eta^5.\]
Applying this to (\ref{eq:LZLfro}) and then back to (\ref{eq:Hfro}) yields 
\[\|H\|_F^2 \lesssim \frac{\sigma^2}{n}\left(\frac{n^3}{\eta^2}+
\frac{n^2\log n}{\eta^{5/2}}\right)\left(1+\frac{\sigma}{\eta}\right)^2
\lesssim \frac{\sigma^2n^2}{\eta^2}\left(1+\frac{\sigma}{\eta}\right)^2,\]
where the second inequality holds for $\eta>n^{-0.1}$. This is the
desired bound on $\|H\|_F$.

This concludes the proof of Lemma \ref{lem:Hbounds}, and hence of
Lemma \ref{lem:noise}.

\subsection{Proof for the bipartite model}\label{sec:bipartiteproof}

We now prove Theorem \ref{thm:bipartite} for exact recovery in the bipartite 
model. We first show that \prettyref{alg:bigrampa} successfully recovers $\pi_1^*$.
This extends the preceding argument in the symmetric case. We then
show that the linear assignment subroutine recovers $\pi_2^*$ if $\hat{\pi}_1=\pi_1^*$.

\subsubsection{Recovery of $\pi_1^*$ by spectral method}

The argument is a minor extension of that in the Gaussian Wigner model. Let us
write 
$A=\sqrt{FF^\top}$, 
%$A=|F^\top|=\sqrt{FF^\top}$, 
and introduce its spectral decomposition
$A=U\Lambda U^\top$. Note that $\Lambda$ and $U$ then consist of the singular values and
left singular vectors of $F$.

To analyze the noiseless solution $X_*=\hat X(A,A)$, note that all three claims of Proposition
\ref{prop:wignerproperties} hold for $A$, where the constants $C,C',c$ may
depend on $\kappa=\lim n/m$. Indeed, here, $\rho$ is the law of
$\sqrt{\lambda}$ when $\lambda$ is distributed according to the Marcenko-Pastur
distribution with density $g(x) = \frac{\sqrt{(\lambda_+-x)(x-\lambda_-)}}{2\pi \kappa x} \indc{\lambda_- \leq x \leq \lambda_+}$ and $\lambda_\pm \triangleq (1\pm \sqrt{\kappa})^2$.
Then the density of $\rho$ is
$2 x g(x^2)$ (In the case of $\kappa=1$, $\rho$ is
the quarter-circle law.)
Therefore, for any $\kappa \in (0,1]$, the density of $\rho$ is supported on 
on $[1-\sqrt{\kappa},1+\sqrt{\kappa}]$ and bounded by some $\kappa$-dependent constant $C$.  The rate of
convergence in (b) follows from \cite[Theorem 1.1]{GotTik11}, and the claims in
(a) and (c) are well-known.
Thus the proof of Lemma \ref{lem:noiseless} applies, and we obtain with
probability $1-5n^{-5}$ that
\begin{equation}\label{eq:bipartitenoiseless}
    \min_{i \in [n]} (X_*)_{ii}>\eta^2/2, \qquad
\max_{i,j \in [n]:i \neq j} (X_*)_{ij}<C\left(\frac{\sqrt{\log
n}}{\eta^{3/2}}+\frac{\log n}{\eta}\right).
\end{equation}

Next, we analyze the noisy solution $X \triangleq \hat X(A,B)$. Set
$B=\sqrt{GG^\top}$, 
%$B=|G^\top|=\sqrt{GG^\top}$, 
and
define $H$ by (\ref{eq:H}) but replacing $\Lambda+\sigma \tZ$ in $\cL$
with the general definition
\[\cL=\bI_n \otimes \Lambda-U^\top BU \otimes \bI_n.\]
Then the representations (\ref{eq:11diff}) and (\ref{eq:12diff}) of
Lemma \ref{lem:vectorization} continue to hold. Furthermore, write the
singular value decomposition $F=U\Lambda V^\top$,
set $\tilde{W}=U^\top W$, and note that $U$ is uniformly random and
independent of $(\Lambda,V,\tilde{W})$. Then
\[U^\top BU=U^\top\sqrt{(F+\sigma W)(F+\sigma W)^\top}U
=\sqrt{(\Lambda V^\top+\sigma \tilde{W})(\Lambda V^\top+\sigma
\tilde{W})^\top}\]
which
is independent of $U$, so that $(\phi,\psi,\xi)$ is still independent of $H$.
Then applying Lemma \ref{lem:noise-conc} conditional on $H$, we get with
probability at least $1-n^{-7}$ that
\[|X_{k\ell}-(X_*)_{k\ell}| \lesssim
\frac{1}{\eta}\left(\|H\|+\|H\|_F\frac{\log n}{n}\right)\]
for both $(k,\ell)=(1,1)$ and $(2,1)$.

To conclude the proof, we need a counterpart of \prettyref{lem:Hbounds} bounding the norms of $H$.
Let us simply use the fact that $H$ has dimension $n^2 \times n^2$ to
bound $\|H\|_F \leq n\|H\|$, and apply
\[\|H\| \leq \|(\cL-\im \eta \bI)^{-1}\| \cdot \|\cL-\cL_*\| \cdot
\|(\cL_*-\im \eta \bI)^{-1}\| \leq \|\cL-\cL_*\|/\eta^2.\]
Then
\[\|\cL-\cL_*\|=
\|U^\top (A-B)U \otimes \bI_n\|
=\|A-B\| \leq
\frac{2}{\pi}\|F-G\|\left(2+\log \frac{\|F\|+\|G\|}{\|F-G\|}\right)\]
where 
the last inequality follows from \cite[Proposition 1]{kato1973continuity}.
For a constant $C>0$, this is at most $C\sigma\log(1/\sigma)$
with probability at least $1-n^{-10}$ by the analogue of
Proposition \ref{prop:wignerproperties}(c) applied to the noise
$W=(G-F)/\sigma$ in this model.
Combining these bounds yields for $(k,\ell)=(1,1)$ and $(2,1)$ that with
probability $1-2n^{-7}$,
\[|X_{k\ell}-(X_*)_{k\ell}| \leq \frac{C\sigma \log(1/\sigma)\log n}{\eta^3}.\]
%\0nb{ZF: I previously wrote out the argument to remove the $\log(1/\sigma)$
%factor here, by improving upon the direct application of Kato's result. But it
%seems unnecessary to make the reader read this complicated argument for this
%small improvement, unless we can also remove the $\log n$ factor by
%bounding the Frobenius norm of $\|A-B\|$.}
This holds for all pairs $(k,\ell)$ with probability at least $1-2n^{-5}$
by a union bound.
Thus for $\eta<c/\log n$, $\sigma \log (1/\sigma)<c'\eta/\log n$, and
sufficiently small constants $c,c'>0$, we get
from (\ref{eq:bipartitenoiseless}) that
\[\min_{i \in [n]} X_{ii}>\max_{i,j \in [n]: i \neq j} X_{ij}.\]
So \prettyref{alg:bigrampa} recovers $\hat{\pi}_1=\pi_1^*$ with probability at least $1-7n^{-5}$
when $\pi_1^*=\id$.
By equivariance of the algorithm, this also holds for any $\pi_1^*$.

\subsubsection{Recovery of $\pi_2^*$ by linear assignment}
We now show that on the event where $\hat{\pi}_1=\pi_1^*$, as long as $n \gtrsim
\frac{\log m}{\log (1 + \frac{1}{4\sigma^2})}$, the linear assignment step of \prettyref{alg:bigrampa} recovers $\hat{\pi}_2=\pi_2^*$ with high
probability. Without loss of generality, let us take $\pi_1^*=\id$, and denote
more simply $\pi^*=\pi_2^*$. We then formalize this claim as follows.

\begin{theorem} \label{thm:linear}
    Consider the single permutation model $G^{\id, \pi^*}=F+\sigma W$
    where $G^{\id, \pi^*}_{ij}=G_{i,\pi^*(j)}$, and $F$ and $W$ are as in Theorem
    \ref{thm:bipartite}. Let
    \[\hat{\pi}=\argmax_{\pi \in \cS_m} \sum_{j=1}^m (F^\top G)_{j,\pi(j)}.\]
		If $n \ge \frac{24\log m}{\log (1 + \frac{1}{4\sigma^2})} $, then $\hat{\pi}=\pi^*$ with probability at least
    $1 - 2 m^{-4}$.
\end{theorem}
\begin{proof}
Without loss of generality, assume that $\pi^*=\id$.
Let us also rescale and consider $G = F + \sigma W$, 
where $F$ and $W$ are $\dimone \times \dimtwo$ random matrices with \iid~$\sN(0,
1)$ entries. Our goal is to show that 
\begin{align*}
\hat \Pi = \argmax_{\Pi \in \fS_m} \, \langle F \Pi, G \rangle 
\end{align*}
coincides with the identity with probability at least $1-m^{-4}$.

For any $\Pi \neq \identity$, we have
$\Iprod{F\Pi}{G} - \Iprod{F}{G} = \sigma \Iprod{F(\Pi-\identity)}{W} - \Iprod{F(\identity-\Pi)}{F}$, where 
$\Iprod{F(\identity-\Pi)}{F} = \frac{1}{2} \|F(\identity-\Pi)\|_F^2$. Then
\begin{align*}
\prob{\Iprod{F\Pi}{G} > \Iprod{F}{G}}
%\prob{\langle F - F \Pi , F \rangle \le \sigma \langle F \Pi - F, W \rangle} 
& = \prob{  \left \langle W, \frac{F (\Pi-\identity)}{ \|F(\identity-\Pi)\|_F } \right\rangle 
\ge \frac{\|F(\identity-\Pi)\|_F }{2\sigma} }  \\
&= \expect{Q\left( \frac{\|F(\identity-\Pi)\|_F }{2\sigma} \right)}\\
& \stepa{\le} \expect{ \exp \left( - \frac{\|F(\identity-\Pi)\|_F^2 }{8\sigma^2}\right)} \\
& \stepb{=} \sth{\expect{ \exp \left( - \frac{\|z^\top (\identity-\Pi)\|_F^2 }{8\sigma^2}\right)}}^n,
%\\
%&= \left( 1+ \frac{1}{4\sigma^2} \right)^{-nk}.
\end{align*}
where (a) follows from the Gaussian tail bound $Q(x) \triangleq \int_x^\infty \frac{1}{\sqrt{2\pi}} e^{-t^2/2} dt \le e^{-x^2/2}$ for $x>0$; 
(b) is because the $n$ rows of $F$ are \iid~copies of $z\sim N(0,\identity_m)$.

Denote the number of non-fixed points of $\Pi$ by $k \geq 2$, which is also the rank of $\identity-\Pi$. Denote its singular values by $\sigma_1,\ldots,\sigma_k$. Then we have $\sum_{i=1}^k \sigma_i^2 = \|\identity-\Pi\|_F^2 = 2 k$ and $\max_{i\in[k]} \sigma_i \leq \|\identity-\Pi\| \leq 2$.
%In particular, we have $\sum_{i=1}^k \indc{\sigma_i^2\geq 1} \geq k/4$. 
By rotational invariance, we have $\|z^\top (\identity-\Pi)\|_F^2 \eqdistr \sum_{i=1}^k \sigma_i^2 w_i^2$, where $w_1,\ldots,w_k\iiddistr N(0,1)$. Then
\begin{align}
\expect{ \exp \left( - \frac{\|z^\top (\identity-\Pi)\|_F^2 }{8\sigma^2}\right)} 
= & ~ \prod_{i=1}^k \expect{ \exp \left( - \frac{\sigma_i^2}{8\sigma^2}  w_i^2  \right)} \label{eq:MGF} \\
%= & ~ \prod_{i=1}^k \frac{1}{\sqrt{1+\frac{\sigma_i^2}{4\sigma^2}}}  \\
= & ~ \exp\sth{ - \frac{1}{2} \sum_{i=1}^k \log\pth{1+\frac{\sigma_i^2}{4\sigma^2}}}  \leq \exp\sth{ - \frac{k}{8} \log\pth{1+\frac{1}{4\sigma^2}}},\nonumber 
\end{align}
where the last step is due to $\sum_{i=1}^k \indc{\sigma_i^2\geq 1} \geq k/4$.\footnote{
The sharp condition $n \log \left( 1 + \frac{1}{\sigma^2} \right)- 4 \log m \to + \infty$ can be obtained by computing the singular values in \prettyref{eq:MGF} exactly; cf.~\cite{dai2019database}.}
Combining the last two displays and applying the union bound over $\Pi \neq \identity$, we have
\begin{align*}
\prob{\hat \Pi \neq \identity}
\leq & ~ \sum_{\Pi \neq \identity} \prob{\Iprod{F\Pi}{G} > \Iprod{F}{G}} \\
\leq & ~ \sum_{k=2}^{m} \binom{m}{k} k!  \left( 1+ \frac{1}{4\sigma^2} \right)^{-nk/8} \le \sum_{k=2}^{m} m^k \left( 1+ \frac{1}{4\sigma^2} \right)^{-nk/8} \leq 2m^{-4},
\end{align*}
provided that $m \geq 2$ and $m \left( 1+ \frac{1}{4\sigma^2} \right)^{-n/8} \leq m^{-2}$.
\end{proof}

\subsection{Gradient descent dynamics}\label{sec:gd}

Finally, we prove Corollary \ref{cor:sol}, which connects $\widehat{X}$ in
(\ref{eq:spectralnew}) to the gradient descent dynamics (\ref{eq:gd}) and the
optimization problems (\ref{eq:x-est}) and (\ref{eq:constrained}).

To show that $\widehat{X}$ solves \eqref{eq:x-est}, note that
the objective function in \eqref{eq:x-est} is quadratic, with
first order optimality condition
$$
A^2X+XB^2 - 2AXB +  \ridge X=\bJ.
$$
Setting $\bx=\vectorize(X)$ and writing this in vectorized form 
$$
\big[ (\bI_n \otimes A - B \otimes \bI_n)^2 + \ridge \bI_{n^2}
\big]\bx=\bone_{n^2},
$$
we see that the vectorized solution to \eqref{eq:x-est} is
\[\hat{\bx}=\big[ (\bI_n \otimes A - B \otimes \bI_n)^2 + \ridge \bI_{n^2}
\big]^{-1}\bone_{n^2} \in \R^{n^2}.\]
Applying the spectral decomposition (\ref{eq:spectraldecomp}),
we get
\begin{equation}
\hat{\bx}=\sum_{ij} \frac{ 1 }{(\lambda_i -
\mu_j)^2+\ridge } ( v_j \otimes u_i ) ( v_j \otimes u_i )^\top \bone_{n^2} = 
\sum_{ij} \frac{ u_i ^\top \allones_n v_j }{(\lambda_i -
\mu_j)^2+\ridge }  \vecc(u_i v_j^\top),
\label{eq:spectralnew-vec}
\end{equation}
which is exactly the vectorization of $\widehat{X}$ in (\ref{eq:X}).

Recall that $\widetilde{X}$ denotes the minimizer of (\ref{eq:constrained}).
Introducing a Lagrange multiplier $2\alpha \in \R$ for the constraint,
 %shows that
%$\widetilde{X}$ minimizes (\ref{eq:x-est}) with
%$\alpha \bone^\top X \bone$ in place of $\bone^\top X \bone$.
the first-order stationarity condition is 
$A^2X+XB^2 - 2AXB +  \ridge X=\alpha \bJ$, and hence
$\widetilde{X}=\alpha \widehat{X}$.
To find $\alpha$, note that $\bone^\top \widetilde{X} \bone
=\alpha \bone^\top \widehat{X} \bone=n$.
%Note that $\bone^\top \widehat{X}\bone \geq 0$, because otherwise the minimum of
%\eqref{eq:x-est} would be positive, contradicting that its objective value
%equals zero at $X=0$. 
Furthermore, 
from \prettyref{eq:spectralnew} we have
\[
\bone^\top \widehat{X} \bone = \sum_{ij} \frac{ \Iprod{u_i}{\ones}^2 \Iprod{v_j}{\ones}^2  }{(\lambda_i -
\mu_j)^2+\ridge } > 0.
\]
%\bone^\top \widehat{X} = \Iprod{\bone_{n^2}}{\hat{\bx}} > 0 which follows directly from \prettyref{eq:spectralnew-vec}.
Hence $\alpha>0$. These claims together establish part (a).

For (b), let us consider the gradient descent dynamics also in its 
vectorized form. Namely, define $\bx^{(t)} \defn \vectorize( X^{(t)} )$.
Then~\eqref{eq:gd} can be written as
$$
\bx^{(t+1)} = \big[ (1 - \stepsize  \ridge) \bI_{n^2} - \stepsize (\bI_n \otimes
A-B \otimes \bI_n)^2\big]
\bx^{(t)} + \stepsize \bone_{n^2}.
$$
For the initialization $\bx^{(0)}=0$, this gives
\begin{align} 
    \bx^{(t)}&=\stepsize \sum_{s = 0}^{t-1} \big[ (1 - \stepsize  \ridge) \bI_{n^2}
    - \stepsize (\bI_n \otimes A-B \otimes \bI_n)^2 \big]^s \bone_{n^2} \notag \\
&= \stepsize \sum_{i, j = 1}^n \sum_{s=0}^{t-1} \big[ 1 - \stepsize \ridge -
    \stepsize ( \lambda_i - \mu_j )^2 \big]^s (v_j \otimes u_i) ( v_j
    \otimes u_i)^\top \bone_{n^2} \notag \\
&= \sum_{i, j = 1}^n \frac{ 1 - [ 1 - \stepsize \ridge - \stepsize ( \lambda_i -
    \mu_j )^2 ]^t }{ \ridge + ( \lambda_i - \mu_j )^2 } (v_j \otimes u_i) (
    v_j \otimes u_i )^\top \bone_{n^2} . \label{eq:vec-sol}
\end{align}
Undoing the vectorization yields part (b).

For (c), note that $\eta^2+(\lambda_i-\mu_j)^2<C$ with probability at least $1-n^{-10}$
by Proposition \ref{prop:wignerproperties}(c), so that the convergence in
part (b) holds provided that the step size $\gamma \leq c$ for some sufficiently small constant $c$. On this event,
for all pairs $(k,\ell)$ we may apply the simple bound
\begin{align*}
|X^{(t)}_{k\ell}-\widehat{X}_{k\ell}|
    &\leq \sum_{ij} \frac{(1-\stepsize (\lambda_i-\mu_j)^2-\stepsize
    \ridge)^t}{(\lambda_i-\mu_j)^2+\ridge}|(u_i^\top \coord_k)(v_j^\top
    \coord_\ell)(u_i^\top \bJ v_j)|\\
    &\leq \frac{(1-\stepsize \ridge)^t}{\ridge} \cdot n^2\max_{ij}
    |(u_i^\top \coord_k)(v_j^\top
    \coord_\ell)(u_i^\top \bJ v_j)|\\
    &\leq \frac{(1-\stepsize \ridge)^t}{\ridge} \cdot n^3.
\end{align*}
In particular, for $t \geq (C\log n)/(\stepsize \ridge)$ and a sufficiently
large constant $C>0$, this is at most $1/n$. Then the conclusion of
Theorem \ref{thm:wigner} with $X^{(t)}$ in place of $\widehat{X}$ still follows from Lemmas \ref{lem:noiseless} and
\ref{lem:noise}.

\section{Numerical experiments}
\label{sec:numeric}

This section is devoted to comparing our spectral method to various methods for
graph matching, using both synthetic examples and real datasets. 
%\subsection{Experiments on synthetic graphs}
Let us first make a few remarks regarding implementation of graph matching algorithms.  

Similar to the last step of \prettyref{alg:grampa}, many methods that
%\0nbr{do we need to add ``that'' here?} 
we compare consist of a final step that rounds a similarity matrix to produce a permutation estimator. 
Throughout the experiments, for the sake of comparison we always use the linear assignment \prettyref{eq:linearassignment} for rounding, which typically yields noticeably better outcomes than the simple greedy rounding \prettyref{eq:greedyround}. 
%\0nbr{JX. I am curious whether LAP will significantly outperforms greedy matching
%(Here I mean the real greedy matching algorithm).}

For \GRAMPA, Theorem~\ref{thm:wigner} suggests that the regularization parameter $\eta$ needs to be chosen so that $\sigma \lor n^{-0.1} \lesssim \eta \lesssim 1/ \log n$. In practice, one may compute
estimates $\hat{\pi}^\eta$ for different values of $\eta$ and select the one with the minimum objective value $\|A-B^{\hat{\pi}^\eta}\|_F^2$.
We find in simulations that the performance of \GRAMPA is in fact not very
sensitive to the choice of $\eta$, unless $\eta$ is extremely close to zero or larger than one. 
For simplicity and consistency, we apply \GRAMPA to centered and normalized adjacency
matrices and fix $\eta = 0.2$ for all synthetic experiments. 
 
%In the sequel, we use the name \GRAMPA to stand for our spectral method given in \prettyref{alg:grampa}. 

\subsection{Universality of \GRAMPA}

Although the main theoretical result of this work, Theorem~\ref{thm:wigner}, is only proved for the Gaussian Wigner model, 
the proposed spectral method in \prettyref{alg:grampa} (denoted by \GRAMPA) can
in fact be used to match any pair of weighted graphs. 
Particularly, in view of the universality results in the companion paper \cite{FMWX19b}, the performance of \GRAMPA for the Gaussian Wigner model is comparable to that for the suitably calibrated \ER model. 
This is verified numerically in Figure~\ref{fig:model} which we now explain.

\begin{figure}[!ht]
\centering
\includegraphics[clip, trim=4.0cm 8.5cm 4.0cm 9.0cm, width=0.5\textwidth]{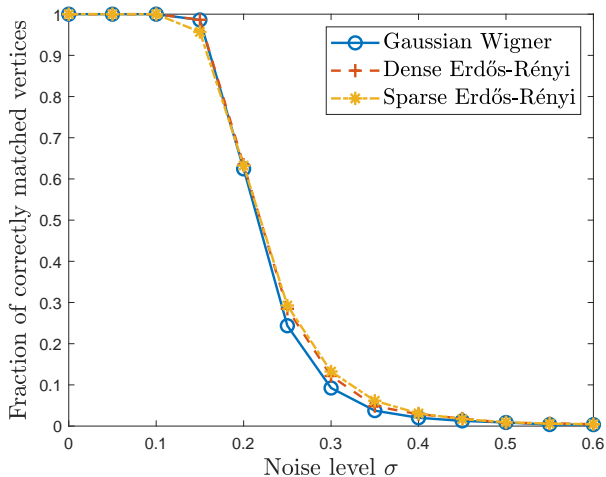}
\caption{Universality of the performance of \GRAMPA on three random graph models with 1000 vertices.}
\label{fig:model}
\end{figure}

Given the latent permutation $\pi^*$, an edge density $p \in (0, 1)$, and a noise parameter $\sigma \in [0, 1]$, we generate two correlated \ER graphs on $n$ vertices with 
adjacency matrices $\bA$ and $\bB$, such that $(\bA_{ij},
\bB_{\pi^*(i),\pi^*(j)})$ are \iid pairs of correlated Bernoulli random
variables with marginal distribution $\Ber(p)$. Conditional on $\bA$, 
\begin{align}
\bB_{\pi^*(i),\pi^*(j)} \sim 
\begin{cases}
\Ber \big( 1 - \sigma^2 (1-p) \big) & \text{ if } \bA_{ij} = 1 , \\
\Ber \big( \sigma^2 p \big) & \text{ if } \bA_{ij} = 0 .
\end{cases}
\label{eq:er-model}
\end{align}
Here the operational meaning of the parameter $\sigma$ is that the fraction of
edges which differ between the two graphs is approximately $2\sigma^2(1-p)
\approx 2 \sigma^2$ for sparse graphs. In particular, the extreme cases of
$\sigma=0$ and $\sigma=1$ correspond to $A$ and $B$ which are perfectly correlated and independent, respectively.

Furthermore, the model parameters in \prettyref{eq:er-model} are calibrated to
be directly comparable with the Gaussian model, so that it is convenient to verify experimentally the universality of our spectral algorithm.
 %and we expect the performance of our spectral algorithm holds universally for both models.
%$$
%B_{\pi^*(i),\pi^*(j)} \sim 
%\begin{cases}
%\Ber(s) & \text{ if } A_{ij} = 1  \\
%\Ber \big( \frac{p (1-s)}{1-p} \big) & \text{ if } A_{ij} = 0 
%\end{cases}
%$$
%where $s = 1 - \sigma^2 (1-p)$. 
Indeed, denote the centered and normalized adjacency matrices by
\begin{align}
A \defn (p(1-p) n)^{-1/2} (\bA - \E[\bA])
\quad \text{ and } \quad
B \defn (p(1-p) n)^{-1/2} (\bB - \E[\bB]) .
\label{eq:normalize}
\end{align} 
Then it is easy to check that $A_{ij}$ and $B_{ij}$ both have mean zero and variance $1/n$, and moreover, 
\begin{align}
\E[ (A_{ij} - B_{\pi^*(i),\pi^*(j)})^2 ] = 2 \sigma^2 / n . 
\label{eq:noise-var}
\end{align} 
%\0nb{This calibration allows the models to be comparable and indeed Fig shows universality. Explain the operational meaning of $\sigma$. $\sigma^2$ plays the role of the fraction of differed edges...}
Note that equation~\eqref{eq:noise-var} also holds for the off-diagonal entries of the Gaussian Wigner model 
$
B = A + \sqrt{2} \sigma Z, 
$
where $A$ and $Z$ are independent $\GOE$ matrices. 
%Hence the noise parameters in the two models are comparable in this sense,\footnote{We use this scaling of $\sigma$ here, rather than $1/\sqrt{2}$ of it, so that the two \ER graphs are uncorrelated when $\sigma = 1$.} and this $\sigma$ is the noise level in the horizontal axis in Figure~\ref{fig:model}. 
In Figure~\ref{fig:model}, we implement \GRAMPA to match pairs of Gaussian Wigner, dense \ER ($p = 0.5$) and sparse \ER ($p = 0.01$) graphs with 1000 vertices,
%\footnote{For \ER graphs we apply \GRAMPA to the centered and normalized matrices $A$ and $B$ in \prettyref{eq:normalize}. 
%\nb{YW: It seems we said in the opening that we will use centered and normalized adj matrices, so I move to delete this footnote.}}
 and plot the fraction of correctly matched pairs of vertices against the noise level $\sigma$, averaged over 10 independent repetitions. 
The performance of \GRAMPA on the three models  is indeed similar, agreeing with the universality results proved in the companion paper \cite{FMWX19b}. 

For this reason, in the sequel we primarily consider the \ER model for synthetic experiments. 
%Consequently, there is no loss of generality that we primarily consider the \ER model in the following synthetic experiments. 
In addition, while \GRAMPA is applicable for matching weighted graphs, many algorithms in the literature were proposed specifically for unweighted 
%\0nb{YW: I changed \ER to unweighted. I presume they are not designed exclusively for \ER ensemble.} 
graphs, so using the \ER model allows us to compare more methods in a consistent setup. 
%Furthermore, since many real-world networks are sparse, testing algorithms on \ER graphs with varying sparsity provides a better indication of their practical performance. 

%In the experiment we apply \GRAMPA to the centered and normalized matrices $A$ and $B$. \0nb{YW: Can we apply \GRAMPA as is without any preprocessing?}

\subsection{Comparison of spectral methods}
\label{sec:numeric-sp}

We now compare the performance of \GRAMPA to several existing spectral methods in the literature. 
Besides the rank-1 method of rounding the outer product of top eigenvectors \eqref{eq:rankone} (denoted by \TopEigenVec), we consider the \IsoRank algorithm of \cite{singh2008global}, the \EigenAlign and \LowRankAlign\footnote{We implement the rank-$2$ version of \LowRankAlign here because a higher rank does not appear to improve its performance in the experiments.} algorithms of \cite{feizi2019spectral}, and Umeyama's method \cite{umeyama1988eigendecomposition} which rounds the similarity matrix \eqref{eq:ume}. 
In Figure~\ref{fig:spec}, we apply these algorithms to match \ER graphs with $100$ vertices\footnote{This experiment is not run on larger graphs because \IsoRank and \EigenAlign involve taking Kronecker products of graphs and are thus not as scalable as the other methods.} and edge density $0.5$. 
For each spectral method, we plot the fraction of correctly matched pairs of vertices of the two graphs versus the noise level $\sigma$, averaged over $10$ independent repetitions. 
While all estimators recover the exact matching in the noiseless case, it is clear that \GRAMPA is more robust to noise than all previous spectral methods by a wide margin. 

\begin{figure}[!ht]
\begin{subfigure}{0.48\textwidth}
%\fbox{
\includegraphics[clip, trim=4.0cm 8.5cm 4.0cm 9.0cm, width=\textwidth]{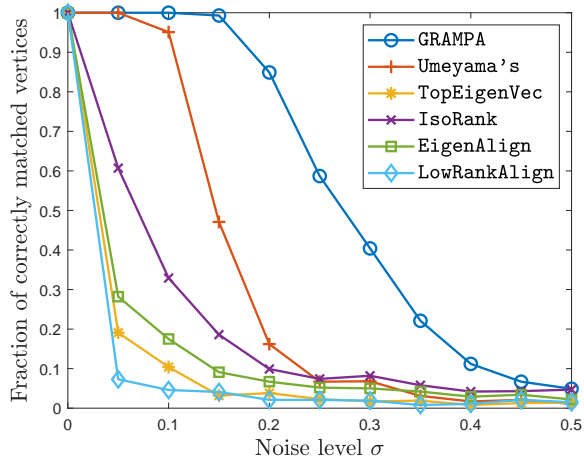}
%}
\caption{Fraction of correctly matched vertices, on \ER graphs with $n=100$
vertices, averaged over $10$ repetitions}
\label{fig:spec}
\end{subfigure}
\hspace{0.5cm}
\begin{subfigure}{0.48\textwidth}
%\fbox{
\includegraphics[clip, trim=4.0cm 8.5cm 4.0cm 9.0cm, width=\textwidth]{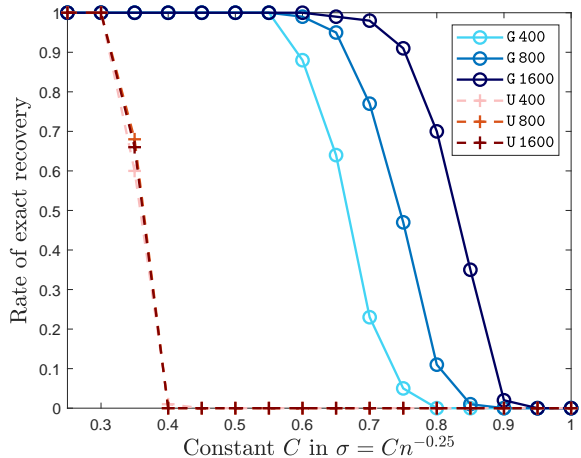}
%}
\caption{Rate of exact recovery for \GRAMPA and Umeyama's method, on \ER graphs
with a varying number of vertices, out of $100$ repetitions}
\label{fig:rate}
\end{subfigure}
\caption{Comparison of six spectral methods for matching \ER graphs with
expected edge density $0.5$ and 100 vertices.}
\label{fig:comp1}
\end{figure}

%\begin{figure}[!ht]
%\centering
%\begin{minipage}[c]{.5\linewidth}
%\includegraphics[clip, trim=11cm 6.5cm 6.2cm 6.3cm, width=\linewidth]{}
%\end{minipage}
%\begin{minipage}[c]{.48\linewidth}
%\includegraphics[clip, trim=1.4cm 6.45cm 2.4cm 6.7cm, width=\linewidth]{}
%\end{minipage} 
%\caption{}
%\label{fig:}
%\end{figure}

%We now turn to a more precise comparison of exact recovery between \GRAMPA and its closest competitor, Umeyama's method.
%%, which exhibited the second best performance in this class of spectral algorithms. 
%Since the bulk of the computational cost in constructing either the similarity matrix \eqref{eq:spectralnew} or \eqref{eq:ume} is due to computing the eigendecompositions of $A$ and $B$, these two spectral methods are equally scalable. 
%Figure~\ref{fig:rate} shows that Umeyama's method is only robust to noise up to
%$\sigma=\frac{1}{\poly(n)}$, whereas we prove in \cite{FMWX19b} that \GRAMPA
%yields exact recovery up to $\sigma=\frac{1}{\polylog(n)}$.

A more precise comparison of \GRAMPA with its closest competitor, Umeyama's method, is given in \prettyref{fig:rate}, which shows that Umeyama's method is only robust to noise up to
$\sigma=\frac{1}{\poly(n)}$, whereas we prove in \cite{FMWX19b} that \GRAMPA
yields exact recovery up to $\sigma=\frac{1}{\polylog(n)}$.
Specifically, we test these two methods on \ER
graphs with edge density $0.5$ and sizes $n=400$, $800$ and $1600$. 
The noise parameter $\sigma$ is set to $C n^{-0.25}$ with varying values of $C$,
where the exponent $-0.25$ is empirically found to be the critical exponent above which Umeyama's method fails. 
Out of $100$ independent trials, we record the fraction of times when the
algorithm exactly recovers the matching between the two graphs, and plot this
quantity against $C$.
From the lines \texttt{U\,400}, \texttt{U\,800}, and \texttt{U\,1600}
corresponding to Umeyama's method on the three respective graph sizes, we see
that the performance of Umeyama's method does not vary with $n$,
supporting that $\sigma \asymp n^{-0.25}$ is the critical threshold for exact recovery by this
method. Conversely, as $n$ increases, the failure of \GRAMPA occurs at a larger
value of $C$, as seen in the curves \texttt{G\,400}, \texttt{G\,800}, and
\texttt{G\,1600}. This aligns with the theoretical result in \cite{FMWX19b} 
that \GRAMPA succeeds for $\sigma=\frac{1}{\polylog(n)}$.

\subsection{Comparison with quadratic programming and Degree Profile}

Next, we consider more competitive graph matching algorithms outside the spectral class. 
Since our method admits an interpretation through the regularized QP \eqref{eq:x-est} or \eqref{eq:constrained}, it is of interest to compare
its performance to (the algorithm that rounds the solution to) the full QP \eqref{eq:ds} with full doubly stochastic constraints, denoted by \QPDS. 
%We denote this method by \QPDS (Quadratic Programming with Doubly Stochastic constraints) in Figure~\ref{fig:comp2}. 
Another recently proposed method for graph matching is Degree Profile
\cite{DMWX18}, for which theoretical guarantees comparable to our results have
been established for the Gaussian Wigner and \ER models. 
%\0nbr{JX. Need to also add ``Gaussian model''?} 
%We also compare it to \GRAMPA in Figure~\ref{fig:comp2}. 

Figure~\ref{fig:qp-dp} plots the fraction of correctly matched vertex pairs by
the three algorithms, on \ER graphs with $500$ vertices and edge density $0.5$, averaged over $10$ independent repetitions. 
\GRAMPA outperforms \DegreeProfile, while \QPDS is clearly the most robust, albeit at a much higher computational cost.
%However, \QPDS has a much higher computational cost.
%, as there is no efficient closed-form expression for the solution. 
Since off-the-shelf QP solvers are extremely slow on instances with $n$ larger than several hundred, we resort to an alternating direction method of multipliers (ADMM) procedure used in \cite{DMWX18}. Still, solving \eqref{eq:ds} is more than $350$ times slower than computing the similarity matrix \eqref{eq:spectralnew} for the instances in Figure~\ref{fig:qp-dp}. 
Moreover, \DegreeProfile is about $15$ times slower.
We argue that \GRAMPA achieves a desirable balance between speed and robustness when implemented on large networks. 

\begin{figure}[t]
\begin{subfigure}{0.48\textwidth}
%\fbox{
\includegraphics[clip, trim=4.0cm 8.5cm 4.0cm 9.0cm, width=\textwidth]{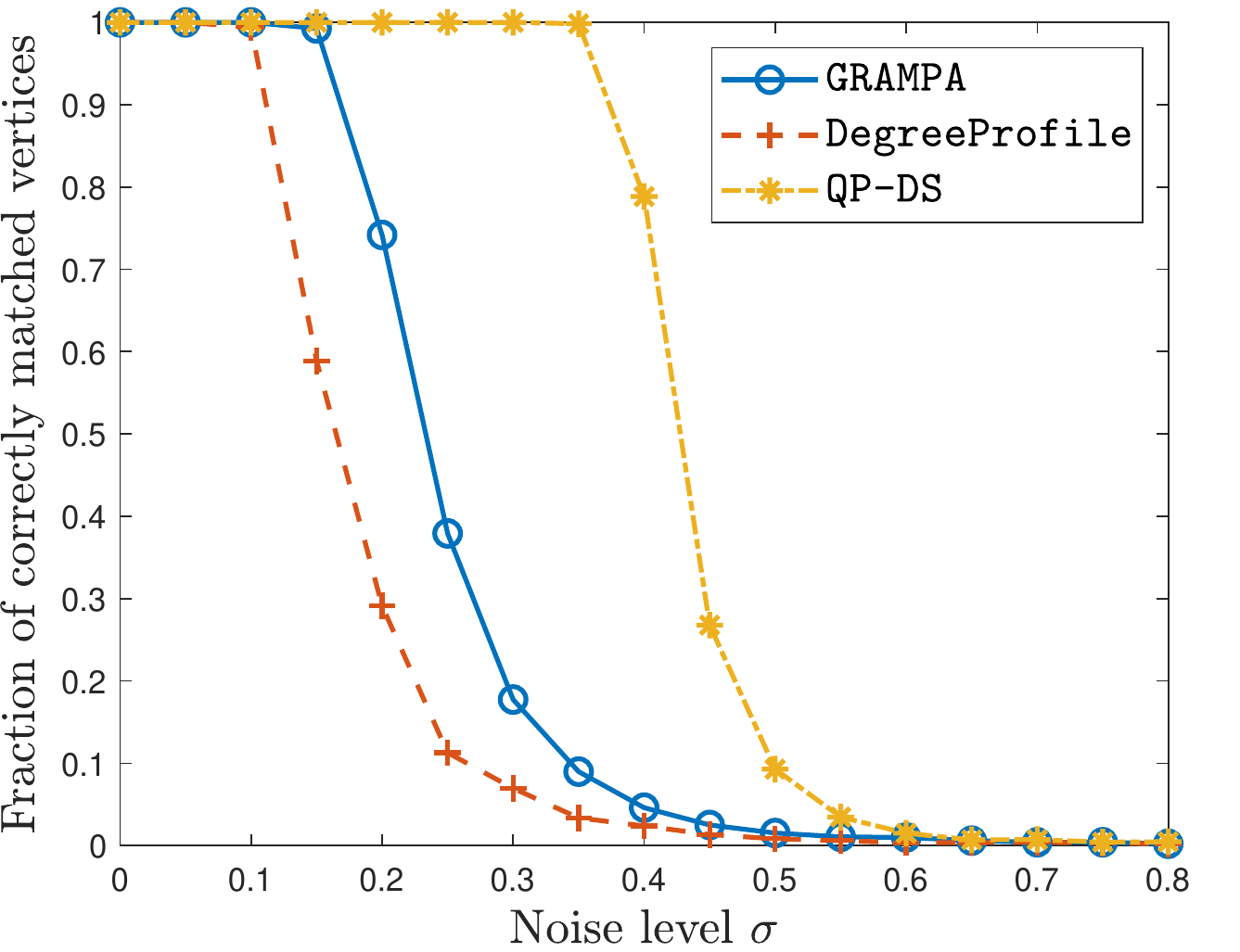}
%}
\caption{Fraction of correctly matched vertices, on \ER graphs with $n=500$
vertices,
%and edge density $0.5$ %\0nb{YW: is $p=0.5$ for both experiments? If so, state it in the main caption.} 
averaged over $10$ repetitions}
\label{fig:qp-dp}
\end{subfigure}
\hspace{0.5cm}
\begin{subfigure}{0.48\textwidth}
%\fbox{
\includegraphics[clip, trim=4.0cm 8.5cm 4.0cm 9.0cm, width=\textwidth]{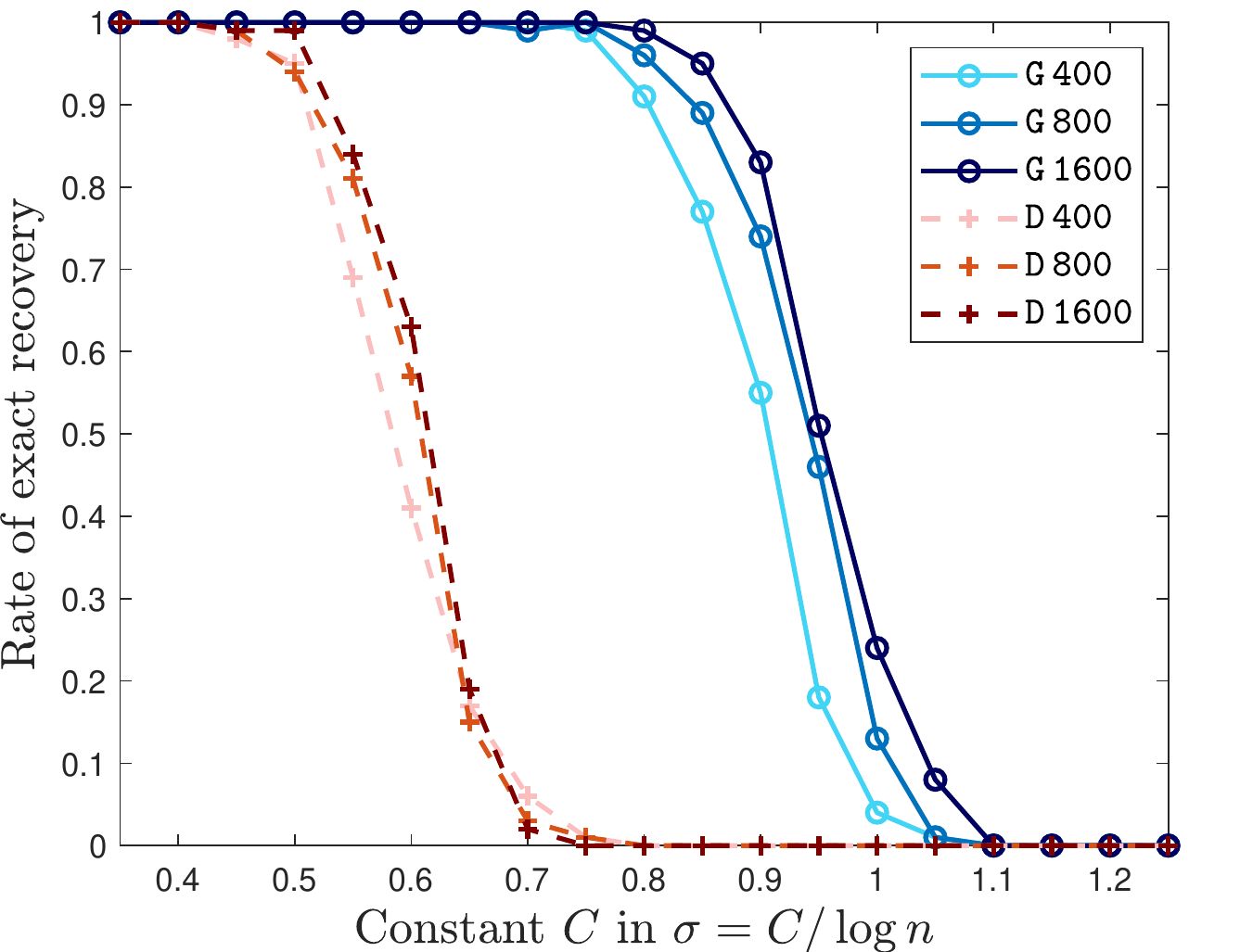}
%}
\caption{Rate of exact recovery for \GRAMPA and \DegreeProfile, on \ER graphs
with a varying number of vertices, out of $100$ repetitions}
%{Slope 2.7824, 2.6190, 2.8867; 15 and 350 times faster}
\label{fig:rate2}
\end{subfigure}
\caption{Comparison of three competitive methods for matching \ER graphs with
expected edge density $0.5$.}
\label{fig:comp2}
\end{figure}

A closer inspection of the exact recovery threshold of \GRAMPA and \DegreeProfile is done in Figure~\ref{fig:rate2}, in a similar way as Figure~\ref{fig:rate}. 
Since Theorem~\ref{thm:wigner} suggests that \GRAMPA is robust to noise up to
$\sigma \lesssim \frac{1}{\log n}$, we set the noise level to be $\sigma = C / \log n$, and plot the fraction of exact recovery against $C$.
The results for \ER graphs of size $n = 400, 800, 1600$ and for the two methods
\GRAMPA and \DegreeProfile are represented by the curves \texttt{G\,400},
\texttt{G\,800}, \texttt{G\,1600} and \texttt{D\,400}, \texttt{D\,800},
\texttt{D\,1600} respectively. 
These results suggest that \GRAMPA is more robust than \DegreeProfile by at least a constant factor. 
%The curves for \GRAMPA (and perhaps also for \DegreeProfile) are moving slightly towards the
%right with increasing $n$, numerically confirming at least a lower bound of
%order $1/\log n$ for the noise-tolerance threshold. These
%curves also suggest that \GRAMPA is more robust by at least a constant factor. 
%We refrain from conjecturing the precise noise-tolerance thresholds of these
%methods at the logarithmic level, as studying logarithmic behavior requires much
%larger network size.

\subsection{More graph ensembles}

%So far we have focused our experiments on the \ER model which is comparable to the Gaussian Wigner model in terms of the behavior of our spectral method. 
To demonstrate the performance of \GRAMPA on other models, we turn to sparser regimes of the \ER model, as well as other random graph ensembles. 
In Figure~\ref{fig:graphs}, we compare the performance of \GRAMPA and
\DegreeProfile (the two fast and robust methods from the preceding experiments)
on correlated pairs of sparse \ER graphs, stochastic blockmodels,
and power-law graphs. 
Following \cite{Pedarsani2011}, we generate these pairs by sampling a ``mother graph'' according to such a
model with edge density $p/s$ for $0<p<s \leq 1$, and then generating $\bA$ and $\bB$ by deleting
each edge independently with probability $1-s$. This yields marginal edge
density $p$ in both $\bA$ and $\bB$.
%\nb{ZF: I reworded the above a bit; double-check this?}
We apply \GRAMPA with the centered and normalized matrices $A$ and $B$ from \eqref{eq:normalize}
as input, with $p$ bing the marginal edge density. In each figure, the fraction
of correctly matched pairs of vertices is plotted against the effective noise level
\[\sigma=\sqrt{ \frac{ 1-s }{ 1-p } },\]
for graphs with $1000$ vertices, averaged over $10$ independent repetitions.
One may verify that the above procedure of generating $(\bA,\bB)$ and the definition
of $\sigma$ both agree with the previous definition \prettyref{eq:er-model} in the \ER setting.

\begin{figure}[t]
\begin{subfigure}{0.33\textwidth}
%\fbox{
\includegraphics[clip, trim=4.0cm 8.5cm 4.0cm 9.0cm, width=\textwidth]{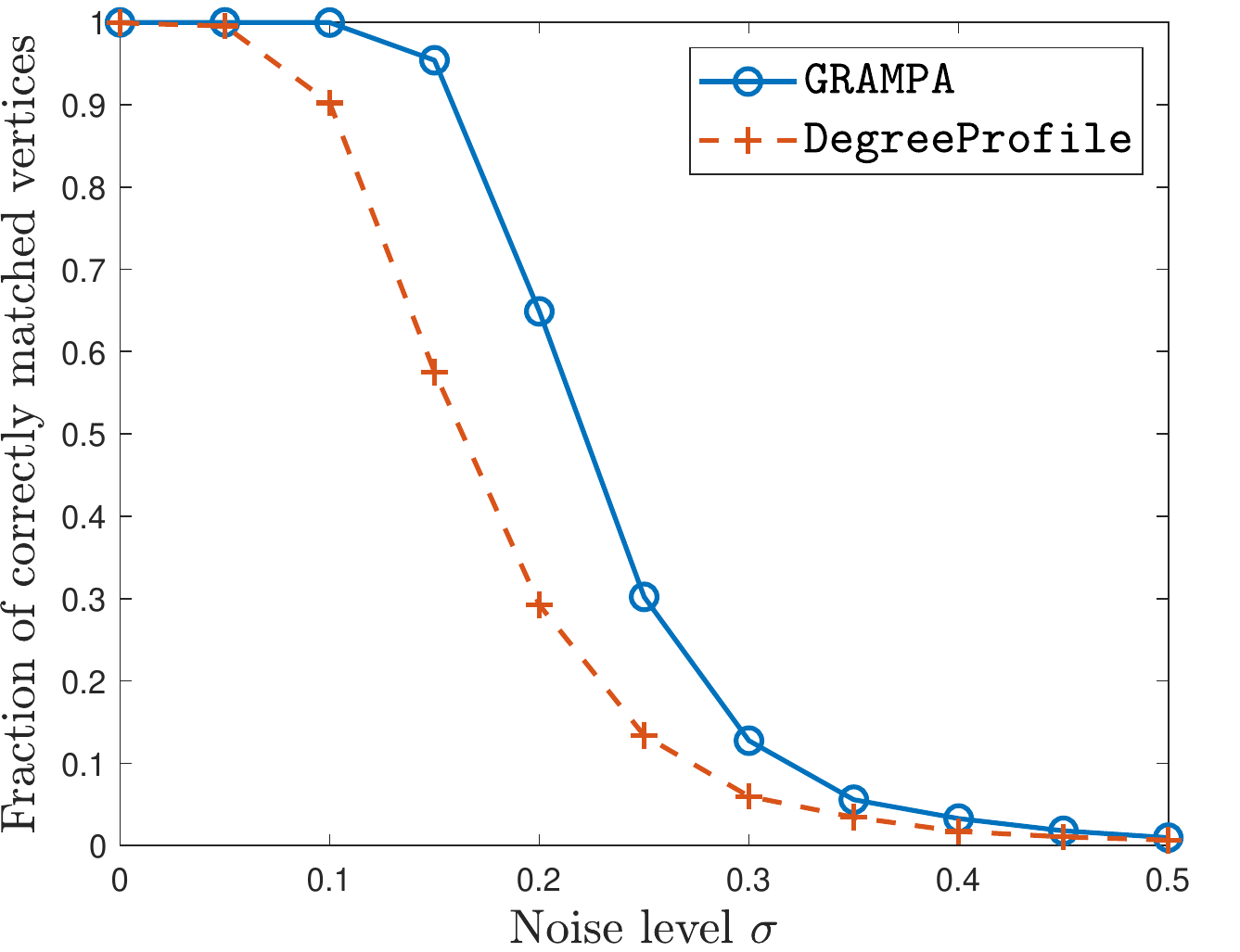}
%}
\caption{\ER graph with \\ $p = 0.01$}
\label{fig:er1}
\smallskip
%\fbox{
\includegraphics[clip, trim=4.0cm 8.5cm 4.0cm 9.0cm, width=\textwidth]{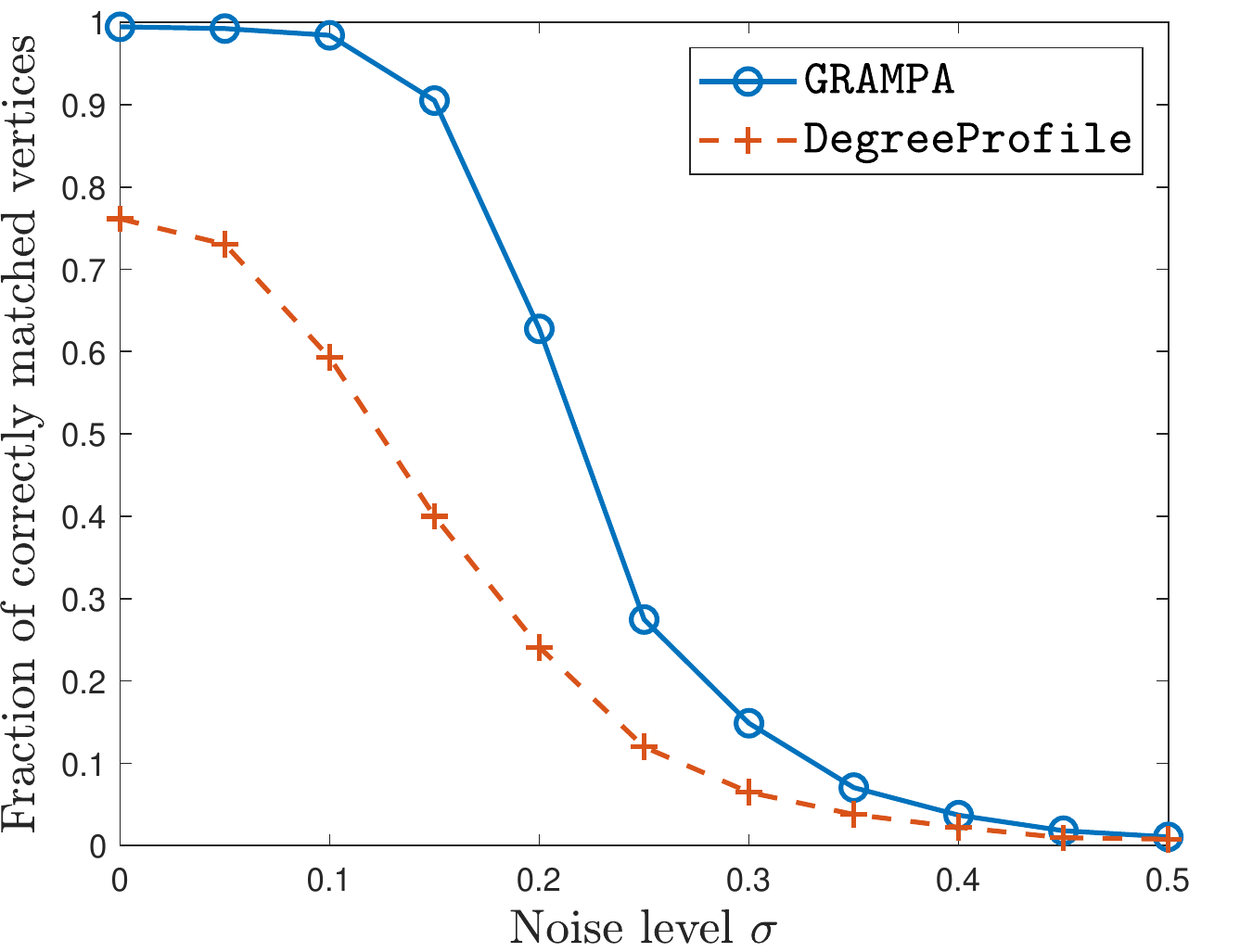}
%}
\caption{\ER graph with \\ $p = 0.005$}
\label{fig:er2}
\end{subfigure}
\begin{subfigure}{0.33\textwidth}
%\fbox{
\includegraphics[clip, trim=4.0cm 8.5cm 4.0cm 9.0cm, width=\textwidth]{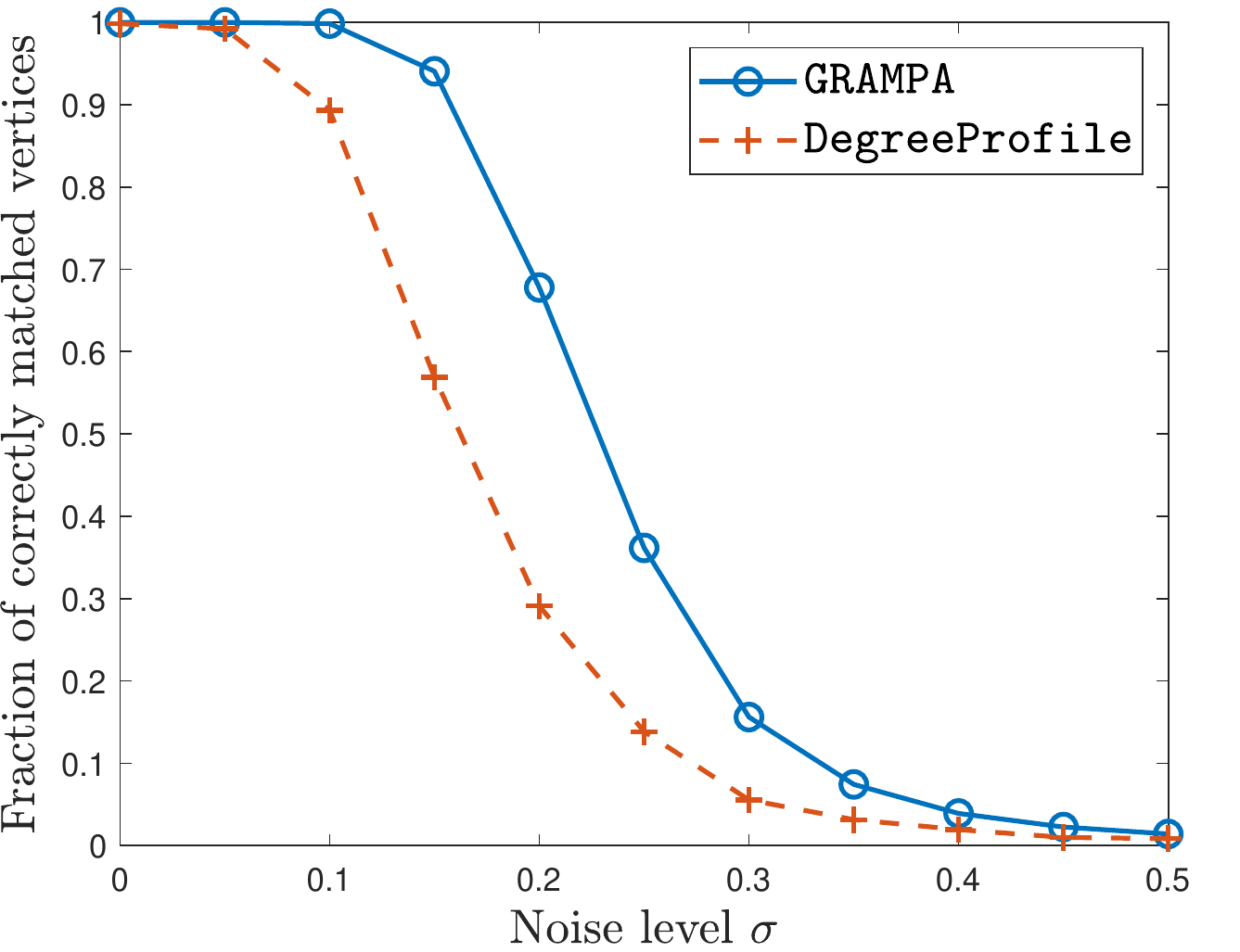}
%}
\caption{Stochastic blockmodel with \\ $p=0.01$, $\pin = 0.016$, and $\pout = 0.004$}
\label{fig:sbm1}
\smallskip
%\fbox{
\includegraphics[clip, trim=4.0cm 8.5cm 4.0cm 9.0cm, width=\textwidth]{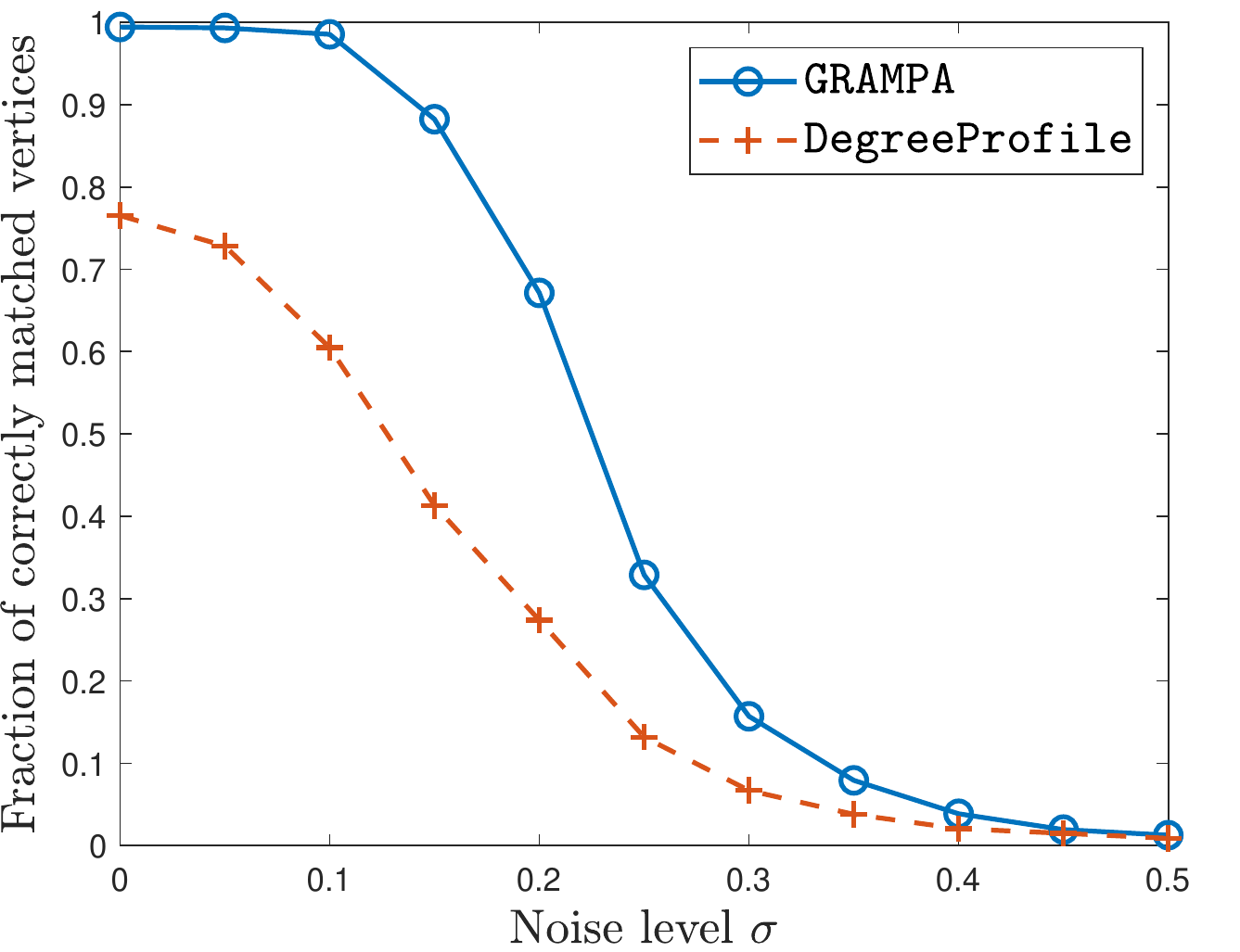}
%}
\caption{Stochastic blockmodel with \\ $p=0.005$, $\pin = 0.008$, and $\pout = 0.002$}
\label{fig:sbm2}
\end{subfigure}
\begin{subfigure}{0.33\textwidth}
%\fbox{
\includegraphics[clip, trim=4.0cm 8.5cm 4.0cm 9.0cm, width=\textwidth]{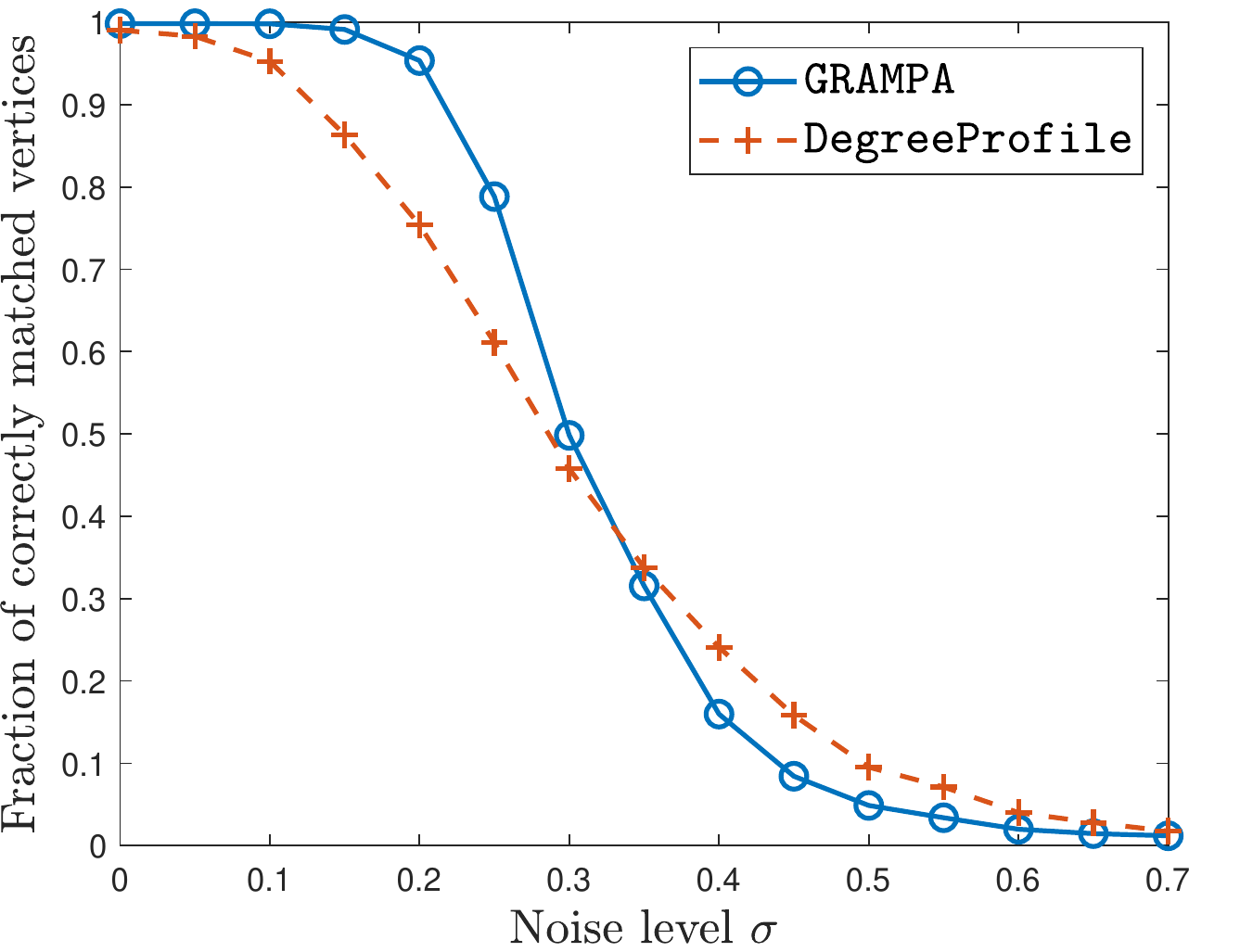}
%}
\caption{Power-law graph with \\ $p = 0.01$}
\label{fig:pl1}
\smallskip
%\fbox{
\includegraphics[clip, trim=4.0cm 8.5cm 4.0cm 9.0cm, width=\textwidth]{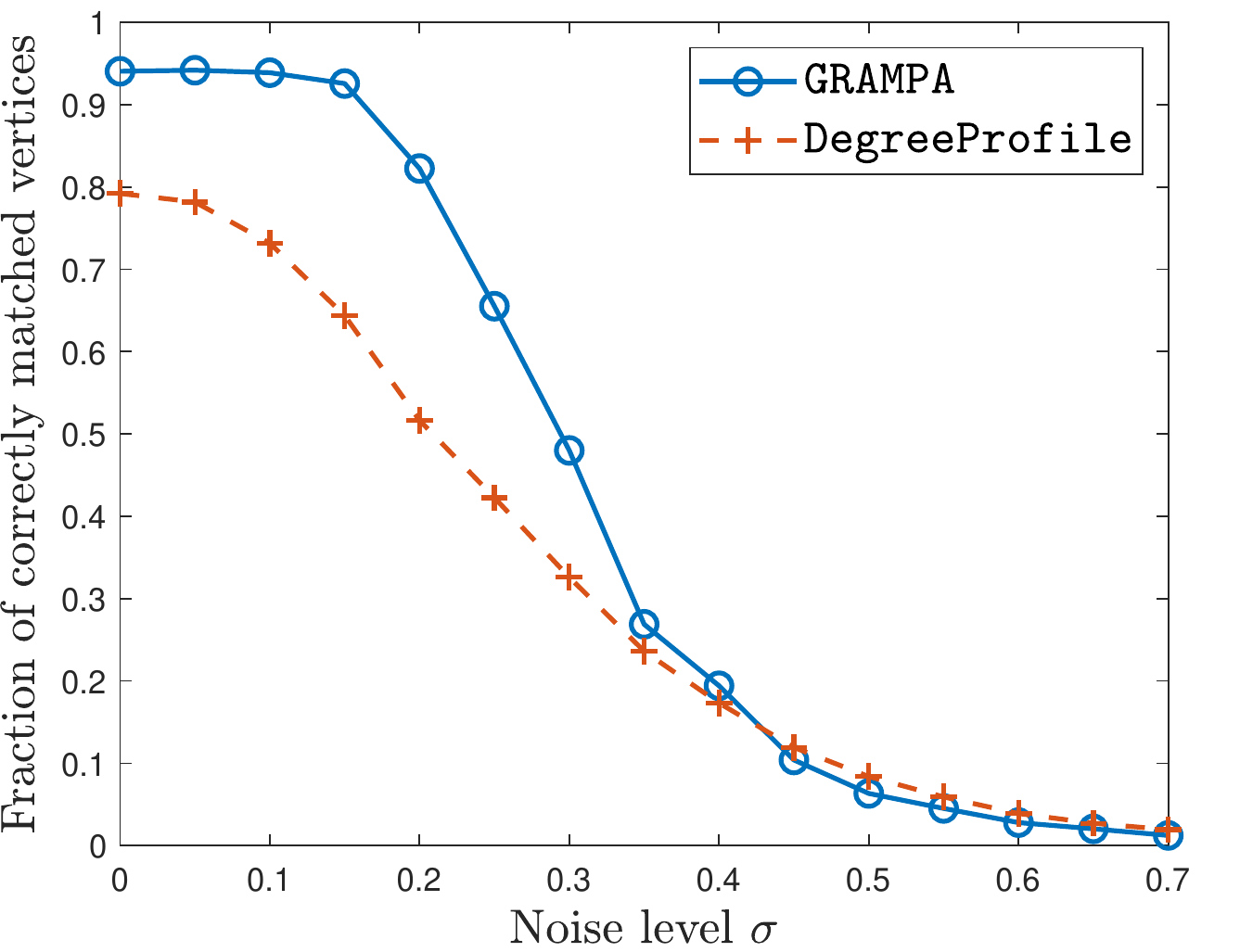}
%}
\caption{Power-law graph with \\ $p = 0.005$}
\label{fig:pl2}
\end{subfigure}
\caption{Comparison of \GRAMPA and \DegreeProfile on synthetic networks}
\label{fig:graphs}
\end{figure}

In Figures~\ref{fig:er1} and~\ref{fig:er2}, we consider \ER graphs with edge
density $p=0.01$ and $0.005$. 
Note that the sharp threshold of $p$ for the connectedness of an \ER graph with
$1000$ vertices is $\frac{ \log n}{ n} \approx 0.0069$
\cite{erds1960evolution}, below which the graphs contain isolated
vertices whose matching is non-identifiable. We see that the performance of \GRAMPA is better than \DegreeProfile in both settings, and particularly in the sparser regime.
%We see that the performance of \GRAMPA is quite consistent for all levels of
%sparsities in the \ER model. \DegreeProfile is weaker than \GRAMPA for sparsity larger than $\frac{ \log n }{ n }$, and even fails at exact recovery with constant probability when sparsity is so low that the graphs are disconnected. 
%Therefore, in addition to being more robust, \GRAMPA respects connected components, which is an important feature that \DegreeProfile does not possess. 
%\nbr{JX. I feel the discussions in the last two sentences are not that clear. First, I think
%\DegreeProfile is also weaker than \GRAMPA for sparsity lower than $\log n/n$, right? 
%Second, when the sparsity is lower than $\log n/n$, 
%exact recovery with high probability is in fact information-theoretically impossible due to the existence of isolated nodes. so I think even \GRAMPA will fail at exact recovery with constant probability when sparsity is so low. Finally, it is unclear what does ``respects connected components '' exactly mean.}
%\nb{CM: yep the information limit is true... I guess my interpretation is simply incorrect. 
%I tested a $G(1000,0.005)$ ER graph on matlab, which turns out to have 5 isolated vertices and 6 connected components. 
%I guess there are so few of them so we don't really see on the figures. 
%Do you have an interpretation of what the figures suggest and why DegreeProfile is not as good here? 
%In any case, since this is what we get from simulation, I suppose we can just state more vaguely that GRAMPA works better than Degree Profile in sparser regimes?}
%

In Figures~\ref{fig:sbm1} and~\ref{fig:sbm2}, we consider the stochastic blockmodel \cite{Holland83} with two communities each of size $500$. 
The probability of an edge between vertices in the same (resp.~different) community is
denoted by $\pin$ (resp.~$\pout$). 
The values of $\pin$ and $\pout$ are chosen so that the overall expected
edge densities are $p=0.01$ and $0.005$ as in the \ER case. 
We observe a similar comparison of the two methods as in the \ER setting.

%\nbr{JX. Here shall we emphasize that \GRAMPA will not match vertices from different connected components in the parent graph, while \DegreeProfile is possible to do so?} 
%\nb{CM: Is this true or the other way round as suggested by the simulation?}
%\nb{ZF: I changed the above to $q$ and $r$, so as to not overload $p$. See also
%the updated figure caption. I'm not sure if we should mention different
%connected components here, as the two communities are not actually disconnected
%for positive $r$.}

Finally, in Figures~\ref{fig:pl1} and~\ref{fig:pl2}, we consider power-law
graphs that are generated according to the following version of the
Barab\'asi-Albert preferential attachment model \cite{barabasi1999emergence}: We start with two vertices connected by one edge. 
Then at each step, denoting the degrees of the existing $k$ vertices by $d_1,
\dots d_k$, we attach a new vertex to the graph by connecting it to the $j$-th
existing vertex independently with probability $\max(C d_j / (\sum_{i=1}^k d_i
),1)$ for each $j=1,\ldots,k$.\footnote{Since a preferential attachment graph
is connected by convention, we may repeat this step until the new vertex is
connected to at least one existing vertex.} 
This process is iterated until the graph has $n$ vertices. 
Here, $C$ is a parameter that determines the final edge density.

%\nbr{JX. Can $C$ be large enough so that $C d_j / (\sum_{i=1}^k d_i )$ exceeeds $1$? I have not seen this variation of PA graphs before. Maybe a reference is useful here?}
%\nb{CM: I added $\le 1$ to be precise, as it doesn't exceed $1$ in the synthetic examples. 
%I don't really have a reference for this variation of PA graph with the tuning constant $C$, which I added. 
%I couldn't find a PA graph model that allows me to tune the sparsity so precisely that I can compare it to the other two models in a clean way, so this is the best I can give. 
%}
%\nb{ZF: I'm OK with this version, and calling it Barabasi-Albert. I changed the
%above to $\max(...,1)$.}

%There are variations of the Barab\'asi-Albert model that generate the graphs in slightly different ways, but due to their asymptotic similarity, it suffices to focus on one of them to evaluate graph matching methods. 

As shown in Figure~\ref{fig:pl1}, 
for matching correlated power-law graphs with overall expected edge density $p
= 0.01$,  \GRAMPA is more noise resilient than \DegreeProfile in terms of
exact recovery. As the noise grows, the performance of \GRAMPA decays faster
than \DegreeProfile in terms of the fraction of correctly matched pairs. 
In Figure~\ref{fig:pl2}, for sparser power-law graphs with expected edge density $p = 0.005$, we again observe that \GRAMPA has significantly better performance than \DegreeProfile. 
Note that in this sparse regime, neither method can achieve exact recovery even in the noiseless
case due to the non-trivial symmetry of the graph arising from, for example, multiple leaf vertices incident to a common parent.
%, so that relabeling these vertices gives an isomorphism of the graph which cannot be identified. 
%Since real-world networks are often close to following a preferential attachment model, we revisit this observation of non-identifiability when studying the real dataset in the next subsection. 
We revisit the issue of non-identifiability when studying the real dataset in the next subsection. 

\subsection{Networks of autonomous systems}

%Our newly proposed method \GRAMPA has been the preferable choice for matching the above synthetic graphs thanks to its robustness and scalability. 
%We now further corroborate this finding by comparing it with the other promising competitor \DegreeProfile on real-world networks. 
We corroborate the improvement of \GRAMPA over \DegreeProfile 
using quantitative benchmarks on a time-evolving
real-world network of $n=10000$ vertices. 
Here, for simplicity, we apply both methods to the
\emph{unnormalized} adjacency matrices, and set $\eta=1$ for \GRAMPA. We find
that the results are not very sensitive to this choice of $\eta$.
Although \QPDS yields better performance in \prettyref{fig:comp2}, it is extremely slow to run on such a large network, so we omit it from the comparison here.

%\nb{ZF: Can we add a statement here
%that although QP-DS yielded better performance in Figure 4, it is infeasible,
%or extremely slow, to run QP-DS on networks of this size, so we could not
%perform this comparison here?}

%Regarding the choice of $\eta$ in \GRAMPA, let us remark the following. 
%Since real-world networks are usually not normalized in a canonical way, it is not obvious how to choose the parameter $\eta$ which is not invariant under rescaling of the adjacency matrices. 
%Nevertheless, for any fixed $\eta$, once we obtain our spectral estimator $\hat \pi^\eta$, we can compute the objective value $\langle A, B^{\hat \pi^\eta} \rangle$. 
%Therefore,
%at the cost of higher computational usage, we can easily choose a parameter $\eta$ among a set of alternatives that yields the largest objective value. 
%Moreover, similar to the situation in simulations, we find that the performance of \GRAMPA is not very sensitive to the choice of $\eta$ as long as it is in a reasonable range, so there is no need for extensive tuning. 
%In the experiment below, we simply set $\eta = 1$ without centering or normalizing 
%%\0nb{YW: do you mean ``without centering''? normalization plays no role} 
%the adjacency matrices. 
%\nb{ZF: I have added some of the discussion in this paragraph to the start of Section 4.
%Here, I propose we just shorten this to: ``Here, for simplicity, we apply both methods to the
%\emph{unnormalized} adjacency matrices, and set $\eta=1$ for \GRAMPA. We find
%that the results are not very sensitive to this choice of $\eta$.''}

We use a subset of the Autonomous Systems dataset from the University of Oregon
Route Views Project \cite{UOregon}, available as part of the Stanford Network Analysis Project \cite{snapnets, leskovec2005graphs}. 
The data consists of instances of a network of autonomous systems observed on nine days between March 31, 2001 and May 26, 2001. 
Edges and (a small fraction of) vertices of the network were added and deleted over time. 
In particular, the number of vertices of the network on the nine days ranges from 10,670 to 11,174 and the number of edges from 22,002 to 23,409. 
The labels of the vertices are known. 

To test the graph matching methods, we consider 10,000 vertices of the network that are present on all nine days.
The resulting nine graphs can be seen as noisy versions of each other, with correlation decaying over time. 
%Comparing each graph to that on the first day March 31,
%we can view the noise level as increasing over time,
%with no noise in the comparison of the first day to itself.
We apply \GRAMPA and \DegreeProfile to match each graph to that on
the first day of March 31, with vertices randomly permuted.

\begin{figure}[!ht]
\begin{subfigure}{0.48\textwidth}
%\fbox{
\includegraphics[clip, trim=3.5cm 8.5cm 4.0cm 8.7cm, width=\textwidth]{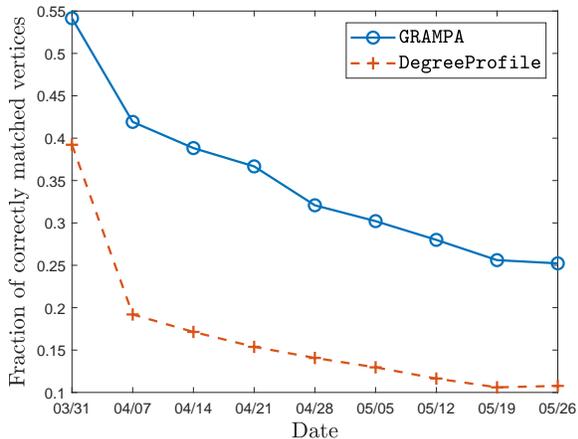}
%}
\caption{Fraction of correctly matched vertices}
\label{fig:corr}
\end{subfigure}
\hspace{0.5cm}
\begin{subfigure}{0.48\textwidth}
%\fbox{
\includegraphics[clip, trim=3.5cm 8.5cm 4.0cm 8.7cm, width=\textwidth]{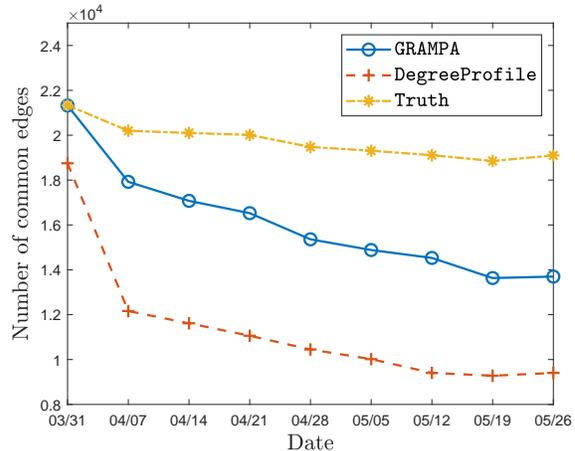}
%}
\caption{Number of common edges}
\label{fig:val}
\end{subfigure}
\begin{subfigure}{0.48\textwidth}
%\fbox{
\includegraphics[clip, trim=3.5cm 8.5cm 4.0cm 8.7cm, width=\textwidth]{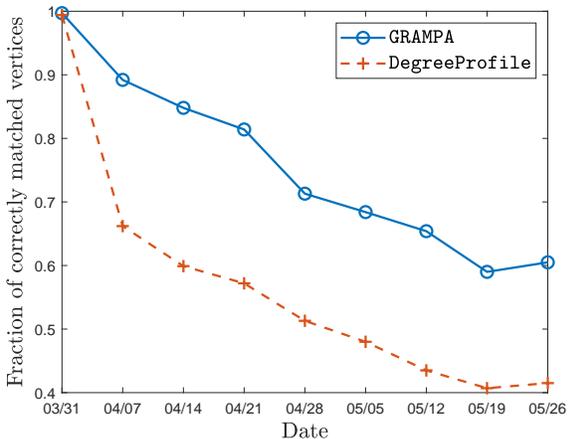}
%}
\caption{Fraction of correctly matched vertices, in high-degree subgraph of 1000 vertices}
\label{fig:corr-2}
\end{subfigure}
\hspace{0.5cm}
\begin{subfigure}{0.48\textwidth}
%\fbox{
\includegraphics[clip, trim=3.5cm 8.5cm 4.0cm 8.7cm, width=\textwidth]{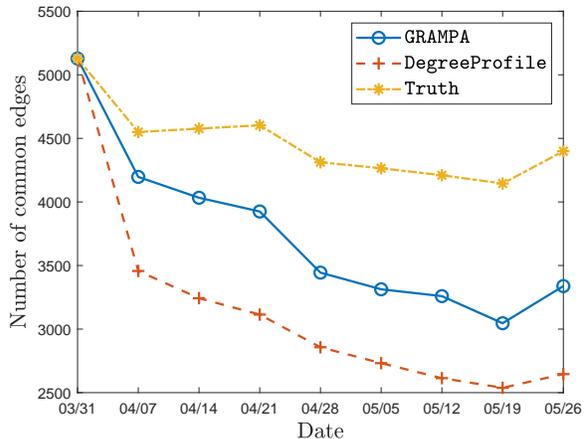}
%}
\caption{Number of common edges, in high-degree subgraph of 1000 vertices}
\label{fig:val-2}
\end{subfigure}
\caption{Comparison of \GRAMPA and \DegreeProfile for matching networks of autonomous systems on nine days to that on the first day}
\label{fig:auto-sys-1}
\end{figure}

In Figure~\ref{fig:corr}, we plot the fraction of correctly matched pairs of vertices against the chronologically ordered dates. 
\GRAMPA correctly matches many more pairs of vertices than \DegreeProfile for all nine days. 
As expected, the performance of both methods degrades over time as the network becomes less correlated with the original one.

For the same reason as in the power-law graphs in \prettyref{fig:pl2}, even the matching of the
graph on the first day to itself is not exact. 
In fact, there are over 3,000 degree-one vertices in this graph, and some of them are attached to the same high-degree vertices, so the exact matching is non-identifiable. 
Thus for a given matching $\hat \pi$, an arguably more relevant figure of merit is the number of common edges, i.e.\ the (rescaled)
objective value $\langle A, B^{\hat \pi} \rangle / 2$.
%where $A$ is the graph on the first day and $B$ is the graph on another day. 
We plot this in Figure~\ref{fig:val} together with the value for the
ground-truth matching. The values of \GRAMPA and of the ground-truth matching
are the same on the first day, indicating that
\GRAMPA successfully finds an automorphism of this graph,
%\0nb{YW: I don't understand. It's only 0.6?} 
while \DegreeProfile fails. Furthermore, \GRAMPA consistently recovers a matching with more common edges over the nine days. 
%and as expected, the performance of both methods exhibit weaker performance as noise increases over time. \0nb{YW: the last sentence already appeared before.} 

As in many real-world networks, high-degree vertices here are hubs in the
network of autonomous systems and play a more important role. 
Therefore, we further evaluate the two methods by comparing their performance on
subgraphs induced by high-degree vertices.
More precisely, we consider the 1,000 vertices that have the highest degrees in the graph on the first day, March 31. 
In Figures~\ref{fig:corr-2} and~\ref{fig:val-2},
we still use the matchings between the entire networks that generated
Figures~\ref{fig:corr} and~\ref{fig:val}, but evaluate the correctness only on those top 1,000 high-degree vertices.
We observe that both methods succeed in exactly matching the subgraph from the
first day to itself, and in general yield much better matchings for the
high-degree vertices than for the remainder of the graph.
Again, \GRAMPA produces better results than \DegreeProfile on these subgraphs
over the nine days, in both measures of performance.

\section{Conclusion}

We have proposed a highly practical spectral method, GRAMPA, for matching a pair of
edge-correlated graphs. By using a similarity matrix that is a
weighted combination of outer products
$u_iv_j^\top$ across all pairs of eigenvectors of the two adjacency
matrices,
GRAMPA exhibits significantly improved noise resilience over previous spectral approaches.
We showed in this work that GRAMPA achieves exact recovery of the latent
matching in a correlated Gaussian Wigner model, up to a noise level $\sigma
\lesssim \frac{1}{\log n}$. In the companion paper \cite{FMWX19b}, we establish
a similar universal guarantee for \ER graphs and other correlated Wigner
matrices, up to a noise level $\sigma \lesssim \frac{1}{\polylog n}$.
GRAMPA exhibits improved recovery accuracy over previous
spectral methods as well as the state-of-the-art degree profile algorithm in
\cite{DMWX18}, on a variety of synthetic graphs and also on a real network
example.

The similarity matrix \prettyref{eq:spectralnew} in GRAMPA can be interpreted as a ridge-regularized further relaxation of the well-known quadratic relaxation (\ref{eq:ds}) of the QAP, where the doubly-stochastic constraint is replaced by $\ones^\top X \ones=n$. In the companion paper \cite{FMWX19b}, we also analyze a
tighter relaxation with constraints $X\ones=\ones$, and establish similar guarantees. In
synthetic experiments on small graphs, we found that solving the full quadratic
program (\ref{eq:ds}), followed by the same rounding procedure as used in
GRAMPA, yields better recovery accuracy in noisy settings.
However, 
unlike GRAMPA, solving (\ref{eq:ds}) does not scale to large networks,
%solving (\ref{eq:ds}) quickly becomes intractible as the number of vertices exceeds a few thousand, 
and the properties of its solution 
currently lack theoretical understanding. We leave these as open problems for
future work.

\appendix

\section{Concentration inequalities for Gaussians}\label{appendix:gaussian}

We collect auxiliary results on concentration of polynomials of Gaussian variables.

\begin{lemma} \label{lem:g-concentrate}
Let $z$ be a standard Gaussian vector in $\R^n$. For any fixed $v \in \R^n$ and $\delta > 0$, it holds with probability at least $1 - \delta$ that
$$
|v^\top z| \le \|v\|_2 \sqrt{2 \log (1/\delta)} .
$$
\end{lemma}

\begin{lemma}[Hanson-Wright inequality] \label{lem:hw}
Let $z$ be a sub-Gaussian vector in $\R^n$, and let $M$ be a fixed matrix in $\complex^{n \times n}$. Then we have
%\begin{align} 
%\p \Big\{ \big| z^\top K z - \E[ z^\top K z ] \big| > t \Big\} \le 2 \exp \Big( - c \Big( \frac{t^2}{\|K\|_F^2} \land \frac{t}{\|K\|} \Big) \Big) .
%\label{eq:hw}
%\end{align}
with probability at least $1 - \delta$ that 
\begin{align}
    |z^\top M z - \Tr M| &\leq C \|z\|_{\psi_2}^2 \Big( \|M\|_F \sqrt{ \log (1/\delta) } +  \| M \|
    \log (1/\delta) \Big) \label{eq:hw} \\
&\leq 2C \|z\|_{\psi_2}^2 \|M\|_F \log (1/\delta) , \notag
\end{align}
where $C$ is a universal constant and $\|z\|_{\psi_2}$ is the sub-Gaussian norm of $z$. 
\end{lemma}
%\0nb{YW: We need the version for complex $K$. Just a sanity check: the complex version is implied by the real-valued version, since $\|K\|_F^2=\|\Re(K)\|_F^2+\|\Im(K)\|_F^2$ and $\|K\|\asymp \|\Re(K)\|+\|\Im(K)\|$. Please check if the reference below contains a complex version.}
%\0nbr{CM: yes, the complexification is in Section 3.1 of the cited paper}

%\begin{proof}
%In view of the relations $\|M\|_F^2=\|\operatorname{Re} M \|_F^2+\|\Im M \|_F^2$ and $\|M\|\asymp \|\operatorname{Re} M \|+\|\Im M \|$, the desired inequality immediately follows from the real-valued version of the Hanson-Wright inequality \cite[Theorem 1.1]{RudVer13}. 
%\end{proof}

See \cite[Section 3.1]{RudVer13} for the complex-valued version of the Hanson-Wright inequality. 
The following lemma is a direct consequence of \prettyref{eq:hw},
by taking $M$ to be a diagonal matrix.

\begin{lemma} \label{lem:g-norm}
Let $z$ be a standard Gaussian vector in $\R^n$. For an entrywise nonnegative vector $v \in \R^n$, it holds with probability at least $1 - \delta$ that
$$
    \Big| \sum_{i=1}^n v_i z_i^2 - \sum_{i=1}^n v_i \Big| \le C\Big(\|v\|_2 \sqrt{
        \log (1/\delta) } + \|v\|_\infty \log (1/\delta)\Big).
$$
In particular, it holds with probability at least $1 - \delta$ that
$$
    \big| \|z\|_2^2 - n \big| \le C\Big(\sqrt{n \log (1/\delta) } + \log
    (1/\delta)\Big) .
$$
%\0nbr{JX. Here it seems that we also have an extra term $\log (1/\delta)$; otherwise, we need 
%the assumption that $\log (1/\delta) \lesssim n$.}
\end{lemma}

\begin{theorem}[Hypercontractivity concentration {\cite[Theorem~1.9]{SchSvi12}}] \label{thm:hyper}
Let $z$ be a standard Gaussian vector in $\R^n$, and let $f(z_1, \dots, z_n)$ be a degree-$d$ polynomial of $z$. Then it holds that
$$
\p \Big\{ \big| f(z) - \E[ f(z) ] \big| > t \Big\} \le e^2 \exp \bigg[ - \Big( \frac{ t^2 }{ C \, \Var [ f(z) ] } \Big)^{1/d} \bigg] ,
$$
where $\Var[ f(z) ]$ denotes the variance of $f(z)$ and $C > 0$ is a universal constant.
\end{theorem}
%For this specific form and a proof, see \cite[Theorem~1.9]{SchSvi12}.

%\begin{lemma} \label{thm:mgf-bd}
%Let $z$ be a standard Gaussian vector in $\R^m$, and let $A$ be a fixed matrix in $\R^{n \times m}$. Then we have
%$$
%\E \exp (- \|A z\|_2^2) = \prod_{i=1}^m (1 + 2 \Sigma_{ii}^2)^{-1/2} . 
%$$
%\end{lemma}
%
%\begin{proof}
%Consider the singular decomposition $A = U \Sigma V^\top$ and let $x = V^\top z$. 
%Then $x$ is also a standard Gaussian vector and $z^\top A^\top A z = x^\top \Sigma^2 x$. 
%Thus it holds that 
%$$
%\E \exp (- \|A z\|_2^2) 
%= \E \exp (- x^\top \Sigma^2 x)
%= \prod_{i=1}^m \E \exp ( - \Sigma_{ii}^2 x_i^2)
%= \prod_{i=1}^m (1 + 2 \Sigma_{ii}^2)^{-1/2} . 
%$$
%%By Jensen's inequality, we have 
%%$$
%%\frac 1m \sum_{i=1}^m \log (1 + 2 \Sigma_{ii}^2) \le \log \Big( 1 + \frac 2m \sum_{i=1}^m \Sigma_{ii}^2 \Big) = \log ( 1 + 2 \|A\|_F^2 / m ) . 
%%$$
%%Combining the above two results completes the proof. 
%\end{proof}

Finally, the following result gives a concentration inequality in terms of the restricted Lipschitz constants, obtained from the usual Gaussian concentration of measure plus a Lipschitz extension argument.
\begin{lemma}
\label{lmm:lipext}	
	Let $B\subset \reals^n$ be an arbitrary measurable subset. Let $F: \reals^n \to \reals$ such that $F$ is $L$-Lipschitz on $B$. 
	Let $X \sim N(0,\identity_n)$. Then for any $t>0$,
	\[
	\prob{|F(X)- \Expect F(X)| \geq t + \delta} \leq \exp\pth{-\frac{ct^2}{L^2}} + \epsilon,
	\]
	where $c$ is a universal constant, $\epsilon = \prob{X \notin B}$ and $\delta = 2 \sqrt{\epsilon ( nL^2+F(0)^2 + \Expect[F(X)^2] )}$.
\end{lemma}
\begin{proof}
Let $\tilde F: \reals^n\to\reals$ be an $L$-Lipschitz extension of $F$, e.g., $\tilde F(x) = \inf_{y\in B} F(y)+L\|x-y\|$. Then by the Gaussian concentration inequality (cf.~e.g.~\cite[Theorem 5.2.2]{Vershynin-HDP}), we have
\[
\prob{|\tilde F(X)- \Expect \tilde F(X)| \geq t} \leq \exp\pth{-\frac{ct^2}{L^2}}.
\]
It remains to show that $|\Expect F(X)-\Expect \tilde F(X)| \leq \delta$. Indeed, by Cauchy-Schwarz,
$|\Expect F(X)-\Expect \tilde F(X)| \leq \Expect |F(X-\tilde F(X)|\indc{X \notin B} \leq \sqrt{\epsilon \Expect[|F(X-\tilde F(X)|^2]}$. Finally, noting that 
$|\tilde F(X)| \leq F(0)+L\|X\|_2$ and $\Expect\|X\|_2^2 =n$ completes the proof.
\end{proof}

\section{Kronecker gymnastics}\label{appendix:kronecker}
Given $A,B \in\complex^{n\times n}$, the Kronecker product $A \otimes B \in \complex^{n^2\times n^2}$ is defined as $\qth{\begin{smallmatrix} a_{11} B&\ldots&a_{1n} B\\
\vdots&\vdots&\vdots\\
a_{n1} B&\ldots&a_{nn} B
\end{smallmatrix}}$.
The vectorized form of $A=[a_1,\ldots,a_n]$ is $\vecc(A)=[a_1^\top,\ldots,a_n^\top]^\top \in\complex^{n\otimes n}$.
It is convenient to identify $[n^2]$ with by $[n]^2$ ordered as $\{(1,1),\ldots,(1,n),\ldots,(n,n)\}$, in which case we have $(A\otimes B)_{ij,k\ell}=A_{ik}B_{j\ell}$ and $\vecc(A)_{ij}=A_{ij}$.

We collect some identities for Kronecker products and vectorizations of matrices used throughout this paper:  
\begin{align*}
\langle A \otimes A, B \otimes B \rangle & = \langle A, B \rangle^2 , \\
 (A \otimes B) ( U \otimes V) & = AU \otimes BV , \\
\vecc(A U B) & = (B^\top \otimes A) \vecc(U) , \\ 
(X \otimes Y) \vecc(U) & = \vecc(Y U X^\top) , \\ 
(A \otimes B)^\top &= A^\top \otimes B^\top . 
\end{align*}
The third equality implies that 
\begin{align*}
\langle A \otimes B, \vecc(U) \vecc(V)^\top \rangle
&= \langle \vecc(U), (A \otimes B) \vecc(V)  \rangle \\
&= \langle \vecc(U), \vecc(B V A^\top)  \rangle =\langle U, BVA^\top \rangle= \langle B, U A V^\top \rangle
\end{align*}
and hence
\begin{align*}
\vecc(U)^\top ( A \otimes B)\vecc(V)=
\langle A \otimes B, \vecc(U) \vecc(V)^\top \rangle & = \langle B, U A V^\top \rangle = \langle A, U^\top B V\rangle = \langle UA,B V\rangle.
\end{align*}

Applying the third equality to column vector $z$ and noting that $\vecc(z^\top)=\vecc(z)=z$, we have
\begin{align*}
(X \otimes y) z = & ~ \vecc(yz^\top X^\top) , \\
(y \otimes X) z = & ~ \vecc(X z y^\top) .
\end{align*}
In particular, it holds that 
\begin{align*}
(\identity \otimes y) z = & ~ \vecc(yz^\top) = z \otimes y , \\
(y \otimes \identity) z = & ~ \vecc(z y^\top)= y \otimes z  . 
\end{align*}

%\medskip
%Also, let $A=[A_1, A_2, \ldots A_n]$,
%where $A_i$ is the $i$-th column of $A$.
%Similary, let $B=[B_1, B_2, \ldots B_m]$.
%Then the columns of $A\otimes B$ are given by
%$A_i \otimes B_j.$

\section{Signal-to-noise heuristics}\label{appendix:SNR}

We justify the choice of the Cauchy weight kernel in (\ref{eq:cauchykernel}) by
a heuristic signal-to-noise calculation for $\widehat{X}$.
We assume without loss of generality that $\pi^*$ is the identity, so that
diagonal entries of $\widehat{X}$ indicate similarity between matching vertices
of $A$ and $B$. Then for the rounding procedure
in (\ref{eq:linearassignment}), we may interpret $n^{-1}\Tr \widehat{X}$ and
$(n^{-2}\sum_{i,j:\,i \neq j} \widehat{X}_{ij}^2)^{1/2}
\approx n^{-1}\|\widehat{X}\|_F$ as the average signal
strength and noise level in $\widehat{X}$. Let us define a corresponding
signal-to-noise ratio as
\[\mathrm{SNR}=\frac{\E[\Tr \widehat{X}]}{\E[\|\widehat{X}\|_F^2]^{1/2}}\]
and compute this quantity in the Gaussian Wigner model.

We abbreviate the spectral weights $w(\lambda_i,\mu_j)$ as $w_{ij}$.
For $\widehat{X}$ defined by
(\ref{eq:spectralnew}) with any weight kernel $w(x,y)$, we have
\[\Tr \widehat{X}=\sum_{ij} w_{ij} \cdot u_i^\top \bJ v_j \cdot u_i^\top v_j.\]
Applying that $(A,B)$ is equal in law to $(OAO^\top,OBO^\top)$ for a
rotation $O$ such that $O\bone=\sqrt{n}\coord_k$,
we obtain for every $k$ that
\[\E[\Tr \widehat{X}]=\sum_{ij} n \cdot \E[w_{ij}
\cdot u_i^\top (\coord_k\coord_k^\top) v_j \cdot u_i^\top v_j].\]
Then averaging over $k=1,\ldots,n$ and applying $\sum_k \coord_k \coord_k^\top
=\bI$ yield that 
\[\E[\Tr \widehat{X}]=\sum_{ij} \E[w_{ij}(u_i^\top v_j)^2].\]
For the noise, we have
\[\|\widehat{X}\|_F^2=\Tr \widehat{X}\widehat{X}^\top
=\sum_{i,j,k,l} w_{ij}w_{kl} (u_i^\top \bJ v_j)
(u_k^\top \bJ v_\ell) \Tr (u_iv_j^\top \cdot v_\ell u_k^\top)=
\sum_{ij} w_{ij}^2 (u_i^\top \bJ v_j)^2.\]
Applying the equality in law of $(A,B)$ and $(OAO^\top,OBO^\top)$ for
a uniform random orthogonal matrix $O$, and writing $r=O\bone/\sqrt{n}$, we get
\[\E[\|\widehat{X}\|_F^2]=\sum_{ij}n^2 \cdot \E[w_{ij}^2(u_i^\top r)^2(v_j^\top
r)^2].\]
Here, $r=(r_1,\ldots,r_n)$ is a uniform random vector on the unit sphere,
independent of $(A,B)$. For any deterministic unit vectors $u,v$ with
$u^\top v=\alpha$, we may rotate to $u=\coord_1$ and $v=\alpha
\coord_1+\sqrt{1-\alpha^2}\coord_2$ to get
\[\E[(u^\top r)^2(v^\top r)^2]=
\E[r_1^2 \cdot (\alpha r_1+\sqrt{1-\alpha^2}r_2)^2]
=\alpha^2\E[r_1^4]+(1-\alpha^2)\E[r_1^2r_2^2]
=\frac{1+2\alpha^2}{n(n+2)},\]
where the last equality applies an elementary computation.
Bounding $1+2\alpha^2 \in [1,3]$ and applying this conditional on $(A,B)$ above, we obtain 
\[\E[\|\widehat{X}\|_F^2]=\frac{cn}{n+2}
\sum_{ij} \E[w_{ij}^2]\]
for some value $c \in [1,3]$.

To summarize,
\[\mathrm{SNR} \asymp \frac{\sum_{ij} \E[w(\lambda_i,\mu_j)
(u_i^\top v_j)^2]}{\sqrt{\sum_{ij} \E[w(\lambda_i,\mu_j)^2]}}.\]
The choice of weights which maximizes this SNR would satisfy
$w(\lambda_i,\mu_j) \propto (u_i^\top v_j)^2$. Recall that for
$n^{-1+\eps} \ll \sigma^2 \ll n^{-\eps}$ and $i,j$ in the bulk of the spectrum,
we have the approximation (\ref{eq:partialalignment}). Thus this optimal choice
of weights takes a Cauchy form, which motivates our
choice in (\ref{eq:cauchykernel}).

We note that this discussion is only heuristic, and
maximizing this definition of SNR does not automatically imply
any rigorous guarantee for exact recovery of $\pi^*$.
Our proposal in (\ref{eq:cauchykernel}) is a bit simpler than the optimal choice
suggested by (\ref{eq:partialalignment}): The constant $C$ in
(\ref{eq:partialalignment}) depends on the semicircle density near $\lambda_i$,
but we do not incorporate this dependence in our definition. Also,
while (\ref{eq:partialalignment}) depends on the noise level $\sigma$,
our main result in Theorem \ref{thm:wigner} shows that $\eta$ need not 
be set based on $\sigma$, which is usually unknown in practice.
Instead, our result shows that the simpler choice $\eta=c/\log n$ is
sufficient for exact recovery of $\pi^*$ over a range of noise
levels $\sigma \lesssim \eta$.

\section*{Acknowledgement}
Y.~Wu and J.~Xu are deeply indebted to Zongming Ma for many fruitful discussions on the QP relaxation \prettyref{eq:ds} in the early stage of the project. 
Y.~Wu and J.~Xu thank Yuxin Chen for suggesting the gradient descent dynamics which led to the initial version of the proof.
Y.~Wu is grateful to Daniel Sussman for pointing out \cite{lyzinski2016graph} and Joel Tropp for \cite{aflalo2015convex}.

\bibliography{matching}
\bibliographystyle{alpha}

\end{document}